\documentclass{article}

     \PassOptionsToPackage{numbers, compress}{natbib}

\usepackage[main, final]{neurips_2026}

\usepackage[utf8]{inputenc} 
\usepackage[T1]{fontenc}    
\usepackage{hyperref}       
\usepackage{bookmark}
\usepackage{url}            
\usepackage{booktabs}       
\usepackage{amsfonts}       
\usepackage{nicefrac}       
\usepackage{microtype}      
\usepackage{xcolor}         

\usepackage{xspace}
\usepackage{xr}
\usepackage{amsmath}
\usepackage{amssymb}
\usepackage{amsthm}
\usepackage{tikz}
\usetikzlibrary{matrix, positioning, arrows, arrows.meta, positioning, calc, fit, mindmap, decorations.pathreplacing}
\usepackage{pgfplots}
\pgfplotsset{compat=1.18}
\usepackage{pgfplotstable}
\usepgfplotslibrary{colorbrewer}
\usepackage{minted}
\newtheorem{theorem}{Theorem}
\newtheorem{proposition}{Proposition}
\newtheorem{corollary}{Corollary}
\newtheorem{definition}{Definition}
\newtheorem{lemma}{Lemma}
\theoremstyle{remark}
\newtheorem{remark}{Remark}

\usepackage{subcaption}
\usepackage{multirow}
\usepackage{placeins}  
\usepackage{wrapfig}

\title{Frequency Domain Reservoir Computing}

\author{
  Klaus Schertler \\
  Airbus Central Research \& Technology \\
  Taufkirchen, Germany \\
  \texttt{klaus.schertler@airbus.com} \\
  \And
  Xiomara Runge \\
  Airbus Central Research \& Technology \\
  Taufkirchen, Germany \\
  \texttt{xiomara.runge@airbus.com} \\
  \AND
  Andrea Ceni \\
  Department of Computer Science \\
  University of Pisa, Italy \\
  \texttt{andrea.ceni@unipi.it} \\
  \And
  David Kappel \\
  Faculty of Technology \\
  Bielefeld University, Germany \\
  \texttt{david.kappel@uni-bielefeld.de} \\
  \And
  Claudio Gallicchio \\
  Department of Computer Science \\
  University of Pisa, Italy \\
  \texttt{claudio.gallicchio@unipi.it} \\
}

\begin{document}

\providecommand{\FD}{frequency domain\xspace}   
\providecommand{\FDh}{frequency-domain\xspace}  
\providecommand{\SD}{spatial domain\xspace}     
\providecommand{\SDh}{spatial-domain\xspace}    
\providecommand{\FT}{Fourier transform\xspace}
\providecommand{\FTh}{Fourier-transformed\xspace}


\providecommand{\SNx}{N_{\text{\!x}}}               
\providecommand{\SNy}{N_{\text{\!y}}}               
\providecommand{\SNr}{N}                            
\providecommand{\FNa}{N_1}                          
\providecommand{\FNb}{N_2}                          
\providecommand{\Sx}[1]{\mathbf{x}_{#1}}            
\providecommand{\Sy}[1]{\mathbf{y}_{#1}}            
\providecommand{\Sr}[1]{\mathbf{r}_{#1}}            
\providecommand{\SWx}{\mathbf{W}_{\text{\!x}}}      
\providecommand{\SWr}{\mathbf{W}_{\text{\!r}}}      
\providecommand{\SWy}{\mathbf{W}_{\text{\!y}}}      
\providecommand{\Sb}{\mathbf{b}}                    
\providecommand{\Sxzp}{\mathbf{x}_{\text{zp}}}      
\providecommand{\SXzp}{\mathbf{X}_{\text{zp}}}      
\providecommand{\Sactfun}{\Phi}                     


\providecommand{\Fr}[1]{\hat{\mathbf{r}}_{#1}}          
\providecommand{\FR}[1]{\hat{\mathbf{R}}_{#1}}          
\providecommand{\Fwx}{\hat{\mathbf{w}}_{\text{\!x}}}    
\providecommand{\FWx}{\hat{\mathbf{W}}_{\text{\!x}}}    
\providecommand{\Fwr}{\hat{\mathbf{w}}_{\text{\!r}}}    
\providecommand{\FWr}{\hat{\mathbf{W}}_{\text{\!r}}}    
\providecommand{\Fb}{\hat{\mathbf{b}}}                  
\providecommand{\FB}{\hat{\mathbf{B}}}                  
\providecommand{\Fx}[1]{\hat{\mathbf{x}}_{#1}}          
\providecommand{\FX}[1]{\hat{\mathbf{X}}_{#1}}          
\providecommand{\Fxzp}{\hat{\mathbf{x}}_{\text{zp}}}    
\providecommand{\FXzp}{\hat{\mathbf{X}}_{\text{zp}}}    

\providecommand{\FWp}{\mathbf{W}_{\text{\!p}}}          
\providecommand{\Fp}[1]{\mathbf{p}_{#1}}                

\providecommand{\Factfun}{\hat{\Phi}}                   


\providecommand{\pack}[1]{\mathcal{P}\{#1\}}              
\providecommand{\fourier}[1]{\mathcal{F}\{#1\}}           
\providecommand{\invfourier}[1]{\mathcal{F}^{-1}\!\{#1\}} 

\maketitle

\begin{abstract}
  While the quadratic sequence-length bottleneck of transformers has fueled a resurgence in recurrent models, effectively capturing complex dynamics requires architectures that balance efficient training with highly expressive latent states. Echo State Networks (ESNs) offer a compelling approach by utilizing fixed recurrent weights to circumvent backpropagation through time, enabling a closed-form training solution. However, achieving the expressivity needed for complex tasks demands large reservoirs, exposing an $\mathcal{O}(\SNr^2)$ state-update bottleneck that prevents ESNs from matching the scale of contemporary recurrent models.
  To address this limitation, we introduce Frequency Domain Reservoir Computing (FRESCO), an ESN architecture operating entirely in the frequency domain while avoiding domain-shift overheads to achieve $\mathcal{O}(\SNr)$ complexity for dense, non-linear updates via structured circulant recurrence. By employing a novel dimensional zero-padding input embedding, a packed \FDh readout, and a natively applied \FDh non-linearity, FRESCO drastically reduces computational costs and energy consumption of training and inference.
  Furthermore, FRESCO achieves competitive predictive performance on memory benchmarks, sequential classification, and multivariate long-horizon forecasting, offering a scalable and energy-efficient path forward for dense recurrent architectures.
\end{abstract}

\vspace{-0.75em}

\section{Introduction}
\label{sec:intro}
\vspace{-0.5em}

Sequence modeling underpins a rapidly growing set of machine learning applications, from language and audio processing to control. While transformer-based architectures currently dominate the landscape, their quadratic cost in the sequence length has renewed the interest in recurrent neural models capable of processing long sequences with bounded per-step cost. At the heart of this resurgence lies a question that has shaped recurrent computation since its origins: how to make recurrent dynamics \emph{expressive}, \emph{stable}, and \emph{cheap to evaluate}.

The recent wave of linear-scaling recurrent architectures, including structured state spaces \cite{gu2021efficiently, gu2020hippo, fu2022hungry}, their selective variants \cite{gu2023mamba, lahoti2026mamba}, and linear recurrent units \cite{orvieto2023resurrecting}, has shown that large recurrent states can be made practical when the per-step cost scales favorably with the state dimension.
This observation reframes the efficiency question for older recurrent paradigms: in order to remain competitive in modern sequence modeling, their per-step cost should scale sub-quadratically with the size of the recurrent state.

Reservoir Computing (RC), particularly the Echo State Network (ESN) \cite{jaeger2001echo, LUKOSEVICIUS2009127}, occupies a complementary position in this landscape. By keeping recurrent connectivity \emph{fixed} and training only a linear readout via convex ridge regression, ESNs completely circumvent backpropagation through time to provide a closed-form solution. They provide well-studied stability conditions rooted in the Echo State Property (ESP), and a hardware profile that is well suited for analog and neuromorphic substrates \cite{yan2024emerging}.

Despite these strengths, the per-step cost of a dense ESN scales as $\mathcal{O}(\SNr^2)$ with the number of reservoir neurons, dominated by a recurrent matrix-vector product. This bottleneck prevents ESNs from scaling to modern sequence modeling regimes, eroding the very efficiency advantage that originally motivated RC. The frequency domain offers a principled route to overcome this bottleneck. Because the 1D Discrete Fourier Transform (DFT) diagonalizes circulant weight matrices, the dense matrix-vector product collapses to an element-wise Hadamard product (circular convolution theorem). However, a naive \FDh ESN incurs additional overheads of domain transitions: an $\mathcal{O}(\SNr \log \SNr)$ Fast Fourier Transform (FFT) to ingest each input, and an inverse FFT before every readout. In the typical $\SNr \gg \SNx$ regime, these transforms dominate the computational cost, preventing the architecture from fully realizing the $\mathcal{O}(\SNr)$ efficiency.

In this paper, we introduce \emph{frequency domain reservoir computing} (FRESCO), an ESN architecture that operates \emph{entirely} in the frequency domain and is designed for efficiency. FRESCO circumvents the input transformation bottleneck by lifting the 1D circular convolution formulation to 2D, enabling a novel \emph{dimensional zero-padding} embedding. This allows an input of size $\SNx$ to be transformed using a single $\mathcal{O}(\SNx \log \SNx)$ FFT rather than a full $\SNr$-point FFT. By replacing the dense $\mathcal{O}(\SNr^2)$ recurrent product with element-wise Hadamard multiplications between \FDh state and weight representations, the per-step cost of the recurrent update drops to $\mathcal{O}(\SNr)$, even for a fully dense, non-linear reservoir. Finally, a \emph{packed \FDh readout} exploits the Hermitian symmetry of the reservoir state to yield a contiguous, real-valued vector. On this representation, ridge regression is provably equivalent to spatial-domain training, entirely removing the inverse FFT from both training and inference. Figure~\ref{fig:FRESCO} contrasts these three architectural advantages against a standard ESN.

Our main contributions are summarized as follows:

\begin{itemize}
\item \textbf{Frequency domain reservoir computing:} We introduce FRESCO, a circular-dense, non-linear ESN architecture recurring with $\mathcal{O}(\SNr)$ complexity natively in the \FD.
\item \textbf{Dimensional zero-padding:} We propose a novel input embedding strategy mapping low-dimensional inputs into the \FD, eliminating standard transform bottlenecks.
\item \textbf{Packed \FDh readout:} We design a redundancy-free, memory-contiguous, and FFT-free state readout, provably equivalent to spatial-domain training and inference.
\item \textbf{High-efficiency:} We show that FRESCO drastically reduces computational and energy footprints while delivering competitive results against established sequence models.
\end{itemize}

\begin{figure}[h]
    \centering
    \resizebox{\linewidth}{!}{
    \input{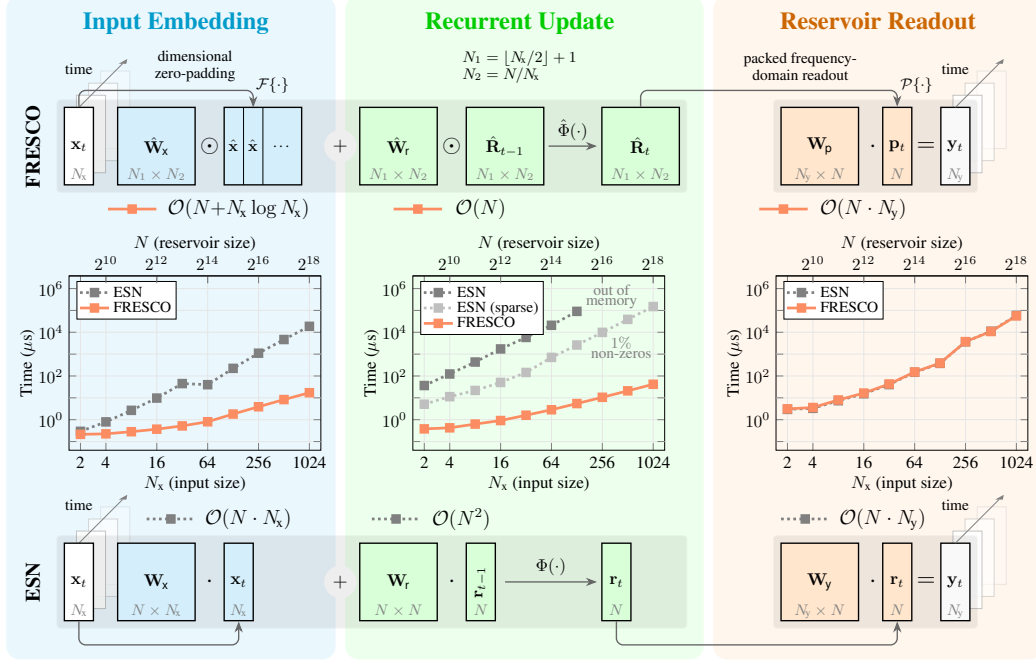}
    }
\caption[FRESCO Overview]{FRESCO (top row) accelerates ESNs by operating natively in the \FD while actively bypassing standard FFT bottlenecks. In the \FD, dense matrix operations mathematically collapse into fast element-wise products ($\odot$), lowering end-to-end computational costs compared to dense and sparse standard ESNs (bottom row). Left: The introduced dimensional zero-padding minimizes input embedding overhead. Center: The recurrent state updates are directly executed via element-wise products, and non-linearities $\Factfun$ are applied in the \FD. Right: To match standard readout speeds, the novel packed \FDh readout computes $\Sy{t}$ directly from a memory-contiguous, non-redundant state, entirely avoiding inverse FFTs. Together, these mechanisms yield highly favorable time and memory scaling with increasing input size $\SNx$ and reservoir size $\SNr$ (central plots: synthetic scaling benchmarks where $\SNr$ is varied as an independent variable and $\SNr=256\SNx$; see Appendix~\ref{app:theory} for details). Notation (see Appendix~\ref{app:nomenclature} for a comprehensive list): complex ($\hat{\cdot}$), real (unaccented), vectors (bold lowercase), matrices (bold uppercase). Bias/leaking terms are omitted.}
\label{fig:FRESCO}
\end{figure}

\vspace{-0.5em}

\section{Background and related work}
\label{sec:rel_work}

\vspace{-0.5em}

This section provides the background and related work to frame our proposed FRESCO approach. We first review Echo State Networks, their computational characteristics, and the stability conditions that govern their dynamics. These aspects are important to inform and motivate the design of FRESCO. We then briefly discuss deep State Space Models (SSMs), a related line of work in sequence modeling that shares theoretical foundations with RC, but follows a fundamentally different training paradigm.

\paragraph{Echo State Networks.}
Echo State Networks (ESNs) are a RC paradigm characterized by a randomly connected, fixed hidden state layer and a trainable linear readout \cite{jaeger2001echo, jaeger2007optimization, LUKOSEVICIUS2009127}, thereby bypassing the vanishing gradient problems of traditional RNNs and enabling highly efficient, closed-form training  (typically ridge regression). 

Here, we refer to the standard leaky-neurons formulation of ESNs introduced in \cite{jaeger2007optimization}. For an input $\Sx{t} \in \mathbb{R}^{\SNx}$, a reservoir of $\SNr$ neurons updates its internal state $\Sr{t} \in \mathbb{R}^{\SNr}$ via:
\begin{equation}
\Sr{t} = (1-\tau) \, \Sr{t-1} + \tau \, \Sactfun\left( \SWx \Sx{t} + \SWr \Sr{t-1} + \Sb \right) \label{eq:esn-sd-update}
\end{equation}
where $\SWr \in \mathbb{R}^{\SNr \times \SNr}$, $\SWx \in \mathbb{R}^{\SNr \times \SNx}$, and $\Sb \in \mathbb{R}^{\SNr}$ are the fixed recurrent weights, input weights, and biases. The leaking rate is $\tau \in (0, 1]$, and $\Sactfun(\cdot)$ is an element-wise non-linearity. The network output $\Sy{t} = \SWy \Sr{t} \in \mathbb{R}^{\SNy}$ is generated through the readout matrix $\SWy \in \mathbb{R}^{\SNy \times \SNr}$, which contains the only optimized parameters.

A dense ESN's per-step complexity is dominated by these matrix-vector multiplications, scaling at $\mathcal{O}(\SNr^2 + \SNr \SNx)$. Overcoming this quadratic bottleneck is the primary objective of FRESCO, while preserving the standard computational profile for both training and inference. A comparison of both architectures is shown in Figure~\ref{fig:FRESCO}.

A key stability requirement for ESNs is the Echo State Property (Definition~\ref{def:esp}), imposing a constraint on the recurrence weights (Proposition~\ref{prop:esp_esn}). In practice, reservoirs are initialized by drawing $\SWr$ at random and then rescaling its elements to enforce a desired value of its spectral radius $\rho(\SWr)$. This requires an $\mathcal{O}(\SNr^3)$ eigendecomposition, a bottleneck that scales poorly with reservoir size.

\begin{definition}[Echo State Property \citep{jaeger2001echo}]
\label{def:esp}
A reservoir system driven by a bounded input sequence $\{\Sx{t}\}_{t\geq 0}$ satisfies the \emph{Echo State Property} (ESP) if, for any two initial states $\Sr{0}$ and $\Sr{0}'$, the corresponding state trajectories satisfy $\|\Sr{t} - \Sr{t}'\| \to 0$ as $t \to \infty$. That is, the system asymptotically forgets its initial condition and provides an ``echo'' of the driving input signal.
\end{definition}

\begin{proposition}[ESP conditions for the ESN \citep{jaeger2001echo, LUKOSEVICIUS2009127}]
\label{prop:esp_esn}
For the ESN of Eq.~\eqref{eq:esn-sd-update} with a non-linearity $\Sactfun$ of Lipschitz constant $L_{\Sactfun}\leq 1$ (e.g.\ $\tanh$):
\emph{(Sufficient)} a sufficient condition for the ESP is $\|\SWr\|_2 < 1$,
where $\|\cdot\|_2$ denotes the matrix $2$-norm (largest singular value $\sigma_{\max}(\SWr)$);
\emph{(Necessary)} a necessary condition for the ESP is $\rho(\SWr) \leq 1$.
\end{proposition}

Having established the computational challenges of standard ESNs, the primary motivation of FRESCO, we now briefly survey State Space Models, a parallel line of work that addresses sequence modeling from a deep learning perspective.

\paragraph{State Space Models.}
Deep State Space Models (SSMs) are a complementary approach to ESNs, which use linear projections to address long-range dependencies in recurrent networks \citep{voelker2019legendre, gu2020hippo, fu2022hungry}, with modern selective variants achieving strong performance on complex sequence tasks \citep{lahoti2026mamba, gu2023mamba}. SSMs often leverage FFT-based computation \citep{gu2021efficiently, chilkuri2021parallelizing} and, like ESNs, can be viewed as kernel machines over impulse responses; however, SSMs typically learn structured recurrent kernels end-to-end, whereas ESNs use fixed random kernels and optimize only a linear readout \citep{singh2025echo}. State-of-the-art SSMs are included in our accuracy, energy, and runtime comparisons presented below.

\paragraph{Frequency-Domain Methods.}
Leveraging frequency-domain representations to replace expensive spatial operations has emerged as a broadly effective strategy across deep learning, from learning PDE solution operators in Fourier space~\citep{li2021fourierneuraloperatorparametric} to sequence and vision modeling, where both unparameterized Fourier mixing~\citep{lee-thorp-etal-2022-fnet} and learnable global frequency-domain filters~\citep{NEURIPS2021_07e87c2f, guibas2022adaptivefourierneuraloperators, huangAdaptiveFrequencyFilters2023} have demonstrated that spatial attention and dense token mixing can be replaced at substantially reduced computational cost. However, these methods are restricted to feed-forward backbones trained end-to-end via backpropagation. Building upon the principles of efficient spectral processing, FRESCO extends frequency-domain operations to recurrent reservoir dynamics, while systematically minimizing FFT transformation overheads through dimensional zero-padding and packed frequency-domain readout.

\section{FRESCO}
\label{sec:fresco}

The core innovation of FRESCO is to i) reformulate the standard ESN of Eq.~\eqref{eq:esn-sd-update} in the \FD to exploit the inherent efficiencies of established \FDh processing techniques, yet ii) systematically circumvent costly FFTs usually associated with these approaches. 
The FRESCO concept is outlined and compared to a standard ESN in Figure~\ref{fig:FRESCO}. It comprises a novel method for efficiently embedding arbitrary inputs into the \FD, executing the recurrent updates in the \FD, and performing 
readouts without requiring subsequent FFTs on the reservoir state. Combined, this reduces the per-step computational complexity to scale strictly linearly with the reservoir size.

As a foundational baseline for the reformulation in \FD, we first consider the case where weight matrices of Eq.~\eqref{eq:esn-sd-update} are constrained to be circulant.

\begin{definition}[Circulant matrix]
A matrix $\mathbf{W} \in \mathbb{R}^{M \times M}$ is called a circulant matrix if each row is a cyclic right shift of the row above it. It is fully specified by its first column, denoted as $\mathbf{w} \in \mathbb{R}^M$. Consequently, the elements of \,$\mathbf{W}$ satisfy $W_{i,j} = w_{(i-j) \bmod M}$.
\end{definition}

Circulant matrices are dense but structured: while they model all-to-all neuron interactions, their $N^2$ entries are determined by $N$ independent parameters. This structured connectivity reduces the representable weight configurations compared to a fully unstructured $\mathbf{W}_r \in \mathbb{R}^{N \times N}$, a trade-off that FRESCO's $\mathcal{O}(N)$ scaling fundamentally rests on. The profound advantage of this constraint is that circulant matrices are natively diagonalized by the DFT, enabling the circular convolution theorem (Lemma~\ref{lem:circular_conv}). 

\begin{lemma}[Circular convolution theorem \citep{oppenheim1999discrete}]
\label{lem:circular_conv}
Let $\mathbf{W} \in \mathbb{R}^{M \times M}$ be a circulant matrix with first column $\mathbf{w} \in \mathbb{R}^M$, and let $\mathbf{x} \in \mathbb{R}^M$ be an arbitrary vector. The matrix-vector product $\mathbf{y} = \mathbf{W} \mathbf{x}$ is equivalent to the circular convolution $\mathbf{y} = \mathbf{w} \star \mathbf{x}$. In the \FD, this operation is given by
\begin{equation}
    \mathcal{F}(\mathbf{y}) = \fourier{\mathbf{w}} \odot \fourier{\mathbf{x}}
\end{equation}
where $\fourier{\cdot}$ denotes the DFT and $\odot$ represents the element-wise (Hadamard) product.
\end{lemma}

Lemma~\ref{lem:circular_conv} thus allows us to replace the costly matrix-vector multiplications 
of Eq.~\eqref{eq:esn-sd-update} with element-wise complex multiplications. This yields the leaky ESN state update in the \FD:

\begin{equation}
\label{eq:esn-fd-update-1D}
\Fr{t} = (1-\tau) \, \Fr{t-1} +  \tau\, \Factfun\left(\Fwx \odot \Fx{t} + \Fwr \odot \Fr{t-1} + \Fb \right),
\end{equation}
where $\Fr{}$, $\Fx{}$, and $\Fb$ are the \FD vector representations of the reservoir state, input data, and bias, respectively. $\Fwx$ and $\Fwr$ denote the \FTh first columns of circulant weight matrices $\SWx$ and $\SWr$. $\Factfun$ represents a non-linear activation function applied in \FD. 

For reasons elaborated upon below, the FRESCO approach generalizes the 1D circular convolution formulation of Eq.~\eqref{eq:esn-fd-update-1D} to a 2D convolution formulation given in Eq.~\eqref{eq:esn-fd-update-2D}, as follows:

\begin{equation}
\label{eq:esn-fd-update-2D}
\FR{t} = (1-\tau) \, \FR{t-1} +  \tau\, \Factfun\left(\FWx \odot \FX{t} + \FWr \odot \FR{t-1} + \FB \right).
\end{equation}
Here, $\FR{}$, $\FX{}$, $\FB$, $\FWr$, and $\FWx$ are complex-valued matrices in $\mathbb{C}^{\FNa \times \FNb}$, representing the 2D \FDh reservoir state, input data, bias, and weights. The dimensions $\FNa$ and $\FNb$ are determined by the input size and the total number of reservoir neurons, as explained in the following.

\paragraph{Input embedding.}
While Eq.~\eqref{eq:esn-fd-update-1D} yields an elegant $\mathcal{O}(\SNr)$ update, transforming a lower-dimensional input $\Sx{t} \in \mathbb{R}^{\SNx}$ into the \FD $\Fx{t} \in \mathbb{C}^{\SNr}$ introduces computational overhead we seek to minimize. 
While zero-padding is a standard technique for matching dimensionalities with static convolution kernels in \FD, applying it to the dynamic input $\Sx{t}$ necessitates computing an $\SNr$-point DFT at every time step, incurring an $\mathcal{O}(\SNr \log \SNr)$ cost.

To circumvent this, FRESCO generalizes the 1D circular convolutions of Eq.~\eqref{eq:esn-fd-update-1D} to 2D circular convolutions. By restructuring the reservoir state into 2D matrices (Eq.~\eqref{eq:esn-fd-update-2D}), we enable a highly efficient strategy called dimensional zero-padding. Rather than padding the 1D input along its existing axis, we embed $\Sx{t}$ by appending zeros along a new second dimension to form a padded matrix $\SXzp$. As formally proved in Theorem~\ref{thm:dimensional_padding} (Appendix~\ref{app:proof-dim-zero-pad}), the 2D DFT $\FX{} = \fourier{\SXzp}$ can be expressed solely by the $\SNx$-point DFT $\Fx{t} = \fourier{\Sx{t}}$. The embedding in Eq.~\eqref{eq:esn-fd-update-2D} therefore becomes:

\begin{equation}
\FWx \odot \FX{t} = \FWx \odot \Fx{t} = \FWx \odot  \fourier{\Sx{t}}, \label{eq:embedding_dim_zeropadding}
\end{equation}

where the column vector $\Fx{t}$ is broadcast for element-wise multiplication with the columns of the 2D input weights $\FWx$ (Figure~\ref{fig:FRESCO}, left). This reduces the embedding complexity to $\mathcal{O}(\SNx \log \SNx)$ while perfectly preserving the functional goal of the convolutions. Timing experiments (Appendix~\ref{app:dim-zero-pad}) demonstrate its superior empirical efficiency over alternative methods.

Intuitively, this strategy folds the standard 1D reservoir into a 2D spatial grid of dimensions $\SNx \times \frac{\SNr}{\SNx}$, where the first axis is determined by the input size $\SNx$. By utilizing a real-to-complex 2D FFT representation (RFFT), FRESCO exploits conjugate symmetry to compactly embed the \FDh representations within the space $\mathbb{C}^{\FNa \times \FNb}$, where $\FNa = \lfloor \SNx/2 \rfloor + 1$ and $\FNb=\SNr/\SNx$ (assuming $\SNr$ is a multiple of $\SNx$). The packed \FDh readout procedure (detailed below) inherently resolves any minor redundancies remaining in this compressed representation.

\paragraph{Recurrence weights.}
In standard ESNs, bounding the spectral radius of $\SWr$ to satisfy the Echo State Property (ESP) requires an expensive $\mathcal{O}(\SNr^3)$ decomposition. FRESCO bypasses this bottleneck by sampling the weight matrix $\FWr$ directly in the \FD. This yields a profound theoretical advantage: the \FDh weight elements correspond exactly to the eigenvalues of the convolution operator. Because the recurrence term $\FWr \odot \FR{}$ in Eq.~\eqref{eq:esn-fd-update-2D} relies on element-wise multiplication, its eigenvectors are simply the standard basis matrices $\mathbf{E}^{(i,j)}$. The eigenvalues naturally emerge from $\FWr \odot \mathbf{E}^{(i,j)} = \FWr[i,j] \cdot \mathbf{E}^{(i,j)}$ as the individual elements $\FWr[i,j]$.
Treating $\mathbb{C}^{\FNa \times \FNb}$ as a Euclidean space under the Frobenius norm, the spectral radius of the Hadamard operator $\mathcal{L}(\cdot) = \FWr \odot (\cdot)$,
and hence of the spatial recurrent matrix $\SWr$ whose eigenvalues are precisely
the entries of $\FWr$, is therefore:
\begin{equation}
\rho(\mathcal{L}) = \max_{i,j} |\FWr[i,j]| = \rho(\SWr),
\label{eq:spectral-radius-hadamard}
\end{equation}
directly readable from the entries of $\FWr$ at zero cost, with no eigendecomposition. 
Consequently, sampling the recurrence weights directly in the \FD enables explicit shaping of the eigenvalue spectrum, such as constraining the complex eigenvalues $\FWr[i,j]$ to a specific radial ring. This grants FRESCO fine-grained control over the reservoir's temporal properties at zero marginal cost. Since these parameters are strictly initialized and maintained in the \FD, their spatial representation is never explicitly computed.

\paragraph{Non-linearity.}
We apply the complex activation function 
\begin{equation}
\operatorname{invabs}(z) = \frac{\alpha z}{1+|z|}
\label{eq:invabs}
\end{equation}
 element-wise to the \FDh reservoir state, where the scaling factor $\alpha$ is a tunable hyperparameter. This introduces the required non-linearity and strictly bounds the magnitude of the state components below $\alpha$. Because the denominator is strictly real and positive, this operation preserves the phase of $z$. 
 Moreover, $\operatorname{invabs}$ admits a Lipschitz constant equal to $\alpha$, a property that enables the ESP characterization derived below for FRESCO.
 However, purely point-wise recurrence inherently isolates the dynamics of each frequency bin, which can restrict overall network expressivity. This is a trade-off we quantify empirically via linear and quadratic memory capacity in Section \ref{sec:benchmark_memory}.

To address this limitation and systematically investigate the impact of cross-frequency coupling, we introduce two variants of the FRESCO architecture, denoted as FRESCO (plain) and FRESCO (mix). FRESCO (plain) applies the non-linearity $\Factfun(z) = \operatorname{invabs}(z)$, Eq.~\eqref{eq:invabs}. FRESCO (mix) introduces a structural mixing mechanism. In addition to the $\operatorname{invabs}$ activation, we apply a discrete circular shift to the flattened pre-activation matrix. By establishing a single-cycle ring topology directly within the \FD, feeding the state of vector bin $i$ into the adjacent bin $i+1$, we elegantly couple the frequency components without sacrificing the optimal scaling of point-wise \FDh operations. 
This choice is principled, as it provides deterministic cross-frequency coupling without introducing learnable parameters or additional computational overhead.
By the Fourier shift theorem, this shift induces a global, complex phase modulation in the spatial-domain state. Upon activation, this modulation drives rich, cross-frequency harmonic mixing. 
In both variants, the subsequent packed \FDh readout (discussed in the following section) implicitly projects the state back to the real domain by extracting only the components corresponding to a real-to-complex representation of the \FDh reservoir.

\paragraph{Readout and training.}
The reservoir output $\Sy{t}$ is typically defined as a real-valued vector in the \SD. A naive readout strategy would require mapping the complex \FDh state $\FR{t}$ back to the spatial domain via an inverse FFT at each time step: $\Sr{t} = \text{vec}(\invfourier{\FR{t}})$. FRESCO bypasses this costly transformation by introducing a packed \FDh readout. This mechanism executes training (via standard ridge regression) and inference (via a real-valued linear readout) natively on the \FDh representation, completely eliminating computational overhead while maintaining mathematical equivalence to a training process performed explicitly on the inverse-transformed reservoir states.

This equivalence stems from the unitary nature of the normalized DFT: performing ridge regression on the flattened complex state $\text{vec}(\FR{})$ is mathematically identical to training on the spatial states (Proposition~\ref{prop:unitary_equiv}, Appendix~\ref{app:ridge_invariance}). However, directly utilizing raw \FDh states is suboptimal, as their inherent Hermitian symmetries retain redundant parameters and compel inefficient complex arithmetic for an ultimately real-valued target. This structural inefficiency persists whether operating natively on complex numbers, or simply reinterpreting the state as interleaved real values. To resolve this, we introduce a packing operation $\pack{\cdot}$ that extracts the non-redundant degrees of freedom from $\FR{t}$ into a dense, real-valued vector $\Fp{t} = \pack{\FR{t}} \in \mathbb{R}^N$ in a highly efficient, memory-contiguous way. The final readout is simply:
\begin{equation}
\Sy{t} = \FWp \cdot \Fp{t},
\label{eq:fd-readout}
\end{equation}

where $\FWp$ is the learned weight matrix, matching the exact trainable parameter count of a standard ESN. Because $\pack{\cdot}$ acts as a scaled unitary operation (formally proved in Proposition~\ref{prop:packed_dft_unitary}, Appendix~\ref{app:unitarity_packed}), ridge regression on $\Fp{}$ remains mathematically equivalent to training on the spatial states, up to a global scaling factor absorbed by the regularization term (Corollary~\ref{cor:scaled_unitary}, Appendix~\ref{app:ridge_invariance}). This equivalence holds when the reservoir state maintains Hermitian symmetry, which the element-wise recurrence of FRESCO (plain) preserves by construction. For FRESCO (mix), this symmetry is not generally maintained, and the packing serves as a heuristic that extracts $\SNr$ real-valued features without a formal equivalence guarantee.

For a hardware-efficient implementation of $\pack{\cdot}$, we represent $\FR{}$ using a custom axis-reordered 2D RFFT format. Standard 2D RFFT formats scatter non-redundant frequency components across non-contiguous memory segments, necessitating costly copying or strided memory gathering. Our custom format consolidates these components into a single contiguous block (Appendix~\ref{app:packed_readout}), extracting $\Fp{}$ using a hardware-efficient memory view. Consequently, FRESCO's per-step readout complexity strictly matches the standard ESN approach  as demonstrated in Figure~\ref{fig:FRESCO}.

\paragraph{Echo State Property.}
Two ingredients enable a complete ESP characterization for both FRESCO variants, in analogy with Proposition~\ref{prop:esp_esn}: the identity $\rho(\mathcal{L}) = \max_{i,j}|\FWr[i,j]|$ (Eq.~\eqref{eq:spectral-radius-hadamard}), which makes the spectral radius directly readable from $\FWr$, and the Lipschitz constant $L_{\Factfun} = \alpha$. We provide a sufficient and a necessary condition below.

\begin{proposition}[Sufficient condition for FRESCO ESP]
\label{prop:esp_sufficient}
If $\;\alpha \cdot \max_{i,j}\lvert\FWr[i,j]\rvert < 1$, then FRESCO satisfies the Echo State Property.
\end{proposition}

The condition holds for both variants; see Appendix~\ref{app:esp-sufficient} for the proof.

\begin{proposition}[Necessary condition for FRESCO ESP]
\label{prop:esp_necessary}
For FRESCO (plain), a necessary condition for the Echo State Property is
\begin{equation}
\label{eq:esp-necessary-plain}
\max_{i,j}
\left\{
\tau\alpha^2 |\FWr[i,j]|^2
+
2(1-\tau)\alpha\,\operatorname{Re}(\FWr[i,j])
\right\}
\le
2-\tau.
\end{equation}
In the special case of FRESCO (plain) with $ \tau=1 $ (no leak), this reduces to $ \alpha \max_{i,j}|\FWr [i,j]| \le 1 $,  closing the gap with the sufficient condition of Proposition~\ref{prop:esp_sufficient}.
For FRESCO (mix), define $ c=\prod_{i,j}\FWr[i,j] $. Then, a necessary condition to satisfy the Echo State Property is
\begin{equation}
    \label{eq:esp-necessary-mix}
(\alpha g)^2\tau
+
2g (1-\tau)\alpha 
\max_{k=0,\ldots,\SNr-1}
\cos\left(
\frac{\theta+2\pi k}{\SNr}
\right)
\le
2-\tau, \qquad \text{where} \quad g=|c|^{1/N},
\quad
\theta=\arg(c).
\end{equation}
In the special case of FRESCO (mix) with $ \tau=1 $ (no leak), this reduces to $ \left(\,\prod_{i,j} |\FWr[i,j]|\right)^{1/\SNr} \leq \dfrac{1}{\alpha} $.
\end{proposition}

Proofs and a detailed discussion are provided in Appendix~\ref{app:esp-necessary}.

\section{Experiments}
\label{sec:experiments}
We evaluate FRESCO along three complementary axes using two proposed variants. FRESCO~(plain) keeps frequencies isolated, applying only the non-linear $\operatorname{invabs}$ activation (Eq.~\ref{eq:invabs}), whereas FRESCO~(mix) introduces additional deterministic single-cycle mixing to drive cross-frequency memory interactions. First, we compare these variants against standard ESNs on classical RC memory tasks (linear and quadratic memory capacity from \citet{jaeger2002short, dambre2012information}) and regression benchmarks (NARMA10, Mackey-Glass from \citet{jaeger2004harnessing}) to assess baseline predictive capabilities, dimensional scaling, and inference speeds. Second, we evaluate these same RC models on ten datasets from the UCR archive \citep{8894743} to test sequence-level discrimination. Finally, for multivariate long-horizon forecasting (ETT, Solar, Weather from \citet{wu2021autoformer,lai2018modeling}), we shift our comparison from RC baselines to contemporary deep sequence models to investigate FRESCO's potential as a highly energy-efficient alternative for complex predictive tasks. Additional ablations of the input embedding topology and complex non-linearity are reported in Appendices~\ref{app:embedding_ablation} and~\ref{app:nonlinearity_ablation}.

For all tasks and models except for the memory tasks, hyperparameters are tuned on a validation split via an Optuna TPE search with Hyperband pruning. The optimal configuration is then retrained on the combined train and validation data and evaluated on a held-out test set. We report the mean and standard deviation across multiple random initializations.
Full details regarding dataset properties, the exact train, validation, and test splits for each dataset, and hyperparameter search ranges are given in Appendix~\ref{app:experiments}.

\subsection{Memory capacity benchmarks}
\label{sec:benchmark_memory}
To isolate how structural priors impact temporal processing, we evaluated the Linear and Quadratic Memory Capacity (LMC and QMC) of our architectures across a logarithmic sweep of input scaling factors. We explicitly separate the effects of network structure from the underlying activation dynamics by comparing two spectral topologies: A standard random disk (eigenvalue magnitudes $\leq 0.9$) and a rigid ring (all eigenvalue magnitudes strictly $0.9$). Detailed descriptions of the data and experiments are provided in \ref{app:dataset-details}.

While circulant recurrent structures incur an initial linear memory capacity (LMC) penalty under standard disk topologies (Figure~\ref{fig:capacity}), operating natively on an orthogonal ring topology (a zero-overhead structural prior in the frequency domain, see Section~\ref{sec:fresco}) fully recovers linear recall to standard ESN baseline levels. Crucially, this architecture unlocks a favorable capacity trade-off for non-linear sequence modeling: deterministic cross-frequency coupling in FRESCO (mix) redistributes network capacity into non-linear modes. Consequently, FRESCO (mix, ring) achieves a peak Quadratic Memory Capacity (QMC) of $8.3$, substantially outperforming both disk-topology ($\text{QMC} = 4.9$) and ring-topology ($\text{QMC} = 5.4$) ESN baselines. In all subsequent experiments, the lower ring radius is learned as a hyperparameter, allowing for a continuous transition between disk and ring topology.

\begin{figure}[t]
    \centering
    \begin{subfigure}[t]{0.48\linewidth}
        \centering
        \includegraphics[width=0.98\linewidth]{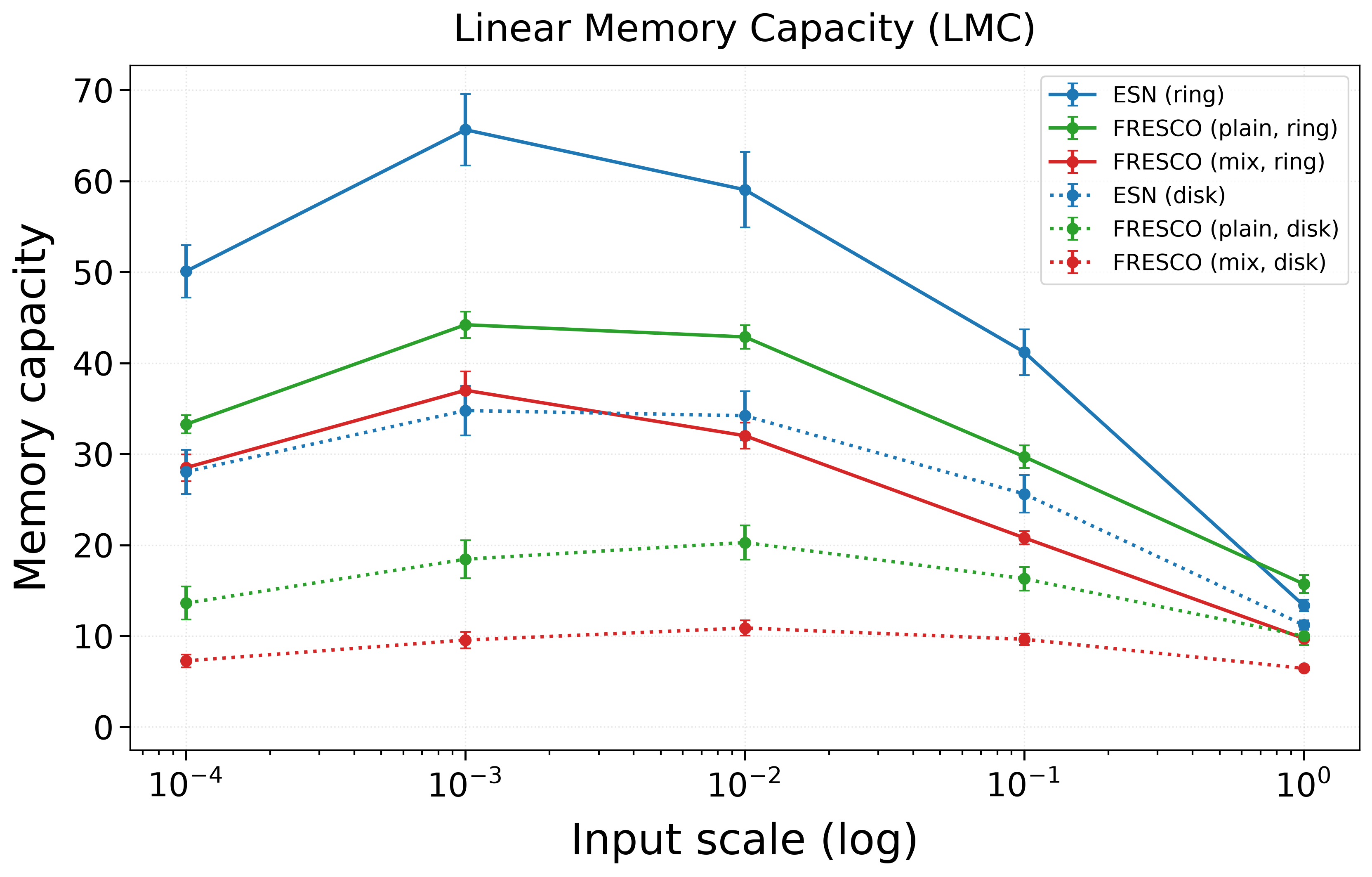}
    \end{subfigure}
    \hfill
    \begin{subfigure}[t]{0.48\linewidth}
        \centering
        \includegraphics[width=0.98\linewidth]{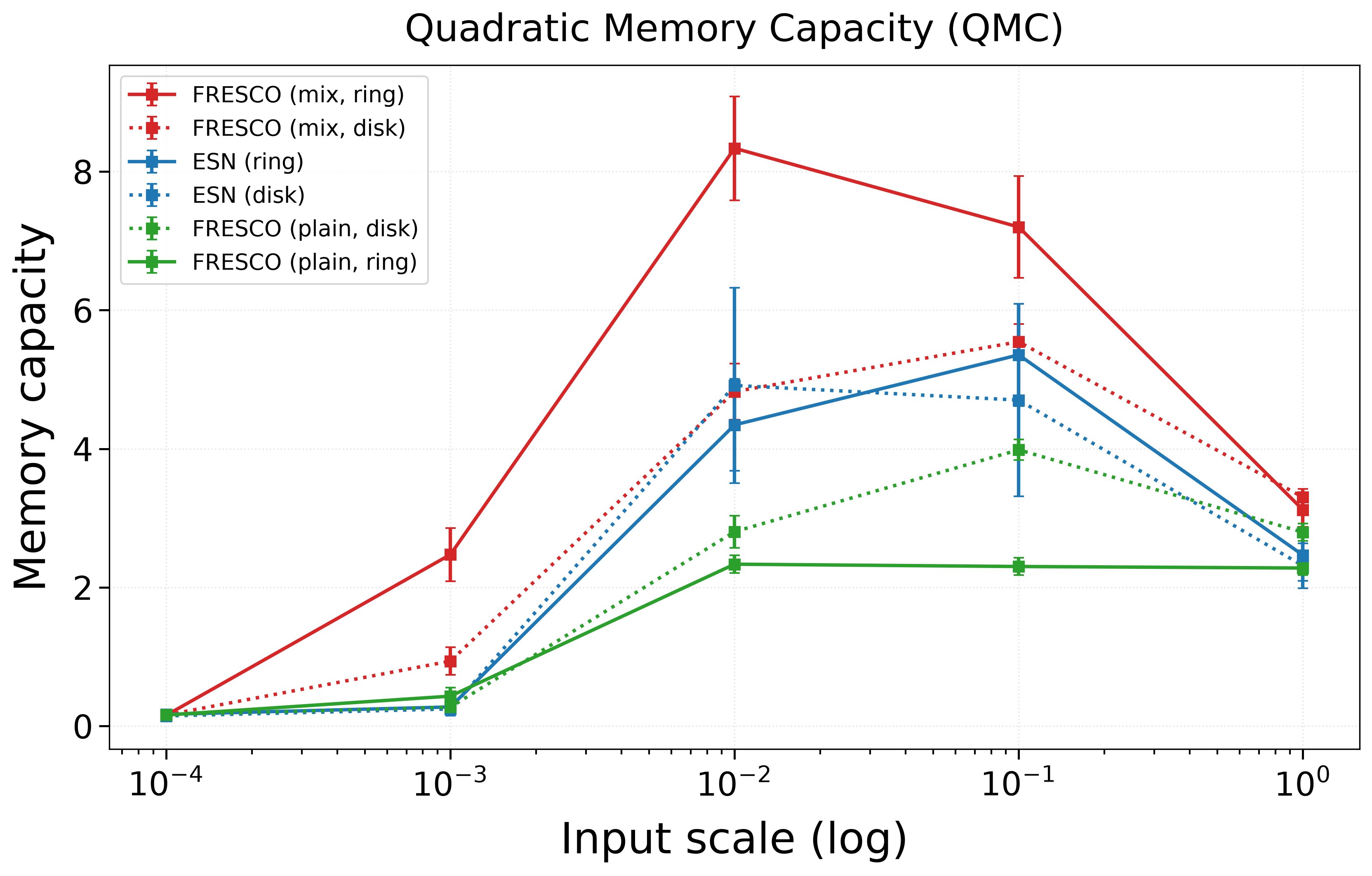}
    \end{subfigure}
    \caption{LMC (left) and QMC (right) across input scales and disk/ring eigenvalue topologies, averaged over 10 initializations. FRESCO variants natively operate on the optimal ring topology, thereby recovering LMC while FRESCO (mix) further excels in QMC.}
    \label{fig:capacity}
\end{figure}

\subsection{Regression RC benchmarks}
\label{sec:benchmark_regression}

We evaluate NARMA10 and 84-step-ahead Mackey-Glass forecasting, which respectively test non-linear memory and chaotic time-delay prediction; complete protocols are provided in Appendix~\ref{app:RC_details}.

Figure~\ref{fig:scaling_size} illustrates the quality-efficiency trade-off by plotting test NRMSE against end-to-end per-step inference time for varying reservoir sizes. In terms of predictive performance, FRESCO (mix) achieves the lowest overall error on Mackey-Glass and matches standard ESN configurations on NARMA10. It substantially outperforms FRESCO (plain) on both tasks, indicating that cross-frequency mixing is essential for capturing highly non-linear temporal dynamics.
Computationally, both FRESCO variants exhibit superior scaling behavior, running up to two orders of magnitude faster than dense ESNs for the largest reservoirs.

\begin{figure}[t]
    \centering
    \begin{subfigure}[t]{0.48\linewidth}
        \centering
        \includegraphics[width=0.98\linewidth]{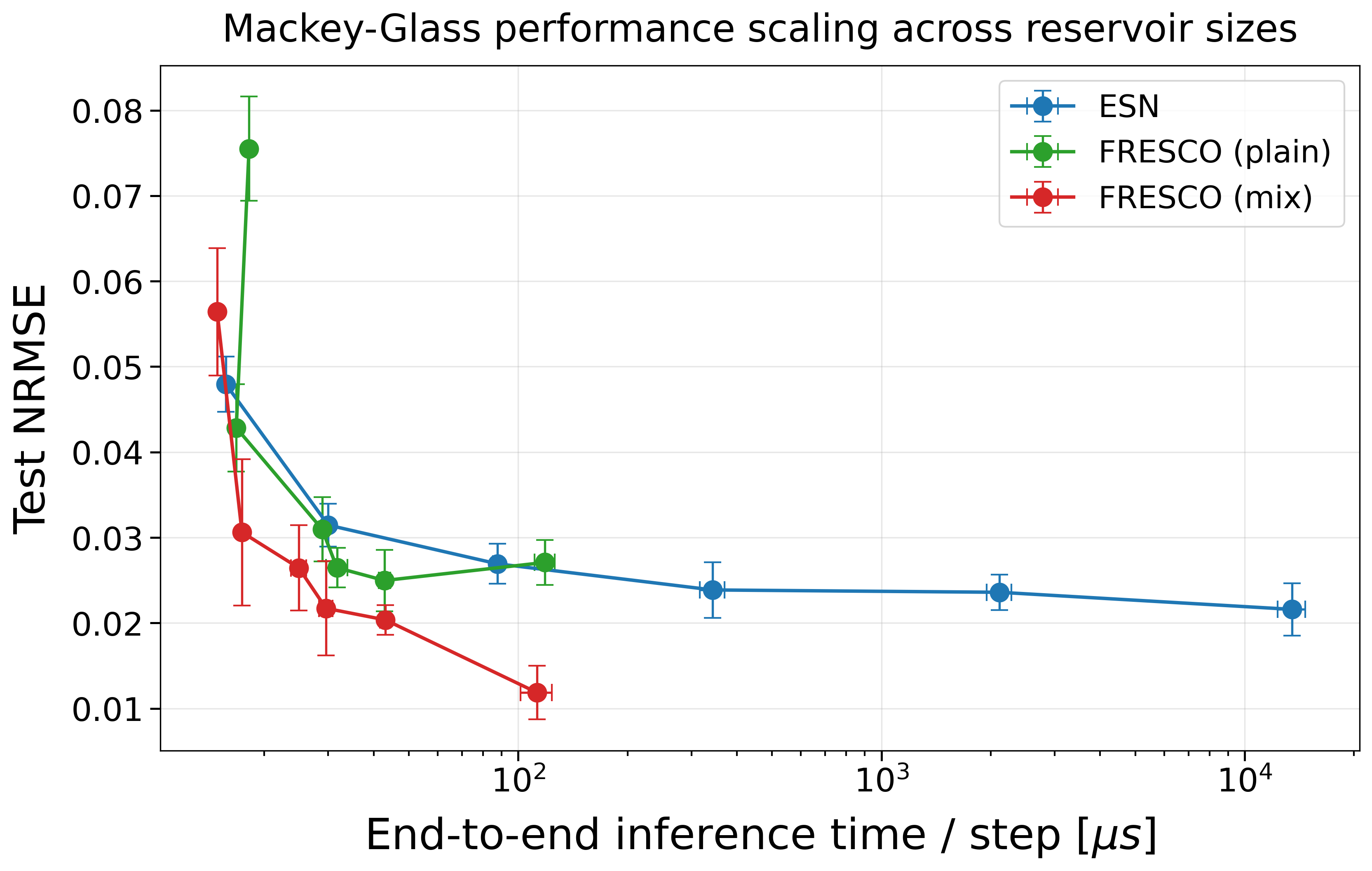}
        \label{fig:ett}
    \end{subfigure}
    \hfill
    \begin{subfigure}[t]{0.48\linewidth}
        \centering
        \includegraphics[width=0.98\linewidth]{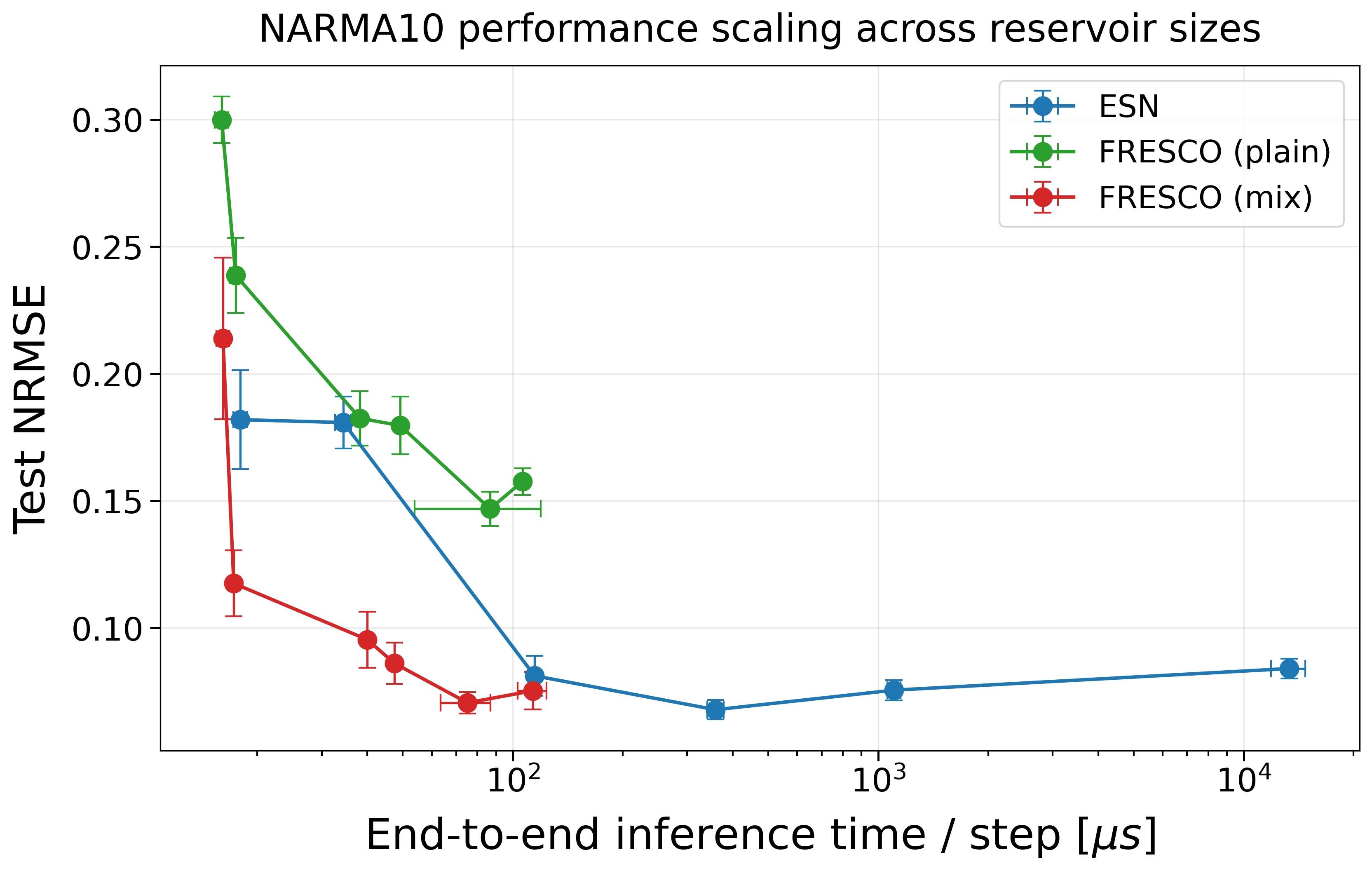}
        \label{fig:weather}
    \end{subfigure}
    \caption{Performance scaling for Mackey-Glass (left) and NARMA10 (right) for reservoir sizes $\in \{ 128, 256, 512, 1024, 2048, 4096 \}$. Plots show mean and standard deviation of NRMSE versus end-to-end execution time from input to output for a single time step across 20 random initializations.}
    \label{fig:scaling_size}
\end{figure}

\subsection{Time-series classification}
\label{sec:benchmark_classification}

We evaluate sequence-level discrimination on ten datasets from the UEA \& UCR repository \cite{8894743} (listed in Table~\ref{tab:reservoir_accuracy}) using the \texttt{aeon} library \cite{JMLR:v25:23-1444}. We use average pooling over the packed reservoir states to extract a single feature vector for the readout classifier (details in Appendix~\ref{app:classification_details}).

As detailed in Table~\ref{tab:reservoir_accuracy}, FRESCO achieves competitive predictive performance alongside substantial inference speedups over standard ESNs, reaching up to a factor of 26. The sole exception is the Libras dataset, where accuracy-driven model selection independently chose a small ESN ($\SNr = 256$) and a larger FRESCO ($\SNr = 2048$), resulting in comparable accuracy but no FRESCO speed advantage on this dataset. As expected, the timing critical-difference (CD) diagram (Figure~\ref{fig:cd-accuracy}, bottom) confirms the general FRESCO speed advantage. Crucially, the accuracy critical-difference analysis (Figure~\ref{fig:cd-accuracy}, top) places all three reservoir models in the same statistical group (based on the standard Wilcoxon procedure~\cite{demvsar2006statistical}), indicating that no aggregate accuracy difference was detected on this set of classification tasks. The per-dataset results indicate that the optimal FRESCO variant is task-dependent.

\begin{figure*}[t]
    \centering

    \begin{minipage}[t]{0.49\textwidth}
        \centering
        \vspace{0pt}
        \captionof{table}{Mean classification test accuracy with standard deviation across ten random initializations. Best results are highlighted in bold. The speedup factor compares the inference time of the highest-accuracy FRESCO model against ESN.}
        \label{tab:reservoir_accuracy}
        \vspace{1em}
        \scriptsize
        \setlength{\tabcolsep}{2.2pt}
        \renewcommand{\arraystretch}{0.95}
        \resizebox{\linewidth}{!}{%
\begin{tabular}{lcccr}
\toprule
Dataset & ESN & FRESCO (plain) & FRESCO (mix) & Speedup \\
\midrule
Adiac & $0.648_{\pm 0.005}$ & \textbf{0.686}$_{\bf\pm 0.017}$ & $0.609_{\pm 0.013}$ & $2.7\times$ \\
FordA & $0.892_{\pm 0.003}$ & $0.900_{\pm 0.009}$ & \textbf{0.908}$_{\bf\pm 0.006}$ & $2.3\times$ \\
JapaneseVowels & $0.975_{\pm 0.001}$ & $0.962_{\pm 0.005}$ & \textbf{0.988}$_{\bf\pm 0.004}$ & $9.3\times$ \\
Libras & \textbf{0.831}$_{\bf\pm 0.008}$ & $0.764_{\pm 0.011}$ & $0.778_{\pm 0.008}$ & $0.1\times$ \\
Lightning2 & $0.674_{\pm 0.009}$ & $0.711_{\pm 0.008}$ & \textbf{0.726}$_{\bf\pm 0.008}$ & $26.0\times$ \\
Lightning7 & $0.571_{\pm 0.029}$ & $0.664_{\pm 0.021}$ & \textbf{0.815}$_{\bf\pm 0.019}$ & $25.7\times$ \\
PEMS-SF & \textbf{0.750}$_{\bf\pm 0.024}$ & $0.747_{\pm 0.013}$ & $0.734_{\pm 0.007}$ & $2.8\times$ \\
ShapesAll & $0.796_{\pm 0.004}$ & \textbf{0.809}$_{\bf\pm 0.006}$ & $0.790_{\pm 0.005}$ & $2.0\times$ \\
Wafer & \textbf{0.996}$_{\bf\pm 0.000}$ & $0.993_{\pm 0.001}$ & $0.994_{\pm 0.001}$ & $5.1\times$ \\
Yoga & \textbf{0.857}$_{\bf\pm 0.005}$ & $0.809_{\pm 0.012}$ & $0.789_{\pm 0.005}$ & $7.4\times$ \\
\bottomrule
\end{tabular}
        }
    \end{minipage}
    \hfill
    \begin{minipage}[t]{0.47\textwidth}
        \centering
        \vspace{0pt}
        \captionof{figure}{Critical difference diagrams of model accuracy (top) and inference time (bottom), confirming statistically indistinguishable predictive performance but significantly faster inference speeds of the FRESCO models.}
        \label{fig:cd-accuracy}
                \includegraphics[
            width=\linewidth,
            keepaspectratio
        ]{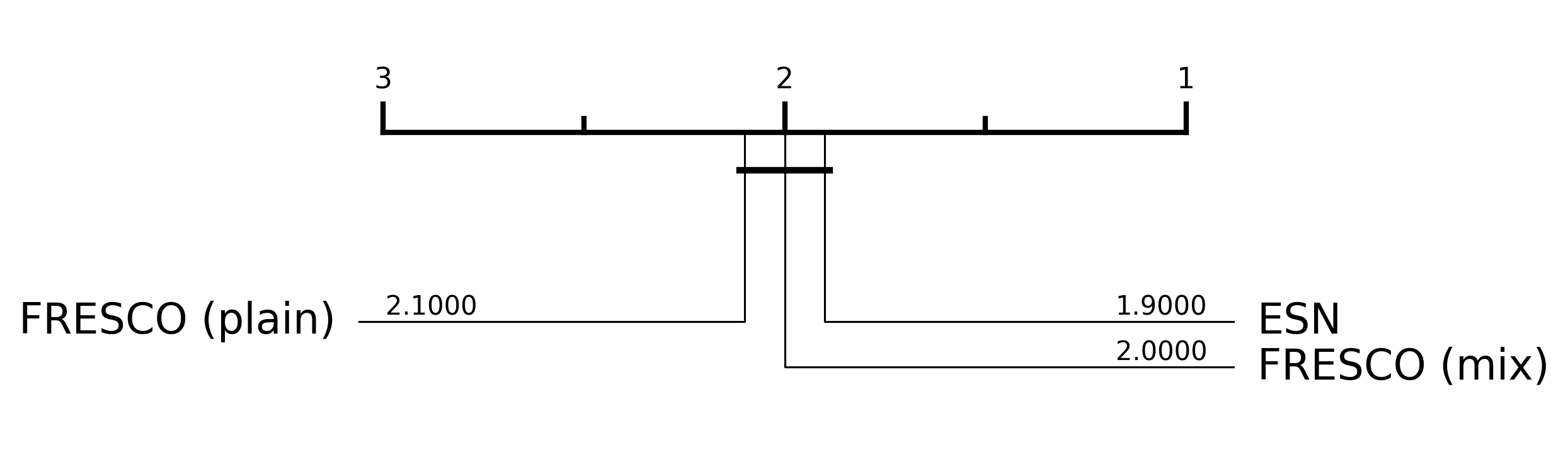}
        
        \vspace{0.5em} 

        \label{fig:cd-time}
        \includegraphics[
            width=\linewidth,
            keepaspectratio
        ]{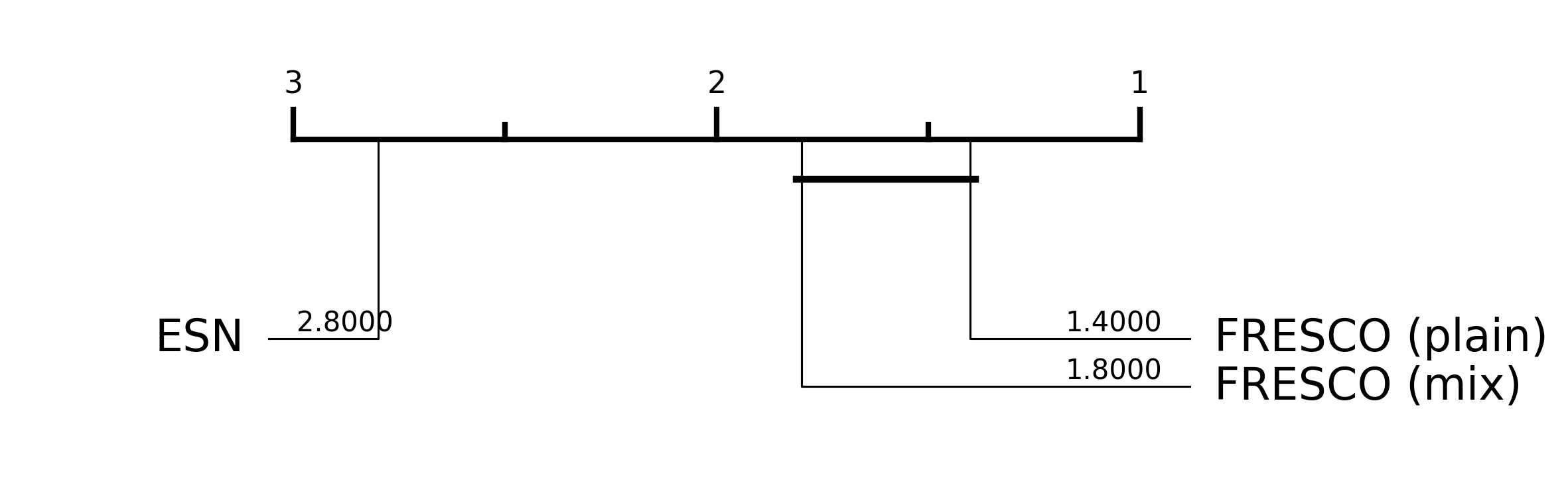}
    \end{minipage}

    \label{fig:accuracy-comparison}
\end{figure*}


\subsection{Long-horizon multivariate forecasting}
\label{sec:long-horizon_forecasting}

\begin{wrapfigure}{r}{0.55\linewidth}
\vspace{-0.5em}
\centering
\includegraphics[width=\linewidth]{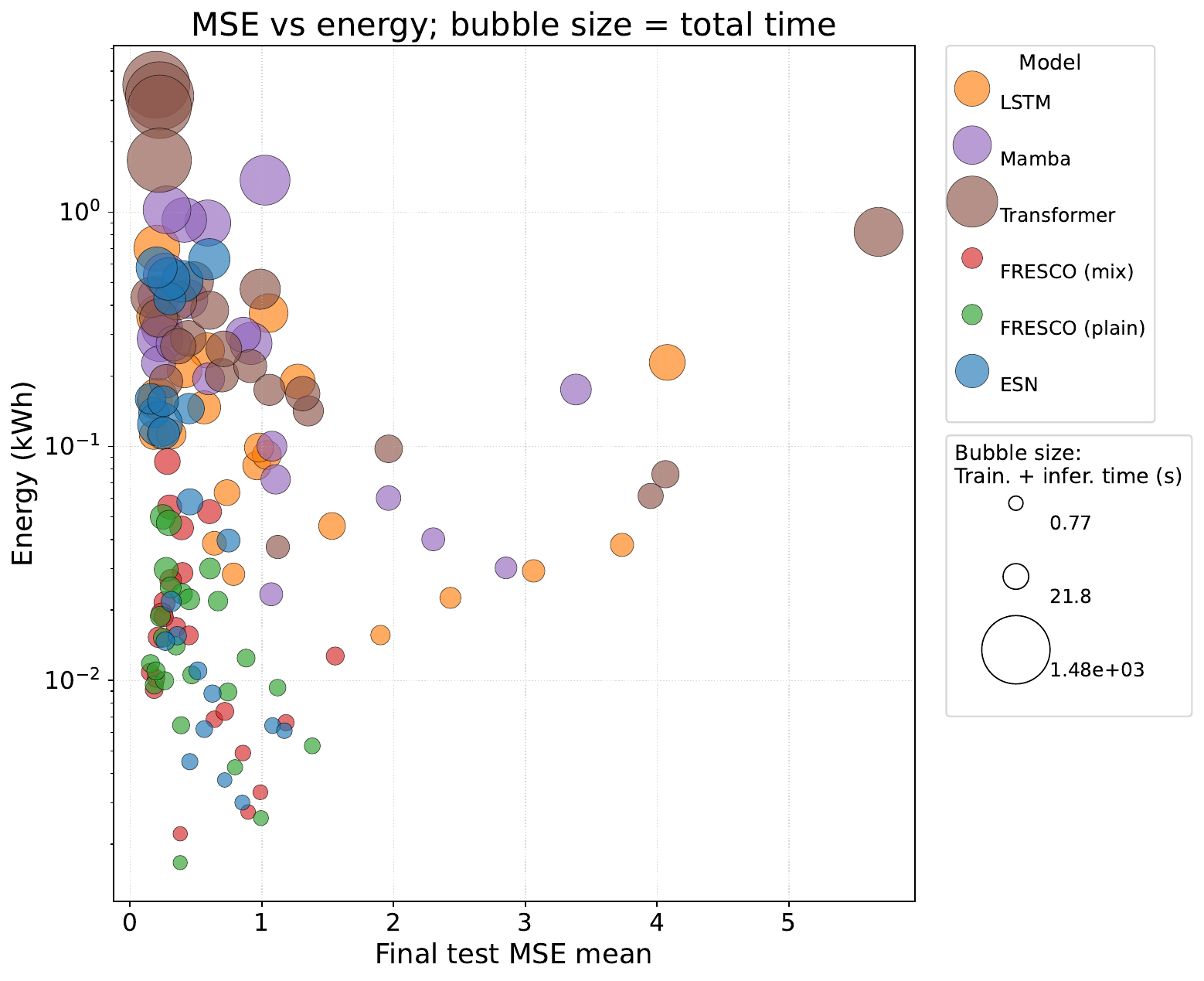}
\vspace{-0.75em}
\caption{Accuracy, energy, and runtime (training and inference) for long-horizon forecasting, showing FRESCO's favorable accuracy–efficiency trade-off among in-house implementations (values in Table~\ref{tab:anytime_log_results}).}
\label{fig:emissions}
\vspace{-1.5em}
\end{wrapfigure}

Finally, we benchmark FRESCO on multivariate long-horizon forecasting (ETT, Solar, Weather) with a fixed 96-step lookback and prediction horizons $H \in \{ 96, 192, 336, 720 \}$. We compare FRESCO and ESNs against three deep sequence baselines: LSTM \cite{hochreiter1997long}, Mamba \cite{gu2023mamba}, and Transformer \cite{vaswani2017attention}. Deep baselines are implemented as encoder-only one-shot forecasters, while reservoir models apply a ridge readout to the mean-pooled final state of the encoder window. Details on pre-processing, optimization, and deep learning baselines are given in Appendices~\ref{app:hypers_long_forecast}-\ref{app:forecasting_dl}. Figure~\ref{fig:emissions} illustrates the resulting accuracy-energy-time trade-off among our in-house PyTorch implementations, all evaluated under a common model-selection protocol. Data refer to the final selected configuration for each model on each dataset and prediction horizon (full numerical results in Appendix~\ref{app:additional_results}, Table~\ref{tab:anytime_log_results}). 

FRESCO occupies a highly favorable regime: it approaches and occasionally exceeds the predictive accuracy of the deep baselines while consuming up to three orders of magnitude less energy (measured in kWh). This reduction correlates directly with FRESCO's minimal training and inference times, demonstrating that \FDh reservoirs can preserve much of the accuracy of substantially more expensive sequence models at a fraction of the environmental cost. While standard ESNs also avoid backpropagation, their dense recurrent scaling often requires extremely large reservoirs to solve complex tasks, leading to substantial energy footprints. 

To provide additional predictive-performance context, Table~\ref{tab:mixfresco_sota} compares FRESCO against specialized, highly tuned forecasting architectures evaluated under identical 96-step lookback conditions \cite{liu2024itransformer}. This comparison is intentionally demanding. Despite relying on a generic five-hour hyperparameter search, fixed recurrent dynamics, and a closed-form readout, FRESCO is competitive on several dataset-horizon combinations with task-specific deep forecasting models. These results highlight a favorable accuracy-efficiency trade-off: FRESCO achieves competitive forecasting quality in many settings at substantially lower training and inference cost.
\begin{table*}[t]
\small
\centering
\caption{Forecasting results with selected models and state-of-the-art deep learning baselines taken from \cite{liu2024itransformer}. Best and second best are colored red and blue, respectively.}
\label{tab:mixfresco_sota}
\scriptsize
\setlength{\tabcolsep}{2.2pt}
\renewcommand{\arraystretch}{0.95}
\resizebox{\textwidth}{!}{%
\begin{tabular}{ll|cc|cc|cc|cc|cc|cc|cc}
\toprule
 & Models  & \multicolumn{2}{c|}{FRESCO (mix)}  & \multicolumn{2}{c|}{FRESCO (plain)}  & \multicolumn{2}{c|}{iTransformer}  & \multicolumn{2}{c|}{PatchTST}  & \multicolumn{2}{c|}{TimesNet}  & \multicolumn{2}{c|}{FEDformer}  & \multicolumn{2}{c}{Autoformer} \\
 &   & \multicolumn{2}{c|}{Ours}  & \multicolumn{2}{c|}{Ours}  & \multicolumn{2}{c|}{\cite{liu2024itransformer}}  & \multicolumn{2}{c|}{\cite{nie2023a}}  & \multicolumn{2}{c|}{\cite{wu2023timesnet}}  & \multicolumn{2}{c|}{\cite{pmlr-v162-zhou22g}}  & \multicolumn{2}{c}{\cite{wu2021autoformer}} \\
Dataset & Horizon  & MSE & MAE  & MSE & MAE  & MSE & MAE  & MSE & MAE  & MSE & MAE  & MSE & MAE  & MSE & MAE \\
\midrule
\multirow{4}{*}{\rotatebox[origin=c]{90}{ETTm1}} & 96 & 0.349 & 0.389 & 0.349 & 0.390 & \textcolor{blue}{0.334} & \textcolor{blue}{0.368} & \textcolor{red}{0.329} & \textcolor{red}{0.367} & 0.338 & 0.375 & 0.379 & 0.419 & 0.505 & 0.475 \\
 & 192 & 0.398 & 0.423 & 0.394 & 0.419 & 0.377 & 0.391 & \textcolor{red}{0.367} & \textcolor{red}{0.385} & \textcolor{blue}{0.374} & \textcolor{blue}{0.387} & 0.426 & 0.441 & 0.553 & 0.496 \\
 & 336 & 0.446 & 0.457 & 0.450 & 0.461 & 0.426 & 0.420 & \textcolor{red}{0.399} & \textcolor{red}{0.410} & \textcolor{blue}{0.410} & \textcolor{blue}{0.411} & 0.445 & 0.459 & 0.621 & 0.537 \\
 & 720 & 0.603 & 0.555 & 0.607 & 0.558 & 0.491 & 0.459 & \textcolor{red}{0.454} & \textcolor{red}{0.439} & \textcolor{blue}{0.478} & \textcolor{blue}{0.450} & 0.543 & 0.490 & 0.671 & 0.561 \\
\midrule
\multirow{4}{*}{\rotatebox[origin=c]{90}{ETTm2}} & 96 & 0.183 & 0.288 & 0.184 & 0.289 & \textcolor{blue}{0.180} & \textcolor{blue}{0.264} & \textcolor{red}{0.175} & \textcolor{red}{0.259} & 0.187 & 0.267 & 0.203 & 0.287 & 0.255 & 0.339 \\
 & 192 & 0.262 & 0.353 & 0.262 & 0.353 & 0.250 & \textcolor{blue}{0.309} & \textcolor{red}{0.241} & \textcolor{red}{0.302} & \textcolor{blue}{0.249} & \textcolor{blue}{0.309} & 0.269 & 0.328 & 0.281 & 0.340 \\
 & 336 & 0.391 & 0.450 & 0.387 & 0.448 & \textcolor{blue}{0.311} & \textcolor{blue}{0.348} & \textcolor{red}{0.305} & \textcolor{red}{0.343} & 0.321 & 0.351 & 0.325 & 0.366 & 0.339 & 0.372 \\
 & 720 & 0.719 & 0.652 & 0.743 & 0.664 & 0.412 & 0.407 & \textcolor{red}{0.402} & \textcolor{red}{0.400} & \textcolor{blue}{0.408} & \textcolor{blue}{0.403} & 0.421 & 0.415 & 0.433 & 0.432 \\
\midrule
\multirow{4}{*}{\rotatebox[origin=c]{90}{Solar}} & 96 & \textcolor{blue}{0.215} & 0.315 & 0.230 & 0.333 & \textcolor{red}{0.203} & \textcolor{red}{0.237} & 0.234 & \textcolor{blue}{0.286} & 0.250 & 0.292 & 0.242 & 0.342 & 0.884 & 0.711 \\
 & 192 & \textcolor{blue}{0.241} & 0.330 & 0.248 & 0.348 & \textcolor{red}{0.233} & \textcolor{red}{0.261} & 0.267 & \textcolor{blue}{0.310} & 0.296 & 0.318 & 0.285 & 0.380 & 0.834 & 0.692 \\
 & 336 & 0.300 & 0.379 & \textcolor{blue}{0.274} & 0.364 & \textcolor{red}{0.248} & \textcolor{red}{0.273} & 0.290 & \textcolor{blue}{0.315} & 0.319 & 0.330 & 0.282 & 0.376 & 0.941 & 0.723 \\
 & 720 & \textcolor{blue}{0.283} & 0.351 & 0.295 & 0.381 & \textcolor{red}{0.249} & \textcolor{red}{0.275} & 0.289 & \textcolor{blue}{0.317} & 0.338 & 0.337 & 0.357 & 0.427 & 0.882 & 0.717 \\
\midrule
\multirow{4}{*}{\rotatebox[origin=c]{90}{Weather}} & 96 & \textcolor{red}{0.153} & 0.232 & \textcolor{blue}{0.155} & 0.235 & 0.174 & \textcolor{red}{0.214} & 0.177 & \textcolor{blue}{0.218} & 0.172 & 0.220 & 0.217 & 0.296 & 0.266 & 0.336 \\
 & 192 & \textcolor{blue}{0.197} & 0.278 & \textcolor{red}{0.196} & 0.277 & 0.221 & \textcolor{red}{0.254} & 0.225 & \textcolor{blue}{0.259} & 0.219 & 0.261 & 0.276 & 0.336 & 0.307 & 0.367 \\
 & 336 & \textcolor{blue}{0.254} & 0.327 & \textcolor{red}{0.250} & 0.322 & 0.278 & \textcolor{red}{0.296} & 0.278 & \textcolor{blue}{0.297} & 0.280 & 0.306 & 0.339 & 0.380 & 0.359 & 0.395 \\
 & 720 & \textcolor{red}{0.308} & 0.367 & \textcolor{blue}{0.309} & 0.365 & 0.358 & \textcolor{red}{0.347} & 0.354 & \textcolor{blue}{0.348} & 0.365 & 0.359 & 0.403 & 0.428 & 0.419 & 0.428 \\
\bottomrule
\end{tabular}%
}
\end{table*}

\section{Conclusion}

FRESCO demonstrates that the computational bottlenecks of dense recurrent layers can be elegantly bypassed by shifting computation to the frequency domain while avoiding
domain-shift overheads. Achieving strictly linear scaling requires constraining recurrent connectivities to circular convolutions, which reduces the structural degrees of freedom relative to a fully unstructured reservoir. The effect on predictive performance is task-dependent: while cross-frequency mixing inexpensively recovers non-linear processing capacity, our memory capacity analysis shows that the trade-off between linear recall and non-linear expressivity is governed by the choice of FRESCO variant and spectral topology.

Empirically, FRESCO shows highly favorable scaling for both input embedding and recurrent updates, surpassing the inference speeds of standard ESNs. Beyond sheer speed, it achieves competitive predictive accuracy relative to heavily parameterized deep sequence models on complex, long-horizon multivariate forecasting while consuming up to three orders of magnitude less energy. Ultimately, these results demonstrate that the mathematical elegance of the frequency domain offers a scalable, energy-efficient alternative to more resource-intensive sequence modeling approaches. 
While these efficiency gains are substantial, the shift to fixed, circulant frequency-domain dynamics imposes specific architectural requirements. A comprehensive discussion of these limitations is provided in Appendix~\ref{app:limitations}.

Because the optimal choice between isolated frequency dynamics and cross-frequency mixing proved task-dependent, investigating novel frequency-domain non-linearities remains a promising avenue to further enhance the expressivity of the FRESCO reservoir. In addition, building on our current efficiency gains, future work will exploit the dimensional zero-padding embedding to scale to even higher-dimensional data, such as 2D spatial inputs.

\section*{Acknowledgments}
DK was partly funded by the Ministry of Culture and Science of the State of North Rhine-Westphalia under project SAIL (NW21-059A).
AC was partly funded by EMERGE, a project funded by the European Innovation Council (prj. code 101070918).

\section*{Author Contributions}
KS: concept, theory, proofs, writing. XR: theory, proofs, implementation, benchmarking, writing. AC: experimental design, proofs, implementation, benchmarking, writing. DK: supervision, writing. CG: supervision, proofs, writing.

\bibliography{references}

\appendix

\section{Proofs and theoretical analysis}
\label{app:theory}

\subsection{Contribution: Input embedding by dimensional zero-padding}
\label{app:dim-zero-pad}

In the FRESCO architecture, the primary computational advantage stems from exploiting the circular convolution theorem, which replaces dense $\mathcal{O}(\SNr^2)$ matrix-vector multiplications with highly efficient $\mathcal{O}(\SNr)$ element-wise operations in the \FD. However, to execute these point-wise recurrent updates, an external real-valued input vector $\Sx{} \in \mathbb{R}^{\SNx}$ must first be transformed into the potentially substantially larger \FDh representation of the reservoir of size $\SNr$ (where $\SNx < \SNr$). 

Because element-wise operations require strictly matching dimensions, the low-dimensional input $\Sx{}$ must be mathematically expanded and transformed to match the full dimensionality of the \FDh reservoir state. The critical requirement here is to ensure that the involved Fast Fourier Transform (FFT) from the spatial to the frequency domain does not introduce a new computational bottleneck. If mapping the input into the \FD scales poorly, this overhead could offset the algorithmic efficiency gained during the recurrent updates.

Specifically, an optimal input embedding mechanism must seamlessly combine both the random projection by the input weights and the Fourier transformation, ultimately yielding a compatible, full-resolution complex-valued representation. In the following Sections~\ref{app:embedding-option-1} to~\ref{app:embedding-option-3}, we analyze three different approaches to computing a \FDh embedding. In the benchmark Section~\ref{app:embedding-benchmark}, we will prove the effectiveness of the proposed dimensional zero-padding technique compared to the other two options.

\subsubsection{Option 1: Complex dense layer embedding}
\label{app:embedding-option-1}

In a standard Echo State Network (ESN), the contribution of a real-valued input vector $\Sx{} \in \mathbb{R}^{\SNx{}}$ to a higher-dimensional reservoir $\Sr{} \in \mathbb{R}^{\SNr}$ (where $\SNr > \SNx$) is computed via the real-valued dense projection of Eq.~\eqref{eq:embedding_real_dense_layer}, using an input weight matrix $\SWx \in \mathbb{R}^{\SNr \times \SNx}$. 

\paragraph{Dense layer embedding (baseline)}
\begin{equation}
\SWx \cdot \Sx{} \label{eq:embedding_real_dense_layer}
\end{equation}

To natively embed this input into the \FD, we can formulate an analogous \emph{complex dense layer embedding}. By exploiting the linearity of the Fourier transform, the effective \FDh input weights can be analytically derived by left-multiplying the spatial weights by the Discrete Fourier Transform (DFT) matrix $\mathbf{F} \in \mathbb{R}^{\SNr \times \SNr}$, such that $\FWx = \mathbf{F} \cdot \SWx \in \mathbb{R}^{\SNr \times \SNx}$.

Because the spatial projection $\SWx \cdot \Sx{}$ is strictly real-valued, its resulting \FDh representation naturally exhibits Hermitian symmetry. We can eliminate this redundant information and computational overhead by restricting $\mathbf{F}$ to the operator of a real-to-complex FFT (RFFT). Assuming an even reservoir dimension $\SNr$, this reduces the number of rows of $\FWx$ from $\SNr$ to $\frac{\SNr}{2} + 1$, leading to $\FWx \in \mathbb{C}^{(\frac{\SNr}{2} + 1) \times N_x}$. The embedding of the input $\Sx{}$ is then executed as the single complex matrix-vector product of Eq.~\eqref{eq:embedding_complex_dense_layer}.

\paragraph{Complex dense layer embedding}
\begin{equation}
\FWx \cdot \Sx{} \label{eq:embedding_complex_dense_layer}
\end{equation}

This is the direct algebraic translation of a real valued dense layer embedding into the \FD.

\subsubsection{Option 2: Embedding by zero-padding}
\label{app:embedding-option-2}

While the complex dense layer embedding explained in the previous Section~\ref{app:embedding-option-1} naturally translates the spatial projection into the \FD, it fundamentally relies on a dense matrix-vector multiplication with a time complexity of $\mathcal{O}(\SNr \cdot \SNx)$. For larger input dimensions $\SNx$, this operation risks becoming a computational bottleneck compared to directly applying a FFT on the input, followed by element-wise operations of the circular convolution theorem.

To be able to directly apply an FFT to the input, we need to pad the input vector $\Sx{}$ with $\SNr - \SNx$ zeros, yielding a padded vector $\Sxzp \in \mathbb{R}^{\SNr}$. This zero-padded vector is subsequently transformed into the \FD using an $\SNr$-point FFT, as illustrated in Figure~\ref{fig:1d_padding}. 

\begin{tikzpicture}
    \begin{scope}[shift={(0,0)}, scale=0.5]
        \node [fit={(0,0) (3,1)}, draw = black!60, fill=gray!10, inner sep=-0.1cm, thick, pin = north:$\Sx{}$] {};
        \draw[help lines, draw] (0,0) grid (3,1);
    \end{scope}
    
    \begin{scope}[shift={(3,0)}, scale=0.5]
        \node at (-1.5,0.5) {$\xrightarrow[]{pad}$};
        \node [fit={(0,0) (3,1)}, draw = black!60, fill=gray!10, inner sep=-0.1cm, thick, pin = north:$\Sx{}$] {};
        \node [fit={(0,0) (8,1)}, inner sep=0.1cm, pin = north:$\Sxzp$] {};
        \draw[help lines, draw] (0,0) grid (8,1);
        \node foreach \x in {3.5,...,7.5} [] at (\x,0.5) {0};
    \end{scope}
    
    \begin{scope}[shift={(8.5,0)}, scale=0.5]
        \node at (-1.5,0.5) {$\xrightarrow[]{\fourier{\cdot}}$};
        \node [fit={(0,0) (8,1)}, draw = black!60, fill=gray!10, inner sep=-0.1cm, thick, pin = north:$\Fxzp \text{$=$} \fourier{\Sxzp}$] {};
        \draw[help lines, draw] (0,0) grid (8,1);
    \end{scope}
\end{tikzpicture}
\begin{figure}[h]
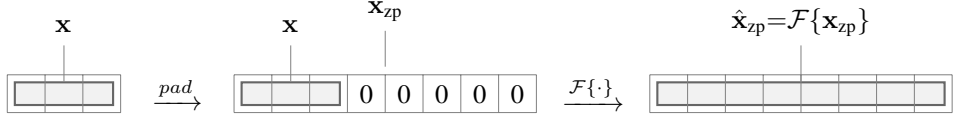

    \centering
    \caption{Standard 1D zero-padding: The input vector $\Sx{}$ is padded with zeros to match the reservoir size before computing the frequency domain representation $\fourier{\Sxzp}$. This requires an $\SNr$-point DFT.}
    \label{fig:1d_padding}
\end{figure}

Since, again, we have a real input vector, we can apply the more efficient RFFT to end up with a \FDh representation of the input given by $\Fxzp \in \mathbb{C}^{\frac{\SNr}{2} + 1}$ for even $\SNr$.
The final input embedding is then efficiently computed via the element-wise product of Eq.~\eqref{eq:embedding_zeropadding}, where $\Fwx \in \mathbb{C}^{\frac{\SNr}{2} + 1}$ represents the \FD input weight vector.

\paragraph{Embedding by zero-padding}
\begin{equation}
\Fwx \odot \Fxzp = \Fwx \odot\fourier{\Sxzp}\label{eq:embedding_zeropadding}
\end{equation}

Although this strategy successfully replaces the $\mathcal{O}(\SNr \cdot \SNx)$ matrix multiplication with an $\mathcal{O}(\SNr)$ element-wise product, it still requires the $\mathcal{O}(\SNr \log \SNr)$ operation of a full $\SNr$-point RFFT to be executed for every embedded input.

\subsubsection{Option 3: Embedding by dimensional zero-padding}
\label{app:embedding-option-3}

To circumvent the bottleneck of a full $\SNr$-point RFFT in the zero-padding approach presented in the previous Section~\ref{app:embedding-option-2}, we introduce the concept of dimensional zero-padding. Rather than padding the one-dimensional input $\Sx{}$ along its existing dimension, we embed the input into a two-dimensional space and apply zero-padding along the newly added dimension. This allows $\Sx{}$ to act like a two-dimensional convolution kernel applied to the input weights - then also represented two-dimensional.
Figure~\ref{fig:dimensional_padding} provides an intuition for this process when mapping a 1D vector into a 2D matrix. As formally proved in Section~\ref{app:proof-dim-zero-pad}, dimensional zero-padding allows us to construct an exact high-dimensional frequency representation of the input without ever computing a $\SNr$-point DFT. Instead, the computation of the full two-dimensional DFT of $\SXzp$ depends only on the $\SNx$-point DFT of the input $\Sx{}$.

\begin{figure}[h]
    \centering
    \begin{tikzpicture}
    \begin{scope}[shift={(0,0)}]
        \node [fit={(0,0) (1,4)}, draw = black!60, fill=gray!10, inner sep=-0.15cm, thick, pin = south:$\Sx{}$] {};
        \node [fit={(0,0) (5,4)}, thick, pin = $\SXzp$] {};
        \draw[help lines, draw] (0,0) grid (5,4);
        \node foreach \x in {1.5,...,4.5} foreach \y in {0.5,...,3.5} [] at (\x,\y) {0};
    \end{scope}

    \begin{scope}[shift={(7,0)}]
        \node [fit={(0,0) (1,4)}, draw = black!60, fill=gray!10, inner sep=-0.15cm, thick, pin = {[] south:$\Fx{} = \fourier{\Sx{}}\quad\quad$}] {};
        \node foreach \x in {1,...,3} [draw = black!60, fit={(\x,0) (\x+1,4)}, fill=gray!10, inner sep=-0.15cm, thick, pin = {[yshift = -0.5em] south:$\dots$}] {};
        \node [fit={(4,0) (5,4)}, draw = black!60, fill = gray!10, inner sep=-0.15cm, thick, pin = {[] south:$\Fx{}$}] {};
        \node [fit={(0,0) (5,4)}, thick, pin = $\FXzp \text{$=$} \fourier{\SXzp}$] {};
        \draw[help lines, draw] (0,0) grid (5,4);
    \end{scope}
    
    \node at (6,2) {$\begin{array}{c} \fourier{\cdot} \\ \longrightarrow  \end{array}$};
\end{tikzpicture}
    \caption{Dimensional zero-padding: A 1D input vector $\Sx{}$ is embedded into a 2D space by padding zeros along the second dimension (columns). The resulting 2D DFT consists simply of repeated columns of the much smaller 1D transform $\Fx{} = \fourier{\Sx{}}$. In practice, neither $\SXzp$ nor $\fourier{\SXzp}$ need to be created explicitly, since the result $\Fx{}$ can be broadcast into the matrix representation of the reservoir.}
    \label{fig:dimensional_padding}
\end{figure}

The underlying mechanics of this dimensional zero-padding can be intuitively understood through the separability property of the multidimensional DFT. Consider the 2D case illustrated in Figure~\ref{fig:dimensional_padding}, where a 1D input vector $\Sx{}$ is embedded into the first column of a 2D matrix, with all remaining columns padded with zeros. Because the 2D DFT is separable, it can be computed sequentially: first by applying 1D DFTs along the columns, and then along the rows. 

The column-wise DFT yields $\Fx{} = \fourier{\Sx{}}$ in the first column, while the transforms of all other zero-padded columns naturally remain zero. Subsequently, when the 1D DFTs are applied along the rows, each row acts as a 1D discrete impulse signal---a single non-zero value at the zeroth index followed entirely by zeros. Since the Fourier transform of such an impulse is a constant sequence, the initial non-zero value is uniformly broadcast across the entire row. Consequently, the final 2D \FDh representation consists solely of identically repeated copies of $\Fx{} = \fourier{\Sx{}}$, effectively bypassing the need to compute the large 2D transform. 

Since $\FXzp$ is now a regular 2D \FDh representation, we can apply the circular convolution theorem in two dimensions by performing the complex element-wise product $\FWx \odot \FXzp$ which we more shortly write in Eq.~\eqref{eq:embedding_dim_zeropadding_app} as $\FWx \odot \Fx{}$ with the convention that $\Fx{}$ gets broadcast over the columns of $\FWx$. 

\paragraph{Embedding by dimensional zero-padding}
\begin{equation}
\FWx \odot \FXzp = \FWx \odot \Fx{} = \FWx \odot  \fourier{\Sx{}} \label{eq:embedding_dim_zeropadding_app}
\end{equation}

Compared to the standard zero-padding approach, dimensional zero-padding reduces the FFT complexity from $\mathcal{O}(\SNr \log \SNr)$ to $\mathcal{O}(\SNx \log \SNx)$.

Note: Although the primary focus of this paper is on the incorporation of 1D inputs into 2D reservoirs \FDh, the dimensional zero-padding property is generalizable to higher dimensions. The mathematical equivalence holds for tensors of arbitrary dimensions (c.f. proof in Section~\ref{app:proof-dim-zero-pad}). This generality provides a highly scalable framework for applications where inputs are more naturally represented in multidimensional forms, such as embedding 2D image patches or visual receptive fields into a 3D reservoir.

\subsubsection{Experiment: Timing of input embedding}
\label{app:embedding-benchmark}

This experiment evaluates the computational efficiency of the various \FDh input embedding strategies given in Eqs.~\eqref{eq:embedding_complex_dense_layer}, \eqref{eq:embedding_zeropadding}, and~\eqref{eq:embedding_dim_zeropadding_app}. As a baseline, we benchmark these methods against the standard real-valued dense layer embedding typically employed in standard (\SDh) ESNs, as given in Eq.~\eqref{eq:embedding_real_dense_layer}. These scaling benchmarks report the minimum execution time across trials to isolate the inherent algorithmic complexity from incidental system noise\footnote{This is specific to the microbenchmarks in Figures~\ref{fig:FRESCO}, \ref{fig:bench_input_embedding} and \ref{fig:bench_readout}; all other experiments throughout the paper report mean and standard deviation across multiple random initializations.}. Runtimes encompass the computation of a single embedding of an input vector $\Sx{}$ using Eqs.~\eqref{eq:embedding_real_dense_layer} to~\eqref{eq:embedding_dim_zeropadding_app} for different input sizes $\SNx$ and reservoir sizes $\SNr$, with $\SNr = 256 \SNx$. 

As illustrated in Figure~\ref{fig:bench_input_embedding}, the proposed dimensional zero-padding approach demonstrates significant computational advantages over both the alternative \FDh options and the standard ESN baseline. These results confirm that mapping \SDh inputs into the \FD does not introduce a computational bottleneck for the FRESCO architecture. Furthermore, the dimensional zero-padding method exhibits vastly superior asymptotic scaling behavior for large input- and reservoir sizes compared to standard ESN architectures. 

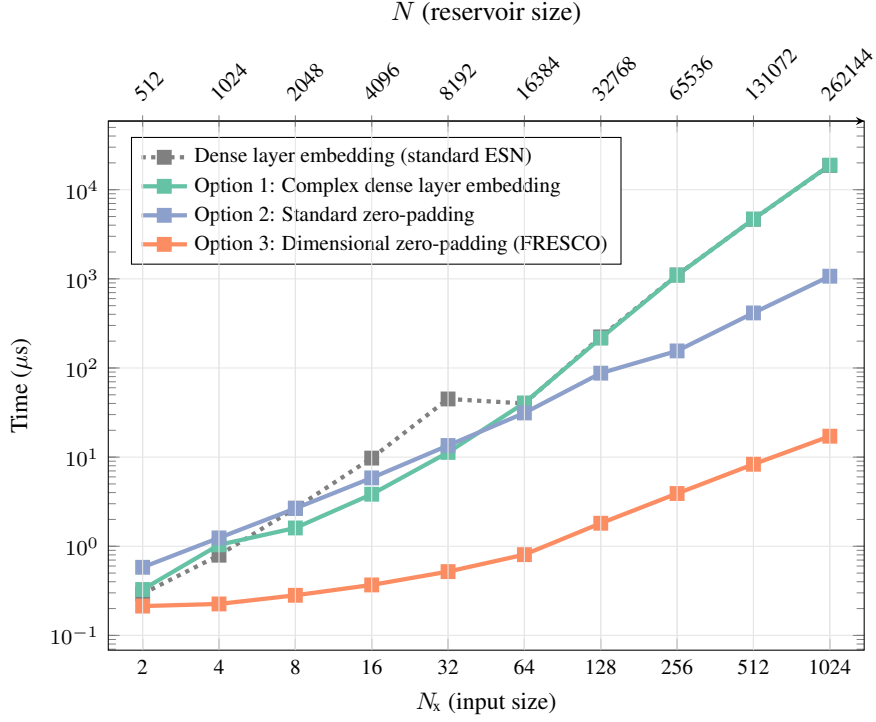
\begin{figure}[htbp]
    \centering
    \begin{tikzpicture}[remember picture] 

    \def\plotWidth{10cm}
    \def\plotHeight{7cm}
    \def\timediv{1000}
    \def\enlargefact{0.05}

    \pgfplotsset{
        my plot layout/.style={
            width=\plotWidth,
            height=\plotHeight,
            scale only axis,
            log basis x=2,
            xtick=data,
            grid=both,
            minor tick num=0, 
            grid style={gray!20},
            yminorgrids=false,
            label style={font=\small},
            tick label style={font=\footnotesize},
            legend style={font=\footnotesize},
        }
    }

    \pgfplotstableread[col sep=comma]{data/data_bench_input_embedding.csv}\datatable

    \begin{axis}[
        my plot layout,
        xmode=log, ymode=log,
        name=mainaxis,          
        xlabel={$\SNx$ (input size)},
        ylabel={Time ($\mu$s)},
        legend pos=north west,
        legend cell align=left,
        enlarge x limits=\enlargefact, 
        xticklabels from table={\datatable}{Nx}   
    ]
        \addplot[color=gray, dotted, mark=square*, mark options={solid}, ultra thick] table [
            x=Nx, y expr={\thisrow{time_matmul_sd_ns}/\timediv}
        ] {\datatable};
        \addlegendentry{Dense layer embedding (standard ESN)}
        
        \addplot[color=Set2-A, mark=square*, ultra thick] table [
            x=Nx, y expr={\thisrow{time_matmul_fd_ns}/\timediv}
        ] {\datatable};
        \addlegendentry{Option 1: Complex dense layer embedding}
        
        \addplot[color=Set2-C, mark=square*, ultra thick] table [
            x=Nx, y expr={\thisrow{time_zeropad_nrm_ns}/\timediv}
        ] {\datatable};
        \addlegendentry{Option 2: Standard zero-padding}

        \addplot[color=Set2-B, mark=square*, ultra thick] table [
            x=Nx, y expr={\thisrow{time_zeropad_dim_ns}/\timediv}
        ] {\datatable};
        \addlegendentry{Option 3: Dimensional zero-padding (FRESCO)}
    \end{axis}

    \begin{axis}[
        my plot layout,
        xmode=log, ymode=log,    
        at={(mainaxis.south west)},
        anchor=south west,      
        axis y line=none,
        axis x line=top,
        enlarge x limits=\enlargefact,
        xlabel={$\SNr$ (reservoir size)},
        xticklabels from table={\datatable}{N},   
        xticklabel style={rotate=45, anchor=south west} 
    ]
        \addplot[draw=none, forget plot] table [
            x=Nx, y expr={\thisrow{time_zeropad_dim_ns}/\timediv}
        ] {\datatable};
    \end{axis}
\end{tikzpicture}
    \caption[Embedding benchmark]{Comparison of minimal computational times (lower is better) for different input embedding methods across increasing input sizes $\SNx$ and reservoir sizes $\SNr$. A fixed scaling ratio of 256 reservoir neurons per input feature is maintained ($\SNr = 256 \SNx$). The proposed embedding via dimensional zero-padding exhibits highly efficient scaling behavior, outperforming standard \SDh ESN dense layers by orders of magnitude at larger network scales. }
    \label{fig:bench_input_embedding}
\end{figure}

\subsubsection{Proof: Dimensional zero-padding}
\label{app:proof-dim-zero-pad}

We formalize this generalized property for arbitrary dimensions in the following theorem and proof.

\begin{theorem}[Dimensional zero-padding\footnotemark{}]
\label{thm:dimensional_padding}
\footnotetext{Notational Convention: For mathematical convenience and consistency with standard signal processing literature, all spatial and frequency domain tensors, as well as the DFT summations in this section, are assumed to be $0$-indexed.}
Let $\mathbf{x}$ be an $n_x$-dimensional tensor over $\mathbb{R}$ or $\mathbb{C}$ with dimensions $N_1 \times N_2 \times \dots \times N_{n_x}$. Let $\mathbf{r}$ be an $n_r$-dimensional tensor over the same field, where $n_r > n_x$, sharing the exact dimensions of $\mathbf{x}$ across its first $n_x$ axes. 

Let the elements of $\mathbf{r}$ be defined by embedding $\mathbf{x}$ such that all entries outside the $0$-th spatial index of the newly added dimensions are zero:
\begin{equation}
    \mathbf{r}[i_1, \dots, i_{n_x}, i_{n_x+1}, \dots, i_{n_r}] = \mathbf{x}[i_1, \dots, i_{n_x}] \cdot \delta(i_{n_x+1}) \cdots \delta(i_{n_r})
\end{equation}
where $i_d \in \{0, 1, \dots, N_d-1\}$ denotes the $0$-indexed spatial coordinate along axis $d$, and $\delta(\cdot)$ is the Kronecker delta function ($\delta(0) = 1$, and $0$ otherwise).

Then, the $n_r$-dimensional DFT of $\mathbf{r}$, denoted as $\fourier{\mathbf{r}}$, perfectly replicates the lower-dimensional DFT of $\mathbf{x}$ across all frequency indices of the padded dimensions:
\begin{equation}
    \fourier{\mathbf{r}}[k_1, \dots, k_{n_x}, k_{n_x+1}, \dots, k_{n_r}] = \fourier{\mathbf{x}}[k_1, \dots, k_{n_x}]
\end{equation}
where $k_d$ represents the frequency index along axis $d$.
\end{theorem}

\begin{proof}
By definition, the $n_r$-dimensional DFT of the padded tensor $\mathbf{r}$, evaluated at the frequency indices $(k_1, \dots, k_{n_r})$, is given by the multi-dimensional summation:
\begin{equation}
    \fourier{\mathbf{r}}[k_1, \dots, k_{n_r}] = \sum_{i_1=0}^{N_1-1} \cdots \sum_{i_{n_r}=0}^{N_{n_r}-1} \mathbf{r}[i_1, \dots, i_{n_r}] \exp\left(-i 2\pi \sum_{d=1}^{n_r} \frac{i_d k_d}{N_d} \right).
\end{equation}

Substituting the definition of the elements of $\mathbf{r}$ into this equation yields:
\begin{align}
    \fourier{\mathbf{r}}[k_1, \dots, k_{n_r}] &= \sum_{i_1=0}^{N_1-1} \cdots \sum_{i_{n_r}=0}^{N_{n_r}-1} \Bigg( \mathbf{x}[i_1, \dots, i_{n_x}] \prod_{j=n_x+1}^{n_r} \delta(i_j) \Bigg) \exp\left(-i 2\pi \sum_{d=1}^{n_r} \frac{i_d k_d}{N_d} \right).
\end{align}

We now apply the sifting property of the Kronecker delta function. For all padded dimensions $j \in \{n_x+1, \dots, n_r\}$, the terms in the summation are non-zero if and only if $i_j = 0$. Consequently, the summations over these higher dimensions collapse entirely to the single index $0$. 

When $i_j = 0$ for all $j > n_x$, their corresponding components in the complex exponent vanish because $\frac{0 \cdot k_j}{N_j} = 0$, meaning $e^0 = 1$. This simplifies the equation to:
\begin{equation}
    \fourier{\mathbf{r}}[k_1, \dots, k_{n_r}] = \sum_{i_1=0}^{N_1-1} \cdots \sum_{i_{n_x}=0}^{N_{n_x}-1} \mathbf{x}[i_1, \dots, i_{n_x}] \exp\left(-i 2\pi \sum_{d=1}^{n_x} \frac{i_d k_d}{N_d} \right).
\end{equation}

We recognize this remaining summation as the exact definition of the $n_x$-dimensional DFT of the original tensor $\mathbf{x}$. The frequency indices belonging to the padded dimensions ($k_{n_x+1}, \dots, k_{n_r}$) have completely dropped out of the expression. Therefore, we arrive at:
\begin{equation}
    \fourier{\mathbf{r}}[k_1, \dots, k_{n_r}] = \fourier{\mathbf{x}}[k_1, \dots, k_{n_x}],
\end{equation}
which concludes the proof.
\end{proof}

\subsection{Contribution: Packed \FDh readout}
\label{app:packed_readout}

To fully realize the performance benefits of a \FDh architecture, the abstract mathematical operations must be grounded in hardware-efficient data structures. The core challenge lies in designing a suitable complex 2D \FDh matrix and memory representation for the \FDh reservoir state $\FR{}$ (and the equally structured weights $\FWr, \FWx$, and biases $\FB$) that simultaneously allows to utilize dimensional zero-padding, while also enabling a fast, vectorizable readout of the reservoir's independent degrees of freedom. 

To allow for dimensional zero-padding (detailed in Section~\ref{app:embedding-option-3}), the only requirement is that the \FDh reservoir state is represented as a complex matrix $\FR{}$ in which one axis reflects the size of the RFFT $\Fx{}$ of the embedded input (c.f. Figure~\ref{fig:dimensional_padding}). In this way, Eq.~\eqref{eq:embedding_dim_zeropadding_app} can be used for input embedding.

During the recurrent operation of the reservoir, its state must be continuously read out to compute the matrix-vector product of a dense linear readout layer. Unfortunately, the complex elements of the matrix $\FR{}$ are not perfectly suited to represent the reservoir state since it contains redundant information from Hermitian symmetries. These redundancies artificially increase the number of values to be processed during readout and can complicate the regularization dynamics in the training of the readout weights\footnotemark{}. 
\footnotetext{In ridge regression, penalty terms may distribute weights across perfectly coupled variables, which can subtly alter the effective regularization applied to the underlying independent degrees of freedom.} 
As established in Proposition~\ref{prop:packed_dft_unitary}, the redundant complex frequency state $\FR{}$ of the reservoir can indeed be perfectly represented by a packed (non-redundant) real-valued vector 

\begin{equation}
    \Fp{} = \pack{\FR{}} \in \mathbb{R}^\SNr, 
    \label{eq:packing}
\end{equation}

with the unitary packing operator $\pack{\cdot}$, so that the linear readout can be performed using

\begin{equation}
\Sy{} = \FWp \cdot \Fp{}
\label{eq:packed_readout}
\end{equation}

where $\FWp \in \mathbb{R}^{\SNx \times \SNr}$ is the trained readout weight matrix. This ensures computational parity with standard spatial-domain reservoirs in terms of reservoir degrees of freedom ($\SNr$) as well as learned parameters of the readout weight ($\SNx \cdot \SNr$).

To achieve maximum hardware efficiency during inference, the computation of Eqs.~\eqref{eq:packing} and~\eqref{eq:packed_readout} should ideally rely on bulk memory operations, allowing for the use of highly optimized dense linear algebra routines (e.g., BLAS \texttt{gemv}). Ideally, this requires the complex matrix $\FR{}$ to be directly reinterpreted as a contiguous 1D array of non-redundant real floating-point numbers, exploiting the native interlaced storage of complex values ($[\Re_0, \Im_0, \Re_1, \Im_1, \dots]$) without the need to actively allocate and construct $\Fp{}$ in memory.

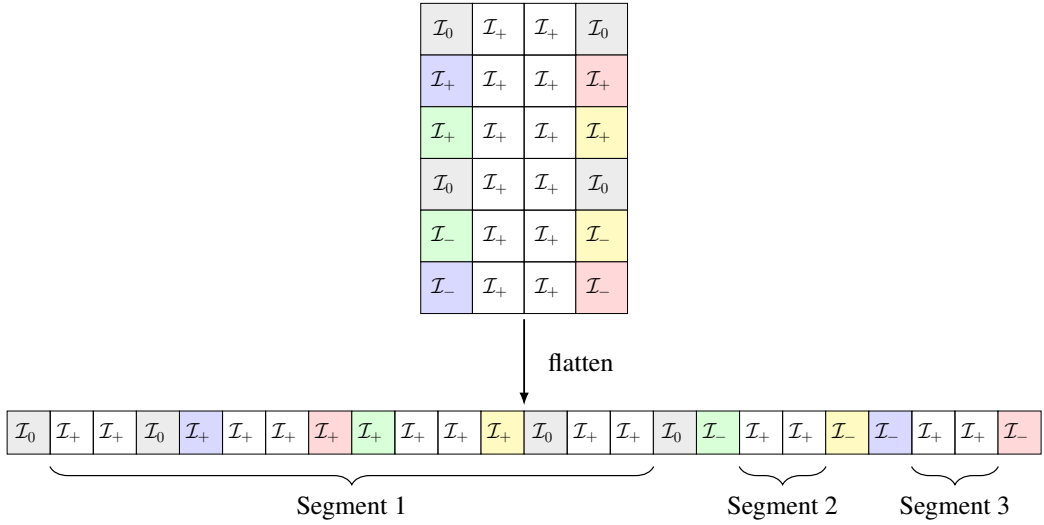
\begin{figure}[h]
    \centering
    \resizebox{\linewidth}{!}{
    \begin{tikzpicture}[
    cell/.style={
        draw, 
        minimum size=12mm, 
        anchor=center, 
        align=center, 
        inner sep=0pt, 
        font=\Large 
    },
    label font/.style={
        font=\LARGE
    }
]
    \matrix[
        matrix of nodes,
        nodes={cell},
        column sep=-\pgflinewidth, 
        row sep=-\pgflinewidth
    ] (mat) at (0,0) {
        |[fill=gray!15]| $\mathcal{I}_0$ & $\mathcal{I}_+$ & $\mathcal{I}_+$ & |[fill=gray!15]| $\mathcal{I}_0$ \\
        |[fill=blue!15]| $\mathcal{I}_+$ & $\mathcal{I}_+$ & $\mathcal{I}_+$ & |[fill=red!15]| $\mathcal{I}_+$ \\
        |[fill=green!15]| $\mathcal{I}_+$ & $\mathcal{I}_+$ & $\mathcal{I}_+$ & |[fill=yellow!30]| $\mathcal{I}_+$ \\
        |[fill=gray!15]| $\mathcal{I}_0$ & $\mathcal{I}_+$ & $\mathcal{I}_+$ & |[fill=gray!15]| $\mathcal{I}_0$ \\
        |[fill=green!15]| $\mathcal{I}_-$ & $\mathcal{I}_+$ & $\mathcal{I}_+$ & |[fill=yellow!30]| $\mathcal{I}_-$ \\
        |[fill=blue!15]| $\mathcal{I}_-$ & $\mathcal{I}_+$ & $\mathcal{I}_+$ & |[fill=red!15]| $\mathcal{I}_-$ \\
    };

    \matrix[
        matrix of nodes,
        nodes={cell, minimum size=10mm}, 
        column sep=-\pgflinewidth,
        below=2cm of mat 
    ] (vec) {
        |[fill=gray!15]| $\mathcal{I}_0$ & $\mathcal{I}_+$ & $\mathcal{I}_+$ & |[fill=gray!15]| $\mathcal{I}_0$ &
        |[fill=blue!15]| $\mathcal{I}_+$ & $\mathcal{I}_+$ & $\mathcal{I}_+$ & |[fill=red!15]| $\mathcal{I}_+$ &
        |[fill=green!15]| $\mathcal{I}_+$ & $\mathcal{I}_+$ & $\mathcal{I}_+$ & |[fill=yellow!30]| $\mathcal{I}_+$ &
        |[fill=gray!15]| $\mathcal{I}_0$ & $\mathcal{I}_+$ & $\mathcal{I}_+$ & |[fill=gray!15]| $\mathcal{I}_0$ &
        |[fill=green!15]| $\mathcal{I}_-$ & $\mathcal{I}_+$ & $\mathcal{I}_+$ & |[fill=yellow!30]| $\mathcal{I}_-$ &
        |[fill=blue!15]| $\mathcal{I}_-$ & $\mathcal{I}_+$ & $\mathcal{I}_+$ & |[fill=red!15]| $\mathcal{I}_-$ \\
    };

    \draw [decorate, decoration={brace, amplitude=12pt}, thick] 
        ([yshift=-8pt]vec-1-15.south east) -- ([yshift=-8pt] vec-1-2.south west) 
        node [label font, midway, below=16pt] {Segment 1};

    \draw [decorate, decoration={brace, amplitude=12pt}, thick] 
        ([yshift=-8pt]vec-1-19.south east) -- ([yshift=-8pt] vec-1-18.south west) 
        node [label font, midway, below=16pt] {Segment 2};

    \draw [decorate, decoration={brace, amplitude=12pt}, thick] 
        ([yshift=-8pt]vec-1-23.south east) -- ([yshift=-8pt] vec-1-22.south west) 
        node [label font, midway, below=16pt] {Segment 3};

    \draw[-Latex, very thick] (mat.south) -- node[right, label font, xshift=12pt] {flatten} (vec.north);

\end{tikzpicture}
    }
    \caption{Memory layout of the standard RFFT of a 6x6 matrix in a row-major memory format. Real values belonging to index group $\mathcal{I}_0$ (c.f. Proposition~\ref{prop:packed_dft_unitary}) are shaded gray, while complex elements (without conjugated counterpart) of group $\mathcal{I}_+$ are shown on white background. All other colors show redundant complex conjugated  pairs ($\mathcal{I}_+$, $\mathcal{I}_-$) as Hermitian symmetry in the first and the last column. For packed readout, we only need to extract one unique element, say $\mathcal{I}_+$ of each color. From the flattened in-memory order one can see that the unique complex elements ($\mathcal{I}_+$) are distributed into 3 different non-consecutive segments which does not allow for an efficient bulk memory readout.}
    \label{fig:memory_rfft}
\end{figure}

The primary obstacle to achieving the contiguous memory layout required for a packed readout lies in how standard software libraries implement multidimensional transforms, and how this scatters Hermitian-redundant information in memory. In a row-major memory architecture (such as C or Python), memory is contiguous along the last dimension (the rows)\footnotemark{}. 
Standard FFT libraries using a row-major ordering effectively compute a 2D RFFT by applying a 1D RFFT along the memory-contiguous rows, followed by a full 1D FFT along the non-contiguous columns. While this ordering may be optimized for transform speed (especially for in-place RFFTs) by maximizing cache utilization, it is detrimental to our packing objectives. It weaves the redundant symmetric components intricately throughout the memory layout (c.f. Figure~\ref{fig:memory_rfft}). Consequently, isolating the unique degrees of freedom from this standard output forces the programmer to explicitly allocate and construct the packed vector $\Fp{}$ using strided memory-gathering operations from different segments, thereby preventing efficient bulk memory access.

\footnotetext{For simplicity, we assume row-major ordering in the following discussion, though the same mechanism applies similarly to column-major architectures, like in Julia.} 

\subsubsection{Axis-reordered 2D RFFT}

To circumvent the memory fragmentation bottleneck, discussed above, we propose a tailored, non-standard structural reordering of the 2D RFFT, which we call the \emph{axis-reordered 2D RFFT}. Specifically, we first apply the 1D RFFT along the non-contiguous (column) axis, and subsequently apply the full 1D FFT along the contiguous (row axis). A simple Python implementation is shown in Listing~\ref{listing:axis_reordered_rfft}.

Before discussing the consequences of this transformation on the memory layout, it is worth noting that since in the discussed row-major memory architecture, the 1D RFFT is applied along the column axis, this naturally defines the dimension along which the input vector $\Fx{}$ created by a 1D RFFT needs to be broadcast across the 2D input weights $\FWx$ (c.f. Figure~\ref{fig:dimensional_padding}). 

\begin{listing}[!ht]
\begin{minted}{python}
import numpy as np

def axis_reordered_rfft2(X):
    """
    Computes the axis-reordered 2D RFFT for a real matrix X of shape
    (K1, K2). Returns a complex array of shape (K1//2 + 1, K2).
    """
    return np.fft.rfft2(X, axes=(-1,-2))
\end{minted}
\caption{Example implementation of the axis-reordered 2D RFFT using Python.}
\label{listing:axis_reordered_rfft}
\end{listing}

The specific sequence of the axis-reordered 2D RFFT fundamentally reorganizes the symmetries, grouping them in a way that aligns better with the contiguous memory layout of unique complex components (c.f. Figure~\ref{fig:memory_rfft_axis_reordered}). 
By applying the first 1D RFFT on the non-contiguous column axis first, the intermediate representation yields strictly real-valued sequences in its DC and Nyquist rows (the first and last rows, respectively). All intermediate rows contain uniquely complex data. When the second, full 1D FFT is subsequently applied along the memory-contiguous rows, it operates on these strictly real first and last rows to produce standard 1D Hermitian symmetries entirely within continuous memory blocks. This alignment allows the redundant symmetric halves of these specific contiguous rows to be effortlessly cut away via simple truncation. As a result, almost the entire mass of non-redundant complex data is safely isolated and consolidated into a single, uninterrupted contiguous block as shown in Figure~\ref{fig:memory_rfft_axis_reordered}.

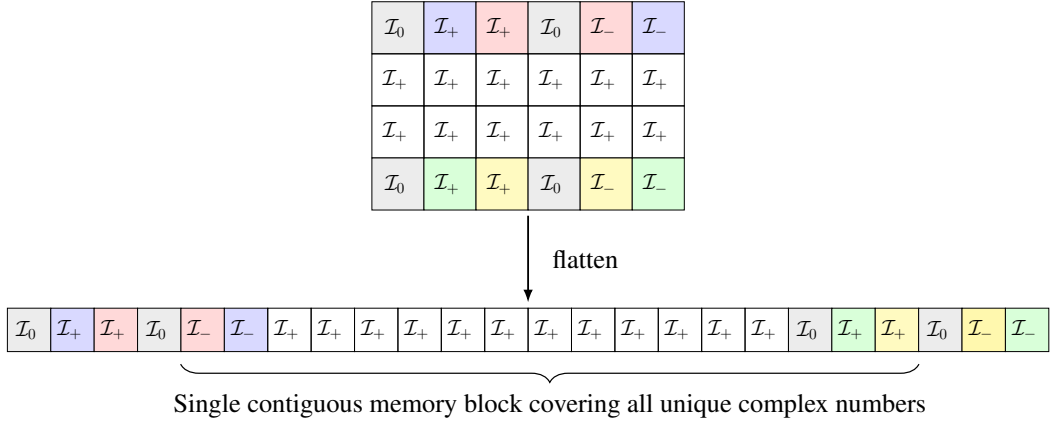
\begin{figure}[h]
    \centering
    \resizebox{\linewidth}{!}{\begin{tikzpicture}[
    cell/.style={
        draw, 
        minimum size=12mm,
        anchor=center, 
        align=center, 
        inner sep=0pt, 
        font=\Large 
    },
    label font/.style={
        font=\LARGE
    }
]

    \matrix[
        matrix of nodes,
        nodes={cell},
        column sep=-\pgflinewidth, 
        row sep=-\pgflinewidth
    ] (mat) at (0,0) {
        |[fill=gray!15]| $\mathcal{I}_0$ & |[fill=blue!15]| $\mathcal{I}_+$ & |[fill=red!15]| $\mathcal{I}_+$ & |[fill=gray!15]| $\mathcal{I}_0$ & |[fill=red!15]| $\mathcal{I}_-$ &|[fill=blue!15]| $\mathcal{I}_-$ \\
        $\mathcal{I}_+$ & $\mathcal{I}_+$ & $\mathcal{I}_+$ & $\mathcal{I}_+$ & $\mathcal{I}_+$ & $\mathcal{I}_+$ \\
        $\mathcal{I}_+$ & $\mathcal{I}_+$ & $\mathcal{I}_+$ & $\mathcal{I}_+$ & $\mathcal{I}_+$ & $\mathcal{I}_+$ \\
        |[fill=gray!15]| $\mathcal{I}_0$ & |[fill=green!15]| $\mathcal{I}_+$ & |[fill=yellow!30]| $\mathcal{I}_+$ & |[fill=gray!15]| $\mathcal{I}_0$ & |[fill=yellow!30]| $\mathcal{I}_-$ &|[fill=green!15]| $\mathcal{I}_-$ \\
    };

    \matrix[
        matrix of nodes,
        nodes={cell, minimum size=10mm}, 
        column sep=-\pgflinewidth,
        below=2cm of mat
    ] (vec) {
        |[fill=gray!15]| $\mathcal{I}_0$ & |[fill=blue!15]| $\mathcal{I}_+$ & |[fill=red!15]| $\mathcal{I}_+$ & |[fill=gray!15]| $\mathcal{I}_0$ & |[fill=red!15]| $\mathcal{I}_-$ &|[fill=blue!15]| $\mathcal{I}_-$ &
        $\mathcal{I}_+$ & $\mathcal{I}_+$ & $\mathcal{I}_+$ & $\mathcal{I}_+$ & $\mathcal{I}_+$ & $\mathcal{I}_+$ &
        $\mathcal{I}_+$ & $\mathcal{I}_+$ & $\mathcal{I}_+$ & $\mathcal{I}_+$ & $\mathcal{I}_+$ & $\mathcal{I}_+$ &
        |[fill=gray!15]| $\mathcal{I}_0$ & |[fill=green!15]| $\mathcal{I}_+$ & |[fill=yellow!30]| $\mathcal{I}_+$ & |[fill=gray!15]| $\mathcal{I}_0$ & |[fill=yellow!40]| $\mathcal{I}_-$ &|[fill=green!15]| $\mathcal{I}_-$ \\
    };

    \draw [decorate, decoration={brace, amplitude=12pt}, thick] 
        ([yshift=-8pt]vec-1-21.south east) -- ([yshift=-8pt]vec-1-5.south west) 
        node [label font, midway, below=16pt] {Single contiguous memory block covering all unique complex numbers};

    \draw[-Latex, very thick] (mat.south) -- node[right, label font, xshift=12pt] {flatten} (vec.north);

\end{tikzpicture}}
    \caption{Memory layout of the proposed axis-reordered RFFT. See Proposition~\ref{prop:packed_dft_unitary} for a definition of the index groups $\mathcal{I}_0$, $\mathcal{I}_+$, and $\mathcal{I}_-$.}
    \label{fig:memory_rfft_axis_reordered}
\end{figure}

\subsubsection{Packed \FDh readout}

Thanks to the axis-reordered 2D RFFT, we eliminate the need for extracting isolated memory segments from $\FR{}$ to construct $\Fp{}$. The vast majority of the packed vector $\Fp{}$ can simply be mapped directly from a continuous memory block via a zero-overhead pointer reinterpretation. Consequently, almost the full readout matrix multiplication can be executed directly on this raw memory view using optimized linear algebra routines on CPU or GPU. 

The minor structural anomalies that remain (specifically, the four self-conjugate elements of $\mathcal{I}_0$: DC-DC, Nyquist-DC, DC-Nyquist, and Nyquist-Nyquist) are subsequently resolved using four scalar corrections per output. Moreover, instead of scaling the large number of elements of $\mathcal{I}_+$ as derived in Proposition~\ref{prop:packed_dft_unitary} by $\sqrt{2}$, we only scale the remaining 4 elements of $\mathcal{I}_0$ by $1/\sqrt{2}$. Although this leads to a global factor on the packed vector $\Fp{}$, we have seen in Corollary~\ref{cor:scaled_unitary} that this can be compensated for by a rescaled regularization parameter $\lambda$.

Listing~\ref{listing:pack} provides an explicit implementation of the packing operation $\pack{\cdot}$, which is practically used to accumulate sequential reservoir states into a dense feature matrix for training the readout weights via ridge regression. 

\begin{listing}[!ht]
\begin{minted}{python}
import numpy as np

def pack(R):
    """
    Packs the redundant complex representation R as e.g. obtained by 
    R = axis_reordered_rfft2(X) into a dense, real-valued vector p of
    length N = K1 * K2 where (K1, K2) is the shape of the real matrix X.
    The original spatial domain dimensions K1 and K2 are assumed to be even.    
    """ 
    K1 = 2 * (R.shape[0] - 1)
    K2 = R.shape[1]

    R_view = R.view(np.float64).flatten()

    # --- Define source indices (indices of X_view) ---

    idx_src_dc_dc = 0
    idx_src_dc_ny = K2
    idx_src_ny_dc = K1 * K2
    idx_src_ny_ny = idx_src_ny_dc + K2
    
    src_start     = idx_src_dc_ny + 2
    src_end       = idx_src_ny_ny
    block_len     = src_end - src_start
    
    # --- Define destination indices (indices of p) ---
    
    idx_dst_ny_dc = idx_src_ny_dc - src_start + 2
    idx_dst_ny_ny = idx_dst_ny_dc + 1

    # --- Copy bulk part ---

    p = np.empty(K1*K2, dtype = np.float64)
    p[2 : 2 + block_len] = R_view[src_start : src_end]
    
    # --- Copy the scaled 4 remaining elements ---
    
    scale            = 1/np.sqrt(2.0)

    p[0]             = R_view[idx_src_dc_dc] * scale
    p[1]             = R_view[idx_src_dc_ny] * scale
    p[idx_dst_ny_dc] = R_view[idx_src_ny_dc] * scale
    p[idx_dst_ny_ny] = R_view[idx_src_ny_ny] * scale
    
    return p
\end{minted}
\caption{Example Python implementation of the packing operation $\Fp{} = \pack{\FR{}}$.}
\label{listing:pack}
\end{listing}

\subsubsection{Experiment: Timing of readout}
\label{app:readout-benchmark}

This experiment evaluates the computational efficiency of the proposed packed \FDh readout defined in Eq.~\eqref{eq:packed_readout}. As a baseline, we benchmark this method against the standard real-valued dense readout layer typically employed in conventional (\SDh) ESNs. The reported runtimes encompass the computation of a single output vector. For FRESCO, this includes the extraction of the packed state vector $\Fp{}$ via the packing operation $\pack{\FR{}}$ followed by the dense matrix-vector multiplication. For the baseline ESN, this consists purely of the standard matrix-vector product. We evaluate these operations across different reservoir sizes $\SNr$ and output sizes $\SNy$ by assuming $\SNr = 256 \SNy$. As illustrated in Figure~\ref{fig:bench_readout}, the packed \FDh readout matches the computational speed of the standard ESN baseline. These results confirm that the tailored memory reordering and the packing operation $\pack{\cdot}$ introduce negligible computational overhead. Consequently, extracting the target output directly from the complex \FDh reservoir state completely circumvents the need for costly inverse FFTs while fully preserving the high readout efficiency of standard ESN architectures. 

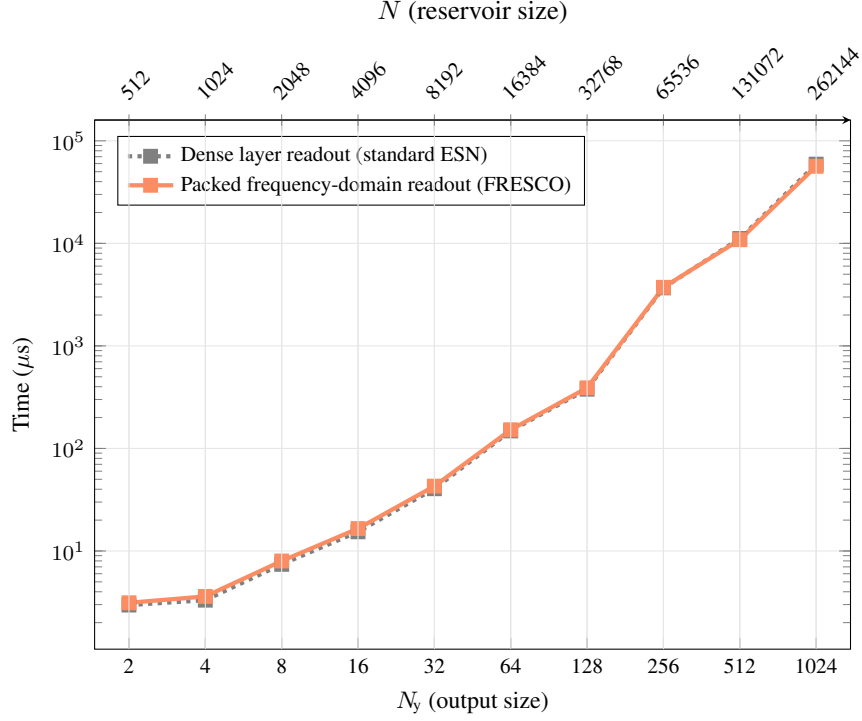
\begin{figure}[htbp]
    \centering
    \begin{tikzpicture}[remember picture] 

    \def\plotWidth{10cm}
    \def\plotHeight{7cm}
    \def\timediv{1000}
    \def\enlargefact{0.05}

    \pgfplotsset{
        my plot layout/.style={
            width=\plotWidth,
            height=\plotHeight,
            scale only axis,
            log basis x=2,
            xtick=data,
            grid=both,
            minor tick num=0, 
            grid style={gray!20},
            yminorgrids=false,
            label style={font=\small},
            tick label style={font=\footnotesize},
            legend style={font=\footnotesize},
        }
    }

    \pgfplotstableread[col sep=comma]{data/data_bench_readout.csv}\datatable

    \begin{axis}[
        my plot layout,
        xmode=log, ymode=log,
        name=mainaxis,          
        xlabel={$\SNy$ (output size)},
        ylabel={Time ($\mu$s)},
        legend pos=north west,
        legend cell align=left,
        enlarge x limits=\enlargefact, 
        xticklabels from table={\datatable}{Nx}   
    ]
        \addplot[color=gray, dotted, mark=square*, mark options={solid}, ultra thick] table [
            x=Nx, y expr={\thisrow{time_readout_sd_blas_ns}/\timediv}
        ] {\datatable};
        \addlegendentry{Dense layer readout (standard ESN)}
        
        \addplot[color=Set2-B, mark=square*, ultra thick] table [
            x=Nx, y expr={\thisrow{time_readout_fd_blas_ns}/\timediv}
        ] {\datatable};
        \addlegendentry{Packed frequency-domain readout (FRESCO)}
    \end{axis}

    \begin{axis}[
        my plot layout,
        xmode=log, ymode=log,    
        at={(mainaxis.south west)},
        anchor=south west,      
        axis y line=none,
        axis x line=top,
        enlarge x limits=\enlargefact,
        xlabel={$\SNr$ (reservoir size)},
        xticklabels from table={\datatable}{N},   
        xticklabel style={rotate=45, anchor=south west} 
    ]
        \addplot[draw=none, forget plot] table [
            x=Nx, y expr={\thisrow{time_readout_fd_blas_ns}/\timediv}
        ] {\datatable};
    \end{axis}
\end{tikzpicture}
    \caption{Comparison of minimal computational times for packed \FDh readout compared to a conventional ESN readout as baseline, for increasing reservoir size $\SNr$ and output size $\SNy$. The small differences confirm that the packing operation $\Fp{} = \pack{\FR{}}$ is negligible compared to the actual matrix-vector products.}
    \label{fig:bench_readout}
\end{figure}

\subsubsection{Proof: Invariance of ridge regression under unitary transformation}
\label{app:ridge_invariance}

In this section, we established that the optimal weights obtained via ridge regression are mathematically invariant under unitary transformations of the feature space. This powerful property guarantees that training a reservoir readout yields equivalent results whether performed in the spatial- or frequency domain, or the packed reservoir representation, provided the transformation between them is unitary (such as the normalized DFT, or the packing operation as shown in the subsequent section).

\begin{proposition}[Ridge regression under unitary feature transformations]
\label{prop:unitary_equiv}
Let $Y \in \mathbb{C}^{k \times n}$ be a target matrix and $R \in \mathbb{C}^{p \times n}$ be a feature matrix. Let $W$ be the weight matrix minimizing the ridge regression objective for the original features $R$. Consider a unitary linear transformation matrix $D \in \mathbb{C}^{p \times p}$ (such that $D^\dagger D = D D^\dagger = I$, where $\dagger$ denotes the Hermitian transpose). Then, the weight matrix $W_{\!D}$ minimizing the ridge regression objective for the transformed features $DR$ is equivalent to the original weights transformed by $D^{-1}$:
\[
W_{\!D} = W D^{-1}.
\]
Thus, performing ridge regression in the transformed domain is mathematically equivalent to the original domain, and the corresponding weights can be recovered by post-multiplying the transformed weights by $D$.
\end{proposition}

\begin{proof}
The ridge regression objective for the original data is to minimize:
\[
L(W) = \|Y - W R\|_F^2 + \lambda \|W\|_F^2,
\]

where $\lambda$ is the regularization parameter. Setting the gradient of $L(W)$ with respect to $W$ to zero yields the standard solution:
\[
W = Y R^\dagger (R R^\dagger + \lambda I)^{-1}
\]

When applying the transformation $D$ so that $R$ is replaced by $D R$, we seek the weight matrix $W_{\!D}$ that minimizes:
\[
L(W_{\!D}) = \|Y - W_{\!D} (D R)\|_F^2 + \lambda \|W_{\!D}\|_F^2
\]

Setting the gradient of $L(W_{\!D})$ to zero and solving for the transformed data yields:
\[
W_{\!D} = Y (D R)^\dagger ((D R) (D R)^\dagger + \lambda I)^{-1}
\]

Expanding the Hermitian transpose $(D R)^\dagger = R^\dagger D^\dagger$ leads to:
\[
W_{\!D} = Y R^\dagger D^\dagger (D R R^\dagger D^\dagger + \lambda I)^{-1}
\]

Because $D$ is a unitary matrix, it satisfies $D D^\dagger = I$. We can therefore rewrite the identity matrix $I$ in the regularization term as $D D^\dagger$. Substituting this yields:
\[
\begin{aligned}
W_{\!D} &= Y R^\dagger D^\dagger (D R R^\dagger D^\dagger + \lambda D D^\dagger)^{-1} \\
W_{\!D} &= Y R^\dagger D^\dagger [D (R R^\dagger + \lambda I) D^\dagger]^{-1}
\end{aligned}
\]

Applying the matrix inverse property $(ABC)^{-1} = C^{-1} B^{-1} A^{-1}$:
\[
W_{\!D} = Y R^\dagger D^\dagger (D^\dagger)^{-1} (R R^\dagger + \lambda I)^{-1} D^{-1}
\]

Since $D^\dagger (D^\dagger)^{-1} = I$ for unitary matrices, the expression simplifies to:
\[
W_{\!D} = Y R^\dagger (R R^\dagger + \lambda I)^{-1} D^{-1}
\]

Recognizing the original definition of $W$, we perfectly recover the spatial weights post-multiplied by $D^{-1}$ (or $D^\dagger$):
\[
W_{\!D} = W D^{-1} = W D^\dagger 
\]
\end{proof}

Depending on the normalization of the DFT and the explicit frequency representation used, this equivalence holds up to a pre-factor that can be absorbed in the regularization:

\begin{corollary}[Scaled unitary transformations]
\label{cor:scaled_unitary}
If $D$ applies a uniform scale alongside a unitary transformation such that $D^\dagger D = cI$ for some positive real scalar $c$, mathematical equivalence is maintained by adjusting the regularization parameter to $\lambda' = \lambda / c$, yielding $W_{\!D} = W(\lambda') D^{-1}$.
\end{corollary}

\begin{remark}[Independence from numerical solvers]
\label{rem:num_solvers}
The mathematical equivalence demonstrated in Proposition~\ref{prop:unitary_equiv} is a fundamental property of the ridge regression objective. Therefore, the invariance holds strictly true regardless of the numerical method used to compute the matrix inversion (e.g., Cholesky decomposition for positive-definite systems or the Bunch-Kaufman decomposition). 
\end{remark}

\subsubsection{Proof: Unitarity of the packed \FDh representations}
\label{app:unitarity_packed}

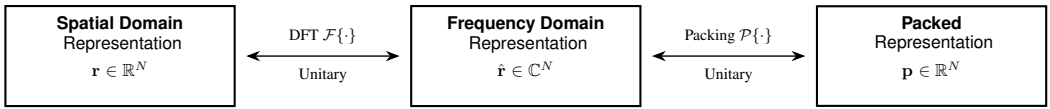
\begin{figure}[h]
    \centering
    \resizebox{\linewidth}{!}{
    \begin{tikzpicture}[
    block/.style={
        rectangle, 
        draw=black, 
        fill=white, 
        very thick, 
        minimum width=4.6cm, 
        minimum height=1.4cm, 
        align=center,
        font=\sffamily,
        inner sep=6pt 
    },
    bi_arrow/.style={
        {Stealth[scale=1.2]}-{Stealth[scale=1.2]}, 
        thick, 
        shorten >=6pt, 
        shorten <=6pt
    },
    math_label/.style={
        font=\small
    }
]

    
    \node[block] (spatial) {
        \textbf{Spatial Domain}\\ 
        \vspace{0.1cm}
        Representation \\
        $\Sr{} \in \mathbb{R}^N$\\
    };

    \node[block] (freq) [right=3.5cm of spatial] {
        \textbf{Frequency Domain}\\ 
        \vspace{0.1cm}
        Representation\\
        $\Fr{} \in \mathbb{C}^N$\\
    };

    \node[block] (packed) [right=3.5cm of freq] {
        \textbf{Packed}\\ 
        \vspace{0.1cm}
        Representation \\
        $\Fp{} \in \mathbb{R}^N$\\
    };


    \draw[bi_arrow] (spatial) -- (freq) 
        node[midway, above=3pt, math_label] {DFT $\fourier{\cdot}$}
        node[midway, below=3pt, math_label] {Unitary};

    \draw[bi_arrow] (freq) -- (packed) 
        node[midway, above=3pt, math_label] {Packing $\pack{\cdot}$}
        node[midway, below=3pt, math_label] {Unitary};

\end{tikzpicture}
    }
    \caption{Relationship between the reservoir state representations. The spatial domain state $\Sr{}$ is transformed to the complex frequency domain $\Fr{}$ via the DFT $\fourier{\cdot}$, which is then transformed to the dense, real-valued packed state $\Fp{}$ via the packing operation $\pack{\cdot}$. Because both transformations are unitary, performing ridge regression on the packed representation $\Fp{}$ yields mathematically identical results to training directly in the spatial domain.}
    \label{fig:unitary_reps}
\end{figure} 

\paragraph{Context and motivation.} 
The purpose of the following proposition is to rigorously prove that the packed representation introduced in Section~\ref{app:packed_readout} is connected to the full complex frequency domain via a unitary mapping. By establishing this unitary link, we can invoke Proposition~\ref{prop:unitary_equiv} from the previous section to definitively claim that performing ridge regression on the highly efficient packed vectors is mathematically identical to training in the original spatial domain. The transformations between the \SDh, \FDh and the packed reservoir states are illustrated in Figure~\ref{fig:unitary_reps}.

Crucially, the following mathematical formulation is constructive. Specifically, Step 4 of the proof (``Proving the Basis Spans $W$ over $\mathbb{R}$'') serves as the practical algorithmic recipe for creating these packed vectors. It details exactly how to isolate purely real self-conjugate elements, how to split complex conjugate pairs into real and imaginary components, and precisely how to scale them (e.g., by a factor of $\sqrt{2}$) to conserve energy and satisfy the unitary condition. Note that while the proof explicitly constructs the unitary matrix $U$ to establish theoretical equivalence, this large $N \times N$ matrix never needs to be instantiated in software; one simply applies the inexpensive, element-wise packing and scaling rules derived herein.

\paragraph{Setup and definitions.} 
Let $x$ be a real-valued $n$-dimensional tensor defined on a spatial grid of dimensions $\mathbf{N} = (N_1, N_2, \dots, N_n)$. The total number of elements is $N = \prod_{d=1}^n N_d$. Let $\mathcal{X}$ be the $n$-dimensional DFT of $x$, residing on the frequency grid $\mathcal{G} = \mathbb{Z}_{N_1} \times \mathbb{Z}_{N_2} \times \dots \times \mathbb{Z}_{N_n}$. 

Because the input $x$ is real, its frequency representation $\mathcal{X}$ exhibits Hermitian symmetry. For any multi-dimensional frequency index $\mathbf{k} \in \mathcal{G}$, this symmetry is defined as:
\begin{equation}
    \mathcal{X}(\mathbf{k}) = \mathcal{X}^*(-\mathbf{k} \pmod{\mathbf{N}})
\end{equation}
where $(\cdot)^*$ denotes complex conjugation, and $-\mathbf{k} \pmod{\mathbf{N}}$ denotes the element-wise negation of the indices modulo their respective dimension sizes. 

Let $F: \mathcal{G} \to \mathcal{I}$ be a bijective flattening function (e.g., standard row-major ordering) that maps the $n$-dimensional grid coordinates to a 1D linear index set $\mathcal{I} = \{0, 1, \dots, N-1\}$. Let $X \in \mathbb{C}^N$ be the flattened 1D vector of $\mathcal{X}$, such that $X[F(\mathbf{k})] = \mathcal{X}(\mathbf{k})$. The Hermitian symmetry on the grid $\mathcal{G}$ induces a bijective mapping $S: \mathcal{I} \to \mathcal{I}$ on the linear indices, defined by $S(F(\mathbf{k})) = F(-\mathbf{k} \pmod{\mathbf{N}})$. Intuitively, $S$ captures the Hermitian symmetry within the flattened 1D representation, mapping each linear index in $\mathcal{I}$ to the index of its complex conjugate counterpart, or mapping to itself if the corresponding frequency component is purely real. Because double negation on the modular grid yields the original coordinate, $S$ is an involution, meaning $S(S(m)) = m$ for all $m \in \mathcal{I}$. 

We define $W = \{ X \in \mathbb{C}^N \mid X[m] = X^*[S(m)] \;\; \forall m \in \mathcal{I} \}$ as the subset of $\mathbb{C}^N$ consisting of all valid frequency vectors that satisfy this Hermitian symmetry. Because these symmetry constraints couple the real and imaginary components of the complex frequencies, they eliminate exactly half of the $2N$ real degrees of freedom available in $\mathbb{C}^N$. Consequently, any valid frequency vector in $W$ is entirely determined by exactly $N$ independent real values. This motivates the formulation of an efficient, packed representation in $\mathbb{R}^N$, alongside a unitary transformation capable of perfectly reconstructing any possible vector in $W$ from this packed state.

\begin{proposition}[Unitarity of the packed representation]
\label{prop:packed_dft_unitary}
There exists a unitary matrix $U \in \mathbb{C}^{N \times N}$ whose columns form an orthonormal basis such that any complex symmetric vector $X \in W$ can be expressed as a linear combination of these columns using real coefficients. Consequently, for any $X \in W$, there exists a unique real-valued packed vector $p \in \mathbb{R}^N$ such that $X = U p$.
\end{proposition}

\begin{proof}
\textbf{1. Partitioning the index set $\mathcal{I}$.} Using the involution $S$, we partition the linear index set $\mathcal{I}$ into three disjoint subsets:
\begin{itemize}
    \item $\mathcal{I}_0 = \{m \in \mathcal{I} \mid S(m) = m\}$. These are the fixed points of the involution (e.g., the generalized DC and Nyquist frequencies). For these indices, $X[m] = X^*[m]$, meaning $X[m] \in \mathbb{R}$.
    \item For all indices where $S(m) \neq m$, the involution groups them into disjoint pairs $(m, S(m))$. We split these pairs into two sets, $\mathcal{I}_+$ and $\mathcal{I}_-$, such that $\mathcal{I}_+$ contains exactly one index from each pair, and $\mathcal{I}_-$ contains its symmetric counterpart. Thus, $m \in \mathcal{I}_+ \iff S(m) \in \mathcal{I}_-$.
\end{itemize}
By definition, $\mathcal{I}_0 \cup \mathcal{I}_+ \cup \mathcal{I}_- = \mathcal{I}$.

\textbf{2. Construction of the basis vectors.} Let $e_m \in \mathbb{R}^N$ denote the standard basis column vector (a $1$ at index $m$, and $0$ elsewhere). We define $N$ basis vectors $u_m \in \mathbb{C}^N$ piecewise:
\begin{align}
    \text{For } m \in \mathcal{I}_0&: \quad u_m = e_m \\
    \text{For } m \in \mathcal{I}_+&: \quad u_m = \frac{1}{\sqrt{2}}(e_m + e_{S(m)}) \\
    \text{For } m \in \mathcal{I}_-&: \quad u_m = \frac{i}{\sqrt{2}}(e_{S(m)} - e_m)
\end{align}

\textbf{3. Proving unitarity ($U^\dagger U = I$).} Let $U = \begin{bmatrix} u_0 & \dots & u_{N-1} \end{bmatrix}$. To show that $U$ is unitary, we verify that its columns are orthonormal ($u_a^\dagger u_b = \delta_{a,b}$). 
For unit norms ($a = b$):
\begin{itemize}
    \item If $a \in \mathcal{I}_0$: $\|u_a\|^2 = e_a^\top e_a = 1$.
    \item If $a \in \mathcal{I}_+$: $\|u_a\|^2 = \frac{1}{2}(e_a^\top e_a + e_{S(a)}^\top e_{S(a)}) = 1$.
    \item If $a \in \mathcal{I}_-$: $\|u_a\|^2 = \left(\frac{-i}{\sqrt{2}}\right)\left(\frac{i}{\sqrt{2}}\right) (e_{S(a)}^\top e_{S(a)} + e_a^\top e_a) = 1$.
\end{itemize}
For orthogonality ($a \neq b$): If $a$ and $b$ do not belong to the same symmetric pair ($b \neq S(a)$), their underlying standard basis vectors are disjoint, yielding $u_a^\dagger u_b = 0$. If they belong to the same pair, assume without loss of generality that $a \in \mathcal{I}_+$ and $b = S(a) \in \mathcal{I}_-$. Their inner product evaluates to:
\begin{equation}
    u_a^\dagger u_b = \left[ \frac{1}{\sqrt{2}}(e_a + e_b) \right]^\dagger \left[ \frac{i}{\sqrt{2}}(e_a - e_b) \right] = \frac{i}{2} (e_a^\top + e_b^\top)(e_a - e_b) = \frac{i}{2}(e_a^\top e_a - e_b^\top e_b) = 0
\end{equation}
Since all $N$ columns are orthonormal, $U$ is unitary.

\textbf{4. Proving the basis spans $W$ over $\mathbb{R}$.} To prove $U$ maps $\mathbb{R}^N \to W$, we must show any $X \in W$ can be uniquely constructed as $X = \sum_{m=0}^{N-1} p_m u_m$ using real coefficients $p_m \in \mathbb{R}$. We define the coefficients $p_m$ as follows:
\begin{itemize}
    \item For $m \in \mathcal{I}_0$: $p_m = X[m]$.
    \item For $m \in \mathcal{I}_+$: $p_m = \sqrt{2} \Re(X[m])$.
    \item For $m \in \mathcal{I}_-$: $p_m = \sqrt{2} \Im(X[S(m)])$.
\end{itemize}
Evaluating the sum for $\mathcal{I}_0$ yields $X[m]e_m$. For a specific symmetric pair $a \in \mathcal{I}_+$ and $b = S(a) \in \mathcal{I}_-$, the sum is:
\begin{align}
    p_a u_a + p_b u_b &= \sqrt{2}\Re(X[a]) \left[ \frac{1}{\sqrt{2}}(e_a + e_b) \right] + \sqrt{2}\Im(X[a]) \left[ \frac{i}{\sqrt{2}}(e_a - e_b) \right] \\
    &= (\Re(X[a]) + i\Im(X[a]))e_a + (\Re(X[a]) - i\Im(X[a]))e_b \\
    &= X[a]e_a + X^*[a]e_b = X[a]e_a + X[b]e_b
\end{align}
This reconstructs the elements of $X$ at both indices $a$ and $b$. Summing over all partitions yields exactly $X = Up$. Because $p \in \mathbb{R}^N$ and $U$ is unitary, the columns of $U$ form an orthonormal basis for $W$ over $\mathbb{R}$.
\end{proof}

\subsection{Proofs of the Echo State Property Conditions}
\label{app:esp}

This section provides the proofs of Propositions~\ref{prop:esp_sufficient} and~\ref{prop:esp_necessary} from the main text.

\subsubsection{Proof of Proposition~\ref{prop:esp_sufficient} (sufficient condition)}
\label{app:esp-sufficient}

Both FRESCO variants share the same Lipschitz constant $L_{\Factfun}=\alpha$. For $\operatorname{invabs}(z)=\frac{\alpha z}{1+|z|}$, writing $z=re^{i\theta}$ shows the radial and tangential singular values of its Jacobian are $\frac{\alpha}{(1+r)^2}\leq\alpha$ and $\frac{\alpha}{1+r}\leq\alpha$, so $|\operatorname{invabs}(z)-\operatorname{invabs}(w)|\leq\alpha|z-w|$. In FRESCO (mix) the cyclic shift is an isometry under the Frobenius norm (permuting elements preserves the sum of squared magnitudes), so the Lipschitz constant of the composed activation remains $\alpha$.

\begin{proof}
Let $\kappa=(1-\tau)+\tau\alpha\max_{i,j}|\FWr[i,j]|$ and assume $\kappa<1$. Fix a bounded input sequence and consider two trajectories $\FR{t}$ and $\FR{t}'$ driven by the same input from different initial states. Let $\Delta_t=\FR{t}-\FR{t}'$ and $\mathbf{U}_t=\FWx\odot\FX{t}+\FB$. The update Eq.~\eqref{eq:esn-fd-update-2D} gives
\begin{align}
\|\Delta_t\|_F
&\leq (1-\tau)\|\Delta_{t-1}\|_F
 + \tau\bigl\|\Factfun\bigl(\FWr\odot\FR{t-1}+\mathbf{U}_t\bigr)
             -\Factfun\bigl(\FWr\odot\FR{t-1}'+\mathbf{U}_t\bigr)\bigr\|_F \notag\\
&\leq (1-\tau)\|\Delta_{t-1}\|_F
 + \tau\,\alpha\,\|\FWr\odot\Delta_{t-1}\|_F \notag\\
&\leq (1-\tau)\|\Delta_{t-1}\|_F
 + \tau\,\alpha\,\max_{i,j}|\FWr[i,j]|\cdot\|\Delta_{t-1}\|_F \notag\\
&= \kappa\,\|\Delta_{t-1}\|_F,
\end{align}
where the second inequality uses $L_{\Factfun}=\alpha$ and the third uses $\|\FWr\odot\mathbf{M}\|_F\leq\max_{i,j}|\FWr[i,j]|\cdot\|\mathbf{M}\|_F$. By induction, $\|\Delta_t\|_F\leq\kappa^t\|\Delta_0\|_F\to 0$ exponentially.
\end{proof}

\paragraph{Connection to the standard ESN condition.}
By Eq.~\eqref{eq:spectral-radius-hadamard}, $\max_{i,j}|\FWr[i,j]|=\rho(\SWr)$, so the condition is equivalent to $\alpha\cdot\rho(\SWr)<1$, directly mirroring Proposition~\ref{prop:esp_esn} with $\alpha$ replacing the Lipschitz constant of $\tanh$.

\subsubsection{Proof of Proposition~\ref{prop:esp_necessary} (necessary condition)}
\label{app:esp-necessary}

\begin{proof}
Consider the zero-input, zero-bias system, for which $\hat{\mathbf{R}}^{*}=\mathbf{0}$ is a fixed point. A necessary condition for the ESP is stability of this fixed point. Thus, the spectral radius of the Jacobian evaluated at $\hat{\mathbf{R}}^{*}=\mathbf{0}$ must be less or equal than 1.

We first consider FRESCO (plain). In this case there is no cyclic mixing, so linearizing around $\mathbf{0}$ (using $|\operatorname{invabs}'(z)|=\alpha/(1+|z|)=\alpha$ at $z=0$):
\begin{equation}
J_0^{\mathrm{plain}}
=
(1-\tau)I+\tau\alpha D_{\FWr},
\end{equation}
where
\[
D_{\FWr}=\operatorname{diag}(\operatorname{vec}(\FWr)).
\]
Therefore the eigenvalues of $J_0^{\mathrm{plain}}$ are
\begin{equation}
\lambda_{i,j}^{\mathrm{plain}}
=
1-\tau+\tau\alpha \FWr[i,j].
\end{equation}
The necessary condition is
\begin{equation}
\rho(J_0^{\mathrm{plain}})
=
\max_{i,j}
\left|
1-\tau+\tau\alpha \FWr[i,j]
\right|
\le 1.
\end{equation}
Squaring both sides gives, for every $i,j$,
\begin{equation}
\left|
1-\tau+\tau\alpha \FWr[i,j]
\right|^2
\le 1.
\end{equation}
Since $1-\tau$ is real, expanding the squared modulus yields
\begin{equation}
(1-\tau)^2
+
\tau^2\alpha^2|\FWr[i,j]|^2
+
2\tau(1-\tau)\alpha\,\operatorname{Re}(\FWr[i,j])
\le 1.
\end{equation}
Using $1-(1-\tau)^2=\tau(2-\tau)$, and dividing by $\tau>0$, we obtain
\begin{equation}
\tau\alpha^2|\FWr[i,j]|^2
+
2(1-\tau)\alpha\,\operatorname{Re}(\FWr[i,j])
\le
2-\tau
\qquad
\text{for all } i,j.
\end{equation}
Equivalently,
\begin{equation}
\max_{i,j}
\left\{
\tau\alpha^2 |\FWr[i,j]|^2
+
2(1-\tau)\alpha\,\operatorname{Re}(\FWr[i,j])
\right\}
\le
2-\tau,
\end{equation}
which proves Eq.~\eqref{eq:esp-necessary-plain}. In the special case
$\tau=1$, this reduces to
\begin{equation}
\alpha^2\max_{i,j}|\FWr[i,j]|^2 \le 1,
\end{equation}
or, equivalently,
\begin{equation}
\alpha\max_{i,j}|\FWr[i,j]| \le 1.
\end{equation}

We now consider FRESCO (mix). In this case the cyclic shift is present, and
the linearized Jacobian is
\begin{equation}
J_0^{\mathrm{mix}}
=
(1-\tau)I+\tau\alpha P D_{\FWr},
\end{equation}
where $P$ is the permutation matrix associated with the cyclic shift. Define
\begin{equation}
A=P D_{\FWr}.
\end{equation}
Since $P$ is a full $\SNr$-cycle, applying $A$ exactly $\SNr$ times returns
each coordinate to its original position and multiplies it once by every
entry of $\FWr$. Hence,
\begin{equation}
A^{\SNr}
=
(PD_{\FWr})^{\SNr}
=
\left(
\prod_{i,j}\FWr[i,j]
\right)I.
\end{equation}
Now, let us define the following quantities:
\begin{equation}
c=\prod_{i,j}\FWr[i,j],
\qquad
g=|c|^{1/\SNr},
\qquad
\theta=\arg(c).
\end{equation}
The characteristic polynomial of the weighted full cycle $A=P D_{\FWr}$ is
\begin{equation}
p_A(\mu)=\mu^{\SNr}-c.
\end{equation}
Therefore the eigenvalues of $A$ are the $\SNr$-th roots of the complex number $c$, namely
\begin{equation}
\mu_k
=
g\exp\left(
i\frac{\theta+2\pi k}{\SNr}
\right),
\qquad
k=0,\ldots,\SNr-1.
\end{equation}
It follows that the eigenvalues of $J_0^{\mathrm{mix}}$ are
\begin{equation}
\lambda_k^{\mathrm{mix}}
=
1-\tau
+
\tau\alpha g
\exp\left(
i\frac{\theta+2\pi k}{\SNr}
\right),
\qquad
k=0,\ldots,\SNr-1.
\end{equation}
The necessary condition is therefore
\begin{equation}
\rho(J_0^{\mathrm{mix}})
=
\max_{k=0,\ldots,\SNr-1}
\left|
1-\tau
+
\tau\alpha g
\exp\left(
i\frac{\theta+2\pi k}{\SNr}
\right)
\right|
\le 1.
\end{equation}
Squaring both sides of the inequality and using $ |a+be^{i\phi}|^2
=
a^2+b^2+2ab\cos(\phi) $, with $a=1-\tau$ and $b=\tau\alpha g$, gives
\begin{equation}
(1-\tau)^2
+
\tau^2(\alpha g)^2
+
2\tau(1-\tau)\alpha g
\max_{k=0,\ldots,\SNr-1}
\cos\left(
\frac{\theta+2\pi k}{\SNr}
\right)
\le 1.
\end{equation}
Again using $ 1-(1-\tau)^2=\tau(2-\tau) $, and dividing by $\tau>0$, we obtain
\begin{equation}
\tau(\alpha g)^2
+
2(1-\tau)\alpha g
\max_{k=0,\ldots,\SNr-1}
\cos\left(
\frac{\theta+2\pi k}{\SNr}
\right)
\le
2-\tau,
\end{equation}
which proves Eq.~\eqref{eq:esp-necessary-mix}.

Finally, when $\tau=1$, the leakage-dependent term vanishes and the condition
reduces to $ (\alpha g)^2\le 1 $, or equivalently $ g\le \frac{1}{\alpha} $.
Since
\begin{equation}
g
=
\left|
\prod_{i,j}\FWr[i,j]
\right|^{1/\SNr}
=
\left(
\prod_{i,j}|\FWr[i,j]|
\right)^{1/\SNr},
\end{equation}
we obtain
\begin{equation}
\left(
\prod_{i,j}|\FWr[i,j]|
\right)^{1/\SNr}
\le
\frac{1}{\alpha}.
\end{equation}
This proves the stated no-leakage special case for FRESCO (mix).
\end{proof}

\paragraph{Gap between sufficient and necessary conditions.}
The sufficient condition of Proposition~\ref{prop:esp_sufficient} controls the
worst frequency coefficient, namely $\max_{i,j}|\FWr[i,j]|$. In contrast, the
necessary condition of Proposition~\ref{prop:esp_necessary} for FRESCO (mix)
is determined by the spectrum of the weighted cyclic shift $P D_{\FWr}$. Since
$P$ is a full $\SNr$-cycle, this spectrum depends on the product
$c=\prod_{i,j}\FWr[i,j]$, and hence on its geometric modulus
$g=|c|^{1/\SNr}$, together with the phase term appearing in
Eq.~\eqref{eq:esp-necessary-mix}. In the no-leakage case $\tau=1$, the
phase-dependent term disappears and the necessary condition reduces to
\[
\left(\prod_{i,j}|\FWr[i,j]|\right)^{1/\SNr}\le \frac{1}{\alpha}.
\]
Thus, in this case, entries with $|\FWr[i,j]|>1/\alpha$ are not ruled out by
the necessary condition, provided that the geometric mean remains below
$1/\alpha$. This gap is induced by the circular shift, which couples all
frequency bins into a single cycle: local stability of the zero-input fixed
point is then governed by the eigenvalues of the weighted cycle, rather than
by the largest individual coefficient. This condition is only necessary, and
therefore satisfying it does not imply the Echo State Property. Nevertheless, it provides a useful diagnostic for initialising the frequency components of FRESCO.

\paragraph{FRESCO (plain): leakage-dependent stability region.}
Without the circular shift, the linearized Jacobian is diagonal,
\[
J_0^{\mathrm{plain}}=(1-\tau)I+\tau\alpha D_{\FWr}.
\]
Therefore local stability requires
\[
\max_{i,j}
\left|
1-\tau+\tau\alpha\FWr[i,j]
\right|
\le 1,
\]
or equivalently Eq.~\eqref{eq:esp-necessary-plain}. For $\tau<1$, this is not
the same as $\alpha\max_{i,j}|\FWr[i,j]|\le 1$: leakage shifts the stability
region to a disk in the complex plane, so the admissibility of each coefficient
depends also on its phase, through $\operatorname{Re}(\FWr[i,j])$. Hence the
sufficient condition remains conservative in the leaky case. Instead, in the
no-leakage case $\tau=1$, the necessary condition reduce to
\[
\alpha\max_{i,j}|\FWr[i,j]|\le 1,
\]
closing the gap with the sufficient condition.


\section{Experimental setup}
\label{app:experiments}

This section provides the details necessary to reproduce the experiments reported in the main paper. All experiments use validation-based model selection. Hyperparameters are optimized with Optuna using a TPE sampler and Hyperband pruning. The validation objective is NRMSE for univariate regression (NARMA10 and Mackey-Glass), accuracy for classification, and MSE for long-horizon multivariate forecasting. After model selection, the best hyperparameter configuration is retrained on the union of the training and validation partitions and evaluated on the held-out test partition over twenty seeds (for NARMA10 and Mackey-Glass), and ten seeds (for all other tasks).

\subsection{Implementation details}
\label{app:implementation}

We provide two implementations of the models used in the experiments. The first implementation is based on NumPy and Numba and is used for the classical Reservoir Computing benchmarks, namely NARMA10, Mackey-Glass, and time-series classification. This backend is CPU-only and was designed to provide a controlled and coherent comparison between FRESCO-family reservoirs and standard Reservoir Computing baselines. In addition to FRESCO models, we also implemented the standard ESN in the same NumPy/Numba stack. This avoids comparing FRESCO against a baseline implemented with different software stacks and ensures that the reported execution times reflect the reservoir update structure rather than backend-specific overheads.

The second implementation is based on PyTorch and is used for the long-horizon multivariate forecasting benchmarks. These experiments require extracting features for a large number of sliding windows and for multivariate input dimensions, and therefore benefit from GPU execution. The PyTorch implementation is used for the ESN, FRESCO models, and deep learning baselines, namely LSTM, Mamba, and Transformer. 

\paragraph{Reservoir variants.}
The ESN baseline uses a dense random recurrent matrix, an input projection, a bias term, leaky integration, and a real-valued $\tanh$ non-linearity. The dense recurrent matrix is rescaled to the desired spectral radius before evaluation. In our experiments, we explore reservoir sizes of up to nearly 18k units (see Appendix \ref{app:hypers_long_forecast}). Given the finite hyperparameter search budget, we avoid spending excessive time on spectral radius rescaling of large dense matrices by adopting a fast rescaling method from the literature \citep{gallicchio2019fast}. 

FRESCO-family reservoirs operate in the frequency domain. Both FRESCO (plain) and FRESCO (mix) use the native complex-valued non-linearity
\begin{equation}
\operatorname{invabs}(z) = \frac{\alpha z}{1+|z|},
\label{eq:app-frequency-nonlinearity}
\end{equation}
where $\alpha$ is the gain parameter. FRESCO (mix) applies an additional deterministic cyclic mixing of frequency coefficients before the non-linearity.

\paragraph{Packed frequency-domain states.}
For the 
FRESCO variants, the complex frequency-domain state contains Hermitian redundancies because it represents a real-valued spatial-domain signal. We therefore store and expose a packed real-valued representation of the non-redundant frequency coefficients to the readout. During training, the collected complex states are packed once into a real feature matrix. During timed inference, 
FRESCO variants use a direct packed readout path that avoids constructing a new packed vector at every time step. 

\paragraph{Readout training.}
For all reservoir models, recurrent and input weights remain fixed after initialization and only the readout is trained. The readout is a linear ridge regression model. For the NumPy/Numba implementation, the default backend is scikit-learn ridge regression with a Cholesky solver. In the PyTorch implementation used for long-horizon forecasting, the same closed-form ridge objective is solved through a custom PyTorch linear algebra module, so that the large reservoir matrices produced by GPU reservoir extraction do not need to be unnecessarily moved back to a CPU solver.

\subsection{Dataset details and pre-processing}
\label{app:dataset-details}

\subsubsection{Memory capacity benchmarks}
\label{app:MC_details}

To evaluate the memory and information processing capabilities of our reservoir architectures, we utilized the Memory Capacity benchmark \citep{jaeger2002short} and the Information Processing Capacity (IPC) framework \citep{dambre2012information}. While Jaeger's Linear Memory Capacity (LMC) quantifies a reservoir's ability to linearly reconstruct past inputs, the IPC generalizes this using orthogonal Legendre polynomials to systematically measure non-linear transformations of past inputs without overlapping the linear metrics.

Within this framework, LMC and Quadratic Memory Capacity (QMC) represent first- and second-order IPC tasks, respectively. To quantify the trade-off between linear recall and non-linear processing across topologies, we measured both capacities across a logarithmic sweep of input scaling factors, driving the networks with an i.i.d.\ input sequence $x_t \sim \mathcal{U}(-1.0, 1.0)$.

\paragraph{Linear Memory Capacity (LMC).} LMC measures the perfect reconstruction of past inputs. For a given delay $k$, the target is:
\begin{equation}
    y_{t,k}^{(LMC)} = x_{t-k}.
\end{equation}
    
\paragraph{Quadratic Memory Capacity (QMC).} QMC measures the retention of non-linear transformations, specifically second-order Legendre polynomials:
\begin{equation}
    y_{t,k}^{(QMC)} = \frac{1}{2}(3x_{t-k}^2 - 1).
\end{equation}

For both tasks, capacity is the sum of the squared Pearson correlation coefficients ($r^2$) between the linear readout prediction $\hat{y}$ and target $y$ over a maximum delay $k_{max}$:
\begin{equation}
    C = \sum_{k=1}^{k_{max}} r^2\left(y_k, \hat{y}_k\right).
\end{equation}

Because standard ESNs possess a strictly bounded total capacity—meaning linear and non-linear memory cannot be maximized simultaneously—we deliberately avoided hyperparameter optimization. Instead, our parameter sweep exposes this fundamental memory trade-off, directly comparing the intrinsic dynamics of the standard ESN and FRESCO variants.

Each model used a reservoir size $N=100$, leak rate $1.0$ (no state leakage), bias scale $0.1$, and gain $\alpha=1.0$. The maximum delay was $k_{max} = 200$. From a $7000$-step sequence, we discarded $1000$ washout steps, trained readouts on the subsequent $5000$ steps using Ridge regression (regularization $10^{-8}$), and evaluated on the final $1000$ steps.

Our evaluation assesses two primary factors: weight topologies and input signal strength. First, to isolate how recurrent weight structure impacts memory, we tested two distinct topologies:
\begin{enumerate}
\item \textbf{Disk:} A standard dense, random setup. The ESN baseline used a Gaussian matrix scaled to a spectral radius of $0.9$. For FRESCO variants, eigenvalues were uniformly distributed within a disk of magnitude $\le 0.9$.
\item \textbf{Ring:} A highly structured setup preventing memory decay, placing all eigenvalues strictly on a ring of magnitude $0.9$. The ESN baseline achieved this via QR decomposition.
\end{enumerate}

Second, to map the network's transition from linear memory to non-linear computation, we evaluated models across 5 logarithmic steps of the input scaling factor ($10^{-4}$ to $1.0$). This factor controls how strongly inputs push the network into its non-linear activation regime. All LMC and QMC scores were averaged across 10 independent random initializations.

\subsubsection{Classical RC regression benchmarks}
\label{app:RC_details}

\paragraph{NARMA10.}
For NARMA10, the input sequence is sampled as $x_t \sim \mathcal{U}(0,0.5)$. The target is generated by
\begin{equation}
y_t =
0.3 y_{t-1}
+
0.05 y_{t-1}\sum_{i=1}^{10} y_{t-i}
+
1.5 x_{t-10}x_{t-1}
+
0.1 .
\label{eq:app-narma10}
\end{equation}
The total sequence length is 1750. The first 200 time steps are discarded as washout. The next 1000 points are used for training, the following 250 points for validation, and the final 300 points for testing.

\paragraph{Mackey-Glass.}
For Mackey-Glass, we use the one-dimensional chaotic time series generated from
\begin{equation}
\frac{dx(t)}{dt}
=
\frac{0.2 x(t-17)}{1 + x(t-17)^{10}}
-
0.1 x(t).
\label{eq:app-mackey-glass}
\end{equation}
We generated a time-series of 10000 time steps integrating the equation \ref{eq:app-mackey-glass}.
The prediction task is 84-step-ahead forecasting. Given the scalar input $x_t$, the model predicts $x_{t+84}$. The first 200 points of the lagged sequence are discarded as washout, followed by 1000 training points and 250 validation points. All remaining 8466 points after lagging, washout, training, and validation are used for testing.

\subsubsection{Time-series classification}
\label{app:classification_details}

We evaluate time-series classification on ten datasets from the UEA/UCR archive using the original train/test splits (detailed in Table~\ref{tab:app-stats-classification}). A validation split, using a validation ratio of 0.2, is created from the original training partition. For each sequence, the reservoir is driven by the full time series and a single feature vector is extracted for the readout classifier. The default representation is average pooling over packed reservoir states,
\begin{equation}
\bar{r}
=
\frac{1}{T}
\sum_{t=1}^{T} r_t.
\label{eq:app-average-pooling}
\end{equation}

The batched NumPy/Numba classification pipeline processes multiple independent sequences together.  This batching only affects feature extraction efficiency: the readout still receives one feature vector per sequence. The batch size is 128 for all models and all datasets.

\begin{table}[t]
  \centering
  \small
  \caption{Dataset statistics for the time-series classification experiments. \# Tr and \# Ts denote the number of training and test sequences, respectively. Length is the number of time steps, Input dim. is the number of input channels, and \# Classes is the number of target classes.}
  \label{tab:app-stats-classification}
  \begin{tabular}{lrrrrrl}
    \toprule
    \textbf{Name} & \textbf{\# Tr} & \textbf{\# Ts} & \textbf{Length} & \textbf{Input dim.} & \textbf{\# Classes} & \textbf{Type} \\
    \midrule
    Adiac & 390 & 391 & 176 & 1 & 37 & Image \\
    FordA & 3601 & 1320 & 500 & 1 & 2 & Sensor \\
    JapaneseVowels & 270 & 370 & 29 & 12 & 9 & Audio \\
    Libras & 180 & 180 & 45 & 2 & 15 & HAR \\
    Lightning2 & 60 & 61 & 637 & 1 & 2 & Sensor \\
    Lightning7 & 70 & 73 & 319 & 1 & 7 & Sensor \\
    PEMS-SF & 267 & 173 & 144 & 963 & 7 & Traffic \\
    ShapesAll & 600 & 600 & 512 & 1 & 60 & Image \\
    Wafer & 1000 & 6164 & 152 & 1 & 2 & Sensor \\
    Yoga & 300 & 3000 & 426 & 1 & 2 & Image \\
    \bottomrule
  \end{tabular}
\end{table}

\subsubsection{Long-horizon multivariate forecasting}
\label{app:forecasting_details}

We evaluate long-horizon multivariate forecasting on ETTh1, ETTh2, ETTm1, ETTm2, Solar, and Weather. The input lookback length is fixed to $L_{\mathrm{in}}=96$, the stride is one, and the forecasting horizons are
\[
H \in \{96,192,336,720\}.
\]
All probed variants are listed in Table~\ref{tab:app-stats-longhorizon}. These tasks are multivariate: each model receives a window
\[
\textbf{X}_{t:t+95} \in \mathbb{R}^{96 \times D}
\]
and predicts all variables over the next horizon,
\[
\hat{\textbf{X}}_{t+96:t+96+H-1} \in \mathbb{R}^{H \times D}.
\]
Inputs and targets are standardized using statistics fitted on the training partition and then applied to validation and test partitions. The same scaling protocol is used across both reservoir models and deep learning baselines.

\begin{table*}[t]
\centering
\scriptsize
\caption{Dataset statistics for long-horizon forecasting with input window length $L_{\mathrm{in}}=96$. ``Freq.'' denotes the sampling frequency, ``Ch.'' the number of variables, and $L$ the total length of the dataset. ``Train'', ``Val'', and ``Test'' indicate the number of time steps in each split. ``Train win.'', ``Val win.'', and ``Test win.'' denote the number of sliding-window samples for each horizon.}
\label{tab:app-stats-longhorizon}
\begin{tabular}{lrrrrrrrrrr}
\toprule
Dataset & Freq. & Ch. & $L$ & Train & Val & Test & $H$ & Train win. & Val win. & Test win. \\
\midrule
ETTh1 & 1h & 7 & 14400 & 8640 & 2880 & 2880 & 96  & 8449  & 2785  & 2785 \\
ETTh1 & 1h & 7 & 14400 & 8640 & 2880 & 2880 & 192 & 8353  & 2689  & 2689 \\
ETTh1 & 1h & 7 & 14400 & 8640 & 2880 & 2880 & 336 & 8209  & 2545  & 2545 \\
ETTh1 & 1h & 7 & 14400 & 8640 & 2880 & 2880 & 720 & 7825  & 2161  & 2161 \\
ETTh2 & 1h & 7 & 14400 & 8640 & 2880 & 2880 & 96  & 8449  & 2785  & 2785 \\
ETTh2 & 1h & 7 & 14400 & 8640 & 2880 & 2880 & 192 & 8353  & 2689  & 2689 \\
ETTh2 & 1h & 7 & 14400 & 8640 & 2880 & 2880 & 336 & 8209  & 2545  & 2545 \\
ETTh2 & 1h & 7 & 14400 & 8640 & 2880 & 2880 & 720 & 7825  & 2161  & 2161 \\
ETTm1 & 15min & 7 & 57600 & 34560 & 11520 & 11520 & 96  & 34369 & 11425 & 11425 \\
ETTm1 & 15min & 7 & 57600 & 34560 & 11520 & 11520 & 192 & 34273 & 11329 & 11329 \\
ETTm1 & 15min & 7 & 57600 & 34560 & 11520 & 11520 & 336 & 34129 & 11185 & 11185 \\
ETTm1 & 15min & 7 & 57600 & 34560 & 11520 & 11520 & 720 & 33745 & 10801 & 10801 \\
ETTm2 & 15min & 7 & 57600 & 34560 & 11520 & 11520 & 96  & 34369 & 11425 & 11425 \\
ETTm2 & 15min & 7 & 57600 & 34560 & 11520 & 11520 & 192 & 34273 & 11329 & 11329 \\
ETTm2 & 15min & 7 & 57600 & 34560 & 11520 & 11520 & 336 & 34129 & 11185 & 11185 \\
ETTm2 & 15min & 7 & 57600 & 34560 & 11520 & 11520 & 720 & 33745 & 10801 & 10801 \\
Solar & 10min & 137 & 52560 & 36792 & 5256 & 10512 & 96  & 36601 & 5161 & 10417 \\
Solar & 10min & 137 & 52560 & 36792 & 5256 & 10512 & 192 & 36505 & 5065 & 10321 \\
Solar & 10min & 137 & 52560 & 36792 & 5256 & 10512 & 336 & 36361 & 4921 & 10177 \\
Solar & 10min & 137 & 52560 & 36792 & 5256 & 10512 & 720 & 35977 & 4537 & 9793 \\
Weather & 10min & 21 & 52696 & 36887 & 5270 & 10539 & 96  & 36696 & 5175 & 10444 \\
Weather & 10min & 21 & 52696 & 36887 & 5270 & 10539 & 192 & 36600 & 5079 & 10348 \\
Weather & 10min & 21 & 52696 & 36887 & 5270 & 10539 & 336 & 36456 & 4935 & 10204 \\
Weather & 10min & 21 & 52696 & 36887 & 5270 & 10539 & 720 & 36072 & 4551 & 9820 \\
\bottomrule
\end{tabular}
\end{table*}

\subsection{Hyperparameters}
\label{app:hyperparameters}

\subsubsection{Scalar RC regression}
\label{app:hypers_regression}

For NARMA10 and Mackey-Glass, we compare FRESCO variants with the standard ESN. For each model and reservoir size, hyperparameter optimization is performed with 1000 Optuna trials. The TPE sampler uses multivariate and grouped sampling with 40 startup trials. Hyperband pruning uses three resource levels and reduction factor 2.

The scalar RC search uses six seeds during model selection, and twenty seeds for final testing.
The multi-fidelity schedule has three stages: 512 training points and two seeds, 768 training points and four seeds, and the full training set with six seeds. The final evaluation retrains the readout on the union of training and validation data and evaluates NRMSE and inference time on the test set.

\begin{table}[t]
\centering
\small
\caption{Hyperparameter search ranges for the scalar RC regression benchmarks. Log-uniform sampling is used where indicated. For FRESCO-family models, the parameter listed as spectral radius corresponds to the outer radius of the frequency-domain eigenvalue ring.}
\label{tab:app-hparams-scalar-rc}
\begin{tabular}{lll}
\toprule
Hyperparameter & Models & Search range \\
\midrule
Spectral radius & ESN & $\mathcal{U}(0.1,1.0)$ \\
Spectral radius & FRESCO variants & $ 1.0 $ \\
Inner ring radius & FRESCO variants & $\mathcal{U}(10^{-3},0.999)$ \\
Leak rate & All & log-$\mathcal{U}(10^{-5},1)$ \\
Input scale & All & log-$\mathcal{U}(10^{-7},10^{2})$ \\
Bias scale & All & log-$\mathcal{U}(10^{-7},10^{2})$ \\
Ridge regularization & All & log-$\mathcal{U}(10^{-12},10^{2})$ \\
Gain $\alpha$ & FRESCO variants & log-$\mathcal{U}(0.1,10)$ \\
\bottomrule
\end{tabular}
\end{table}

The scalar regression metric is
\begin{equation}
\mathrm{NRMSE}
=
\sqrt{
\frac{
\frac{1}{T}\sum_{t=1}^{T}(\hat{y}_t-y_t)^2
}{
\mathrm{Var}(y)
}
}.
\label{eq:app-nrmse}
\end{equation}
For the plots of Figure \ref{fig:scaling_size}, the vertical coordinate is the mean test NRMSE across final seeds, and the horizontal coordinate is the mean end-to-end inference time per temporal step in microseconds. Timing includes input ingestion, reservoir update, state packing for FRESCO variants, and readout evaluation.

\subsubsection{Time-series classification}

For time-series classification, we use the same reservoir models as in the scalar RC benchmarks: the standard ESN versus the two FRESCO variants. The reservoir size is included in model selection and is sampled from
\[
\{128,256,512,1024,2048\}.
\]
The remaining reservoir hyperparameter ranges are the same as in Table~\ref{tab:app-hparams-scalar-rc}. Hyperparameters are selected by validation accuracy. For each dataset and model, the search is run for up to 1000 trials or until the wall-clock budget of one hour is reached. After model selection, the best configuration is retrained on the full original training split and evaluated on the original test split over ten random seeds. We report accuracy in Table \ref{tab:reservoir_accuracy}.

\subsubsection{Long-horizon reservoir forecasting}
\label{app:hypers_long_forecast}

For long-horizon forecasting, the ESN and FRESCO variants are implemented in PyTorch. For each input window, the reservoir state is reset to zero, driven for 96 time steps, and the final state is used as the window-level feature vector. The readout maps this feature vector directly to the flattened multivariate forecasting horizon.
The reservoir sizes explored by the search protocol as well as the remaining hyperparameter are detailed in Table~\ref{tab:app-hparams-long-rc}.

For large reservoirs, we apply contiguous block mean-pooling to the reservoir features before the readout computation. If the raw reservoir feature dimension is $N$ and the block size is $b$, the pooled feature dimension is $N/b$. The automatic pooling schedule is chosen so that all candidate reservoir sizes for a given dataset family yield the same readout input dimension. In this way, we can exploit the benefits of large reservoir dynamics without increasing the readout trainable parameters. For ETT, the readout input dimension after pooling is 256; for Solar it is 276; and for Weather it is 264.

Reservoir model selection minimizes validation MSE. The search uses three seeds during model selection and ten final seeds for test evaluation. The multi-fidelity schedule uses 256 training windows with one seed, 1024 training windows with two seeds, and the full training set with three seeds. The Optuna sampler is TPE with multivariate and grouped sampling, with 20 startup trials. Hyperband pruning uses three resource levels and reduction factor 3. Each search is capped by both a maximum number of 1000 trials and a wall-clock timeout budget of five hours. This limit was chosen to reflect practical, resource-constrained deployment scenarios, while Table~\ref{tab:mixfresco_sota} provides a comparison with highly optimized models from the literature achieving state-of-the-art results.

\begin{table}[t]
\centering
\small
\caption{Hyperparameter search ranges for reservoir models in long-horizon forecasting.}
\label{tab:app-hparams-long-rc}
\begin{tabular}{lll}
\toprule
Hyperparameter & Datasets/Models & Search range \\
\midrule
Reservoir size & ETT & $\{256,512,1024,2048,4096,8192,16384\}$ \\
Reservoir size & Solar & $\{276,552,1104,2208,4416,8832,17664\}$ \\
Reservoir size & Weather & $\{264,528,1056,2112,4224,8448,16896\}$ \\
Spectral radius & ESN & $\mathcal{U}(0.1,1.0)$ \\
Spectral radius & FRESCO variants  & $ 1.0 $ \\
Inner ring radius & FRESCO variants & $\mathcal{U}(10^{-3},0.999)$ \\
Leak rate & All & log-$\mathcal{U}(10^{-5},1)$ \\
Input scale & All & log-$\mathcal{U}(10^{-7},10^{2})$ \\
Bias scale & All & log-$\mathcal{U}(10^{-7},10^{2})$ \\
Ridge regularization & All & log-$\mathcal{U}(10^{-6},10^{6})$ \\
Gain $\alpha$ & FRESCO variants & log-$\mathcal{U}(0.1,10)$ \\
\bottomrule
\end{tabular}
\end{table}

\subsubsection{Deep forecasting baselines}
\label{app:forecasting_dl} 

The LSTM, Mamba, and Transformer baselines\footnote{These deep baselines are budget-limited in-house implementations designed for a controlled comparison under a common model-selection protocol. Table~\ref{tab:mixfresco_sota} separately provides published results from specialized architectures for predictive-performance context.} are encoder-only one-shot forecasters. Each model receives an input window
\[
\textbf{X} \in \mathbb{R}^{96 \times D},
\]
encodes the window, extracts the last-token representation, and applies a linear prediction head to output $H \times D$ values.

The LSTM baseline uses a stack of LSTM layers followed by dropout and a linear head. The searched hyperparameters are: hidden size in $\{64,128,256,384,512,768\}$, number of layers in $\{1,2,3,4,5,6\} $, dropout in $\{0,0.1,0.2,0.3\}$, batch size in $\{32,64,128\}$, learning rate in $[10^{-4},5\cdot 10^{-3}]$, and weight decay in $[10^{-6},10^{-2}]$.

The Mamba baseline uses a linear input embedding, stacked Mamba blocks, a final normalization layer, and a linear head. The searched hyperparameters are: model hidden dimension in $ \{128,256,384,512,768\}$, number of layers in $\{2,4,6,8\}$, state dimension in $\{16,32,64\}$, convolution width in $\{2,4\}$, expansion factor in $\{1,2\}$, dropout in $\{0,0.1,0.2,0.3\}$, batch size in $\{32,64,128\}$, learning rate in $[10^{-4},5\cdot 10^{-3}]$, and weight decay in $[10^{-6},10^{-2}]$.

The Transformer baseline uses a linear input embedding, sinusoidal positional encoding, Transformer encoder layers with GELU activations, and a linear head. The searched hyperparameters are: embedding hidden dimension in $ \{64,128,256,512,768\}$, number of attention heads in $\{4,8,16\}$, number of encoder layers in $\{2,4,6,8\}$, feed-forward dimension in $\{64,128,256,512,1024,2048\}$, dropout in $\{0,0.1,0.2,0.3\}$, batch size in $\{32,64,128\}$, learning rate in $[10^{-4},5\cdot 10^{-3}]$, and weight decay in $[10^{-6},10^{-2}]$. Invalid Transformer configurations for which $d_{\mathrm{model}}$ is not divisible by the number of heads are pruned.

For all deep baselines, the search objective is validation MSE. The multi-fidelity schedule uses 512 training windows for 2 epochs, 2048 training windows for 5 epochs, and the full training set for 10 epochs. Training uses AdamW, gradient clipping with norm 1.0, and early stopping with patience 3 during model selection. After model selection, the best configuration is retrained on the union of training and validation data, with the last 10\% of this chronological train+validation set used as an internal holdout for early stopping. Final training runs for at most 50 epochs, uses patience 10, and applies a StepLR scheduler with step size 10 and decay factor 0.5. Final evaluation is averaged over ten random seeds.

\subsection{Computational resources}
\label{app:computational-resources}

All efficiency comparisons in this paper are strictly self-contained: each experiment compares models within a single hardware environment, and runtimes or energy measurements are never compared across different devices. The three environments are detailed below.

All NumPy/Numba RC experiments are CPU-only. They were run on a machine with two Intel(R) Xeon(R) CPU E5-2698 v4 processors at 2.20 GHz, for a total of 40 physical CPU cores. 

The results in Figs.~\ref{fig:FRESCO}, \ref{fig:bench_input_embedding}, and \ref{fig:bench_readout} are obtained on an Apple M3 Pro CPU with 12 cores (6 performance and 6 efficiency) and 18 GB of RAM. OpenBLAS\footnotemark{} is used to efficiently compute the matrix-vector products for the standard ESN baseline and Eq.~\ref{eq:embedding_complex_dense_layer}.
\footnotetext{\url{https://www.openmathlib.org/OpenBLAS/}}

Long-horizon forecasting experiments are implemented in PyTorch and run on a workstation with three NVIDIA A100 80GB PCIe GPUs. Reservoir feature extraction in the PyTorch implementation is batched over sliding windows, with default window batch size 512. However, differently from standard deep learning models, the RC results are not affected by the batch size since the readout is a closed-form solution of the reservoir features collected.

For final long-horizon train-and-test runs, energy consumption and carbon emissions are tracked with CodeCarbon's tracker \citep{courty2024mlco2}. The tracker covers the full final loop over the ten final seeds, including final training on the train+validation split and evaluation on the test set. Energy is reported in kWh.

\section{Additional results}
\label{app:additional_results}

\subsection{Full results on long-horizon multivariate forecasting}

In Table \ref{tab:anytime_log_results}, we report the test MSE and MAE results for ESN, LSTM, Mamba, Transformer, and FRESCO variants, obtained with the model selection protocol described in Appendix \ref{app:experiments}.
%
%
%
\begin{table*}[t]
\centering
\caption{Full forecasting results parsed from anytime Bayesian optimization logs. The best result is red and the second best result is blue.}
\label{tab:anytime_log_results}
\scriptsize
\setlength{\tabcolsep}{2.2pt}
\renewcommand{\arraystretch}{0.95}
\resizebox{\textwidth}{!}{%
\begin{tabular}{ll|cc|cc|cc|cc|cc|cc}
\toprule
Dataset & Models & \multicolumn{2}{c}{ESN} & \multicolumn{2}{c}{LSTM} & \multicolumn{2}{c}{Mamba} & \multicolumn{2}{c}{Transformer} & \multicolumn{2}{c}{FRESCO (plain)} & \multicolumn{2}{c}{FRESCO (mix)} \\
 &  & \multicolumn{2}{c}{Ours} & \multicolumn{2}{c}{Ours} & \multicolumn{2}{c}{Ours} & \multicolumn{2}{c}{Ours} & \multicolumn{2}{c}{Ours} & \multicolumn{2}{c}{Ours} \\
Dataset & Horizon & MSE & MAE & MSE & MAE & MSE & MAE & MSE & MAE & MSE & MAE & MSE & MAE \\
\midrule
\multirow{5}{*}{\rotatebox[origin=c]{90}{ETTm1}} & 96 & 0.358 & 0.396 & 0.562 & 0.534 & 0.589 & 0.532 & 0.443 & 0.460 & \textcolor{blue}{0.349} & \textcolor{blue}{0.390} & \textcolor{red}{0.349} & \textcolor{red}{0.389} \\
 & 192 & \textcolor{blue}{0.395} & \textcolor{blue}{0.419} & 0.583 & 0.560 & 0.596 & 0.563 & 0.483 & 0.490 & \textcolor{red}{0.394} & \textcolor{red}{0.419} & 0.398 & 0.423 \\
 & 336 & \textcolor{blue}{0.448} & \textcolor{blue}{0.457} & 0.736 & 0.659 & 0.860 & 0.712 & 0.604 & 0.585 & 0.450 & 0.461 & \textcolor{red}{0.446} & \textcolor{red}{0.457} \\
 & 720 & \textcolor{red}{0.602} & \textcolor{red}{0.553} & 1.052 & 0.796 & 1.024 & 0.772 & 0.988 & 0.800 & 0.607 & 0.558 & \textcolor{blue}{0.603} & \textcolor{blue}{0.555} \\
 & Avg & 0.451 & \textcolor{blue}{0.456} & 0.733 & 0.637 & 0.767 & 0.645 & 0.629 & 0.584 & \textcolor{blue}{0.450} & 0.457 & \textcolor{red}{0.449} & \textcolor{red}{0.456} \\
\midrule
\multirow{5}{*}{\rotatebox[origin=c]{90}{ETTm2}} & 96 & \textcolor{red}{0.182} & \textcolor{red}{0.284} & 0.640 & 0.629 & 0.411 & 0.474 & 0.359 & 0.445 & 0.184 & 0.289 & \textcolor{blue}{0.183} & \textcolor{blue}{0.288} \\
 & 192 & 0.268 & 0.354 & 1.273 & 0.898 & 0.797 & 0.684 & 0.712 & 0.641 & \textcolor{red}{0.262} & \textcolor{blue}{0.353} & \textcolor{blue}{0.262} & \textcolor{red}{0.353} \\
 & 336 & 0.453 & 0.470 & 1.533 & 0.956 & 1.522 & 0.935 & 1.310 & 0.863 & \textcolor{red}{0.387} & \textcolor{red}{0.448} & \textcolor{blue}{0.391} & \textcolor{blue}{0.450} \\
 & 720 & \textcolor{red}{0.563} & \textcolor{red}{0.564} & 4.077 & 1.629 & 3.939 & 1.515 & 5.681 & 1.739 & 0.743 & 0.664 & \textcolor{blue}{0.719} & \textcolor{blue}{0.652} \\
 & Avg & \textcolor{red}{0.367} & \textcolor{red}{0.418} & 1.881 & 1.028 & 1.667 & 0.902 & 2.015 & 0.922 & 0.394 & 0.439 & \textcolor{blue}{0.389} & \textcolor{blue}{0.436} \\
\midrule
\multirow{5}{*}{\rotatebox[origin=c]{90}{Weather}} & 96 & 0.156 & 0.235 & 0.180 & 0.264 & 0.216 & 0.297 & 0.160 & 0.244 & \textcolor{blue}{0.155} & \textcolor{blue}{0.235} & \textcolor{red}{0.153} & \textcolor{red}{0.232} \\
 & 192 & 0.202 & 0.283 & 0.240 & 0.314 & 0.280 & 0.352 & 0.217 & 0.299 & \textcolor{red}{0.196} & \textcolor{red}{0.277} & \textcolor{blue}{0.197} & \textcolor{blue}{0.278} \\
 & 336 & \textcolor{blue}{0.251} & \textcolor{blue}{0.326} & 0.314 & 0.369 & 0.331 & 0.390 & 0.273 & 0.345 & \textcolor{red}{0.250} & \textcolor{red}{0.322} & 0.254 & 0.327 \\
 & 720 & 0.312 & 0.371 & 0.414 & 0.426 & 0.444 & 0.451 & 0.365 & 0.404 & \textcolor{blue}{0.309} & \textcolor{red}{0.365} & \textcolor{red}{0.308} & \textcolor{blue}{0.367} \\
 & Avg & 0.230 & 0.304 & 0.287 & 0.343 & 0.318 & 0.372 & 0.254 & 0.323 & \textcolor{red}{0.227} & \textcolor{red}{0.300} & \textcolor{blue}{0.228} & \textcolor{blue}{0.301} \\
\midrule
\multirow{5}{*}{\rotatebox[origin=c]{90}{Solar}} & 96 & 0.227 & 0.336 & \textcolor{red}{0.188} & \textcolor{red}{0.242} & 0.216 & 0.258 & \textcolor{blue}{0.199} & \textcolor{blue}{0.244} & 0.230 & 0.333 & 0.215 & 0.315 \\
 & 192 & 0.256 & 0.351 & \textcolor{red}{0.204} & \textcolor{red}{0.256} & 0.244 & 0.292 & \textcolor{blue}{0.224} & \textcolor{blue}{0.280} & 0.248 & 0.348 & 0.241 & 0.330 \\
 & 336 & 0.301 & 0.380 & \textcolor{red}{0.213} & \textcolor{red}{0.270} & 0.267 & 0.298 & \textcolor{blue}{0.225} & \textcolor{blue}{0.292} & 0.274 & 0.364 & 0.300 & 0.379 \\
 & 720 & 0.293 & 0.373 & \textcolor{red}{0.211} & \textcolor{red}{0.270} & 0.228 & \textcolor{blue}{0.287} & \textcolor{blue}{0.222} & 0.292 & 0.295 & 0.381 & 0.283 & 0.351 \\
 & Avg & 0.269 & 0.360 & \textcolor{red}{0.204} & \textcolor{red}{0.259} & 0.239 & 0.283 & \textcolor{blue}{0.217} & \textcolor{blue}{0.277} & 0.262 & 0.356 & 0.260 & 0.344 \\
\midrule
\multirow{5}{*}{\rotatebox[origin=c]{90}{ETTh1}} & 96 & \textcolor{blue}{0.455} & \textcolor{blue}{0.464} & 0.785 & 0.697 & 0.917 & 0.736 & 0.697 & 0.641 & 0.468 & 0.475 & \textcolor{red}{0.452} & \textcolor{red}{0.462} \\
 & 192 & \textcolor{red}{0.626} & \textcolor{red}{0.567} & 0.962 & 0.776 & 1.071 & 0.803 & 0.912 & 0.744 & 0.667 & 0.591 & \textcolor{blue}{0.640} & \textcolor{blue}{0.578} \\
 & 336 & \textcolor{red}{0.852} & \textcolor{red}{0.695} & 1.038 & 0.802 & 1.077 & 0.807 & 1.121 & 0.850 & \textcolor{blue}{0.881} & \textcolor{blue}{0.704} & 0.896 & 0.718 \\
 & 720 & 1.081 & \textcolor{blue}{0.825} & \textcolor{red}{0.978} & \textcolor{red}{0.799} & 1.105 & 0.840 & \textcolor{blue}{1.058} & 0.842 & 1.119 & 0.840 & 1.184 & 0.847 \\
 & Avg & \textcolor{red}{0.753} & \textcolor{red}{0.638} & 0.941 & 0.768 & 1.043 & 0.796 & 0.947 & 0.769 & \textcolor{blue}{0.784} & 0.652 & 0.793 & \textcolor{blue}{0.651} \\
\midrule
\multirow{5}{*}{\rotatebox[origin=c]{90}{ETTh2}} & 96 & 0.515 & 0.513 & 1.901 & 1.104 & 1.962 & 1.055 & 1.353 & 0.900 & \textcolor{red}{0.381} & \textcolor{red}{0.443} & \textcolor{blue}{0.381} & \textcolor{blue}{0.443} \\
 & 192 & \textcolor{red}{0.747} & \textcolor{red}{0.672} & 2.432 & 1.295 & 3.384 & 1.593 & 4.063 & 1.577 & \textcolor{blue}{0.797} & \textcolor{blue}{0.691} & 0.856 & 0.715 \\
 & 336 & \textcolor{red}{0.718} & \textcolor{red}{0.659} & 3.062 & 1.511 & 2.301 & 1.278 & 1.963 & 1.187 & 0.994 & \textcolor{blue}{0.779} & \textcolor{blue}{0.989} & 0.780 \\
 & 720 & \textcolor{red}{1.171} & \textcolor{red}{0.866} & 3.734 & 1.665 & 2.853 & 1.412 & 3.951 & 1.691 & \textcolor{blue}{1.383} & \textcolor{blue}{0.941} & 1.557 & 1.000 \\
 & Avg & \textcolor{red}{0.788} & \textcolor{red}{0.677} & 2.782 & 1.394 & 2.625 & 1.335 & 2.832 & 1.339 & \textcolor{blue}{0.888} & \textcolor{blue}{0.714} & 0.946 & 0.734 \\
\bottomrule
\end{tabular}%
}
\end{table*}

\subsection{Input embedding topology}
\label{app:embedding_ablation}

Our main evaluation emphasizes the computational advantage of dimensional zero-padding. Here, we isolate its structural effect by comparing the proposed 2D embedding with a 1D zero-padding variant, which pads the input along its original feature axis before applying an $N$-point RFFT. This comparison is meaningful only for multivariate inputs: for univariate inputs, the two transforms are equivalent.

We evaluate both FRESCO variants on the two multivariate classification datasets JapaneseVowels and PEMS-SF, and repeat HPO independently for every model and topology. Consequently, the accuracies in Table~\ref{tab:embedding_ablation} may differ slightly from those in the main paper. On JapaneseVowels, the two embeddings attain comparable predictive accuracy; however, the 1D variant requires the largest reservoir considered by the search grid ($N=4096$), whereas the 2D embedding reaches comparable performance with substantially smaller reservoirs. On PEMS-SF, whose input dimensionality is much larger, the 2D embedding improves accuracy for both variants at comparable reservoir sizes. Thus, dimensional zero-padding provides not only the computational benefit described in Section~3, but also a useful structural prior for multivariate inputs.

\begin{table}[t]
\centering
\caption{Ablation of the input embedding topology. Each entry reports mean test accuracy and standard deviation after independently optimizing hyperparameters for the respective model and topology.}
\label{tab:embedding_ablation}
\small
\begin{tabular}{llccc}
\toprule
Dataset & Model & Topology & $N$ & Accuracy \\
\midrule
\multirow{4}{*}{JapaneseVowels}
& FRESCO (plain) & 2D & 1024 & $0.9702_{\pm 0.0030}$ \\
& FRESCO (plain) & 1D & 4096 & $\textbf{0.9735}_{\pm 0.0020}$ \\
& FRESCO (mix)   & 2D & 512  & $\textbf{0.9702}_{\pm 0.0024}$ \\
& FRESCO (mix)   & 1D & 4096 & $0.9686_{\pm 0.0079}$ \\
\midrule
\multirow{4}{*}{PEMS-SF}
& FRESCO (plain) & 2D & 4096 & $\textbf{0.7283}_{\pm 0.0271}$ \\
& FRESCO (plain) & 1D & 4096 & $0.6960_{\pm 0.0078}$ \\
& FRESCO (mix)   & 2D & 2048 & $\textbf{0.7318}_{\pm 0.0158}$ \\
& FRESCO (mix)   & 1D & 2048 & $0.6994_{\pm 0.0097}$ \\
\bottomrule
\end{tabular}
\end{table}

\subsection{Non-linearity}
\label{app:nonlinearity_ablation}

To assess whether the empirical performance of FRESCO is specific to \textsc{invabs}, we replace the activation in both FRESCO variants with three alternative complex-valued non-linearities:
\begin{align}
\operatorname{invabs}_{\alpha}(z) &=
\frac{\alpha z}{1+|z|}, \nonumber\\
\phi_{\mathrm{modReLU}}(z) &=
\alpha \max(|z|-\beta,0)\exp(i\arg z), \qquad \beta \geq 0, \nonumber\\
\phi_{\mathrm{card}}(z) &=
\frac{\alpha[1+\cos(\arg z)]z}{2}, \nonumber\\
\phi_{\exp\text{-}\tanh}(z) &=
\alpha\tanh(|z|/m)\exp(i\arg z), \qquad m>0.
\end{align}
The latter three are respectively modReLU~\citep{arjovsky2016unitary}, cardioid~\citep{virtue2017better}, and a phase-preserving $\tanh$ activation. We evaluate $N\in\{128,1024\}$ on NARMA10 and Mackey-Glass. Model selection is performed independently for every task, reservoir size, FRESCO variant, and activation using the three-stage TPE/Hyperband protocol of Appendix~\ref{app:hypers_regression}. We report mean test NRMSE and standard deviation over 20 reservoir initializations.

Table~\ref{tab:nonlinearity_ablation} shows that the predictive behavior of FRESCO is not specific to \textsc{invabs}: alternative saturation laws and explicit phase-dependent gating can match or improve its performance. The activation nevertheless remains a meaningful architectural choice, with performance varying across tasks and reservoir variants. We retain \textsc{invabs} in the main formulation because it combines competitive performance, bounded reservoir states, and the explicit Lipschitz characterization used in our ESP analysis.

\begin{table*}[t]
\small
\centering
\caption{Non-linearity ablation on scalar RC regression benchmarks. Results are test NRMSE (mean $\pm$ standard deviation) over 20 reservoir initializations. Lower is better.}
\label{tab:nonlinearity_ablation}
\begin{tabular}{llcccc}
\toprule
& & \multicolumn{2}{c}{NARMA10} & \multicolumn{2}{c}{Mackey--Glass} \\
\cmidrule(lr){3-4}\cmidrule(lr){5-6}
Model & Non-linearity & $N=128$ & $N=1024$ & $N=128$ & $N=1024$ \\
\midrule
\multirow{4}{*}{FRESCO (plain)}
& \textsc{invabs} & $0.3000_{\pm 0.0092}$ & $0.1798_{\pm 0.0114}$ & $\textbf{0.0428}_{\pm 0.0051}$ & $0.0263_{\pm 0.0023}$ \\
& modReLU          & $0.3357_{\pm 0.0126}$ & $0.1925_{\pm 0.0059}$ & $0.1087_{\pm 0.0081}$ & $0.0455_{\pm 0.0039}$ \\
& cardioid         & $0.3225_{\pm 0.0187}$ & $\textbf{0.1392}_{\pm 0.0076}$ & $0.0821_{\pm 0.0064}$ & $0.0251_{\pm 0.0023}$ \\
& exp-tanh         & $\textbf{0.2881}_{\pm 0.0139}$ & $0.1646_{\pm 0.0063}$ & $0.0590_{\pm 0.0055}$ & $\textbf{0.0226}_{\pm 0.0022}$ \\
\midrule
\multirow{4}{*}{FRESCO (mix)}
& \textsc{invabs} & $\textbf{0.2132}_{\pm 0.0316}$ & $0.0852_{\pm 0.0082}$ & $\textbf{0.0565}_{\pm 0.0074}$ & $0.0217_{\pm 0.0054}$ \\
& modReLU          & $0.3176_{\pm 0.0260}$ & $0.3231_{\pm 0.0364}$ & $0.0688_{\pm 0.0102}$ & $0.0236_{\pm 0.0029}$ \\
& cardioid         & $0.2691_{\pm 0.0610}$ & $\textbf{0.0780}_{\pm 0.0084}$ & $0.0615_{\pm 0.0079}$ & $0.0187_{\pm 0.0041}$ \\
& exp-tanh         & $0.2516_{\pm 0.0293}$ & $0.0831_{\pm 0.0066}$ & $0.0574_{\pm 0.0030}$ & $\textbf{0.0137}_{\pm 0.0025}$ \\
\bottomrule
\end{tabular}
\end{table*}

\section{Additional Information}
\label{app:additional_information}

\subsection{Limitations}
\label{app:limitations}

While FRESCO significantly improves the computational and energy efficiency of dense recurrent layers, it inherently carries some limitations related to the reservoir computing paradigm and its specific frequency-domain formulation. Like the standard ESN, FRESCO relies on fixed, randomly initialized recurrent weights, training only the linear readout. While this avoids backpropagation through time and enables efficient closed-form optimization, it inherently sacrifices some absolute task performance compared to deep sequence models where recurrent dynamics are explicitly trained end-to-end for a specific task.

For FRESCO, this limitation is compounded by restricting the spatial recurrent connectivity to circulant topologies. The effect on predictive performance is task-dependent: our memory capacity analysis shows that the circulant constraint reduces pure linear recall, but FRESCO can recover comparable linear memory capacity via ring-topology initialization at $\mathcal{O}(N)$ cost, and cross-frequency mixing in FRESCO (mix) yields non-linear processing capacity exceeding that of standard dense ESNs (Section~\ref{sec:benchmark_memory}). Nevertheless, the circulant structure theoretically constrains the accessible dynamic state space compared to a fully unstructured dense weight matrix.

Furthermore, while dimensional zero-padding scales efficiently with input dimensionality, the stability and expressivity of frequency-domain recurrence over ultra-long context windows remain to be empirically validated.
Theorem~\ref{thm:dimensional_padding} establishes the mathematical foundation for extending FRESCO to higher-dimensional inputs (e.g., 2D image patches embedded into a 3D frequency-domain reservoir via multi-dimensional circular convolutions), retaining $\mathcal{O}(\SNr)$ scaling. However, the practical behavior of this generalization has not yet been evaluated.
In addition, the dimensional zero-padding embedding assumes that the reservoir size $N$ is a multiple of the input dimension $N_x$. While this is easily managed during architectural design and hyperparameter tuning, it does impose a discrete step-size constraint on the configuration space.

The computational execution times reported in our scaling analysis (in Figure~\ref{fig:FRESCO} and Appendices~\ref{app:embedding-benchmark} and~\ref{app:readout-benchmark}) represent minimal execution times across multiple trials (all other experiments using real data are evaluated with standard mean and deviation metrics). While this isolates the pure algorithmic complexity and hardware-level scaling limits by filtering out transient system fluctuations, it depicts an ideal best-case scenario. Real-world deployment times may experience slight variances due to hardware throttling, concurrent processes, or sub-optimal memory caching.

\subsection{Nomenclature}
\label{app:nomenclature}

Table~\ref{tab:nomenclature} summarizes the mathematical symbols used throughout the main text and the supplementary material. Note that some variables currently have multiple notations across different sections, which are grouped together in the ``Symbol'' column for reference.

\begin{table}[ht]
    \centering
    \caption{Nomenclature and Notation Summary}
    \label{tab:nomenclature}
    \small
    \begin{tabular}{ll p{8.5cm} l}
        \toprule
        \textbf{Symbol} & \textbf{Name} & \textbf{Description} & \textbf{Space} \\
        \midrule
        \multicolumn{4}{l}{\textit{General}} \\
        \midrule
        $\SNx$ & Input size & Number of reservoir input features. & $\mathbb{N}$ \\
        $\SNy$ & Output size & Number of reservoir output features in the readout layer. & $\mathbb{N}$ \\
        $\SNr$ & Reservoir size & Number of non-linear nodes (neurons) comprising the reservoir. & $\mathbb{N}$ \\

        \midrule
        \multicolumn{4}{l}{\textit{Spatial Domain (SD) - Standard ESN}} \\
        \midrule
        $\Sx{t}$ & Input vector & Reservoir input vector at discrete time step $t$. & $\mathbb{R}^{\SNx}$ \\
        $\Sy{t}$ & Output vector & Readout output vector at time step $t$. & $\mathbb{R}^{\SNy}$ \\
        $\Sr{t}$ & Reservoir state & Hidden state vector representing the reservoir. & $\mathbb{R}^{\SNr}$ \\
        $\SWx$ & Input weights & Input weights mapping the input vector to the reservoir nodes. & $\mathbb{R}^{\SNr \times \SNx}$ \\
        $\SWr$ & Recurrent weights & Recurrent weight matrix governing the reservoir's dynamics. & $\mathbb{R}^{\SNr \times \SNr}$ \\
        $\SWy$ & Readout weights & Readout weights mapping the reservoir state to the output. & $\mathbb{R}^{\SNy \times \SNr}$ \\
        $\Sb$  & Bias vector & Constant offset applied to the reservoir activation arguments. & $\mathbb{R}^{\SNr}$ \\
        $\Sxzp$ & Padded input vector & Zero-padded input vector. & $\mathbb{R}^{\SNr}$ \\
        $\Sactfun$ & Activation function & Non-linear activation function. & Function \\
        
        \midrule
        \multicolumn{4}{l}{\textit{Frequency Domain (FD) - FRESCO}} \\
        
        \midrule
        $\Fx{t}$ & Input vector & FD input vector at discrete time step $t$. & $\mathbb{C}^{\SNx}$ \\
        $\Fxzp$ & Padded input vector & FD representation of zero-padded input vector. & $\mathbb{C}^{\SNr}$ \\
        $\Fwx$ & Input weight vector & FD input weights vector. & $\mathbb{C}^{\SNr}$ \\
        $\Fr{t}$ & Reservoir vector & FD reservoir state vector at discrete time step $t$. & $\mathbb{C}^{\SNr}$ \\
        $\Fwr$ & Recurrent weight vector & FD recurrent weight vector. & $\mathbb{C}^{\SNr}$ \\
        $\Fb$ & Bias vector & FD bias vector. & $\mathbb{C}^{\SNr}$ \\
        
        \midrule
        $\SXzp$ & Padded input matrix & SD input matrix (dimensionally zero-padded input vector). & $\mathbb{R}^{\SNx \times N_2}$ \\
        $\FXzp, \FX{t}$ & Padded input matrix & FD representation of input matrix. & $\mathbb{C}^{N_1 \times N_2}$ \\
        $\FWx$ & Input weight matrix & FD input weight matrix. & $\mathbb{C}^{N_1 \times N_2}$ \\
        $\FR{t}$ & Reservoir state & FD reservoir state matrix at discrete time step $t$. & $\mathbb{C}^{N_1 \times N_2}$ \\
        $\FWr$ & Recurrent weight matrix & FD recurrent weight matrix. & $\mathbb{C}^{N_1 \times N_2}$ \\
        $\FB$ & Bias matrix &  FD bias matrix. & $\mathbb{C}^{N_1 \times N_2}$ \\
        $\Factfun$ & Activation function & Complex-valued non-linear activation function applied in FD. & Function \\
        
        \midrule
        $\Fp{t}$ & Packed state vector & Real-valued packed reservoir state representation $\Fp{t} = \pack{\FR{t}}$. & $\mathbb{R}^{\SNr}$ \\
        $\FWp$ & Packed readout matrix & Readout weight matrix for packed state vectors. & $\mathbb{R}^{\SNy \times \SNr}$ \\
        
        \midrule
        \multicolumn{4}{l}{\textit{Operators \& Transforms}} \\
        \midrule
        $\pack{\cdot}$ & Packing operator & Operator mapping $\FR{t}$ to a non-redundant real-valued vector. & Operator \\
        $\fourier{\cdot}$ & Fourier operator & Forward Discrete Fourier Transform (DFT) operator. & Operator \\
        $\invfourier{\cdot}$ & Inverse Fourier & Inverse Discrete Fourier Transform (IDFT) operator. & Operator \\
        \bottomrule
    \end{tabular}
\end{table}



\end{document}